\documentclass[11pt]{article}

\usepackage{setspace}
\usepackage{amssymb}
\usepackage[english]{babel}
\usepackage[T1]{fontenc}
\usepackage[utf8]{inputenc}
\usepackage{eurosym}
\usepackage{amsfonts}
\usepackage{amsmath}
\usepackage{graphicx}
\usepackage{float}
\usepackage{ragged2e}
\usepackage{natbib}
\usepackage{hyperref}
\usepackage{array,float,booktabs,caption}
\usepackage[left=2.5cm,right=2.5cm,top=2.5cm,bottom=2.5cm]{geometry}
\usepackage[flushleft]{threeparttable}
\usepackage{multirow}
\usepackage{rotating}
\usepackage{amsthm}
\usepackage{dsfont} 
\usepackage{xcolor}
\usepackage{setspace}
\usepackage[font=normalsize,justification=centering]{caption}
\usepackage{subcaption}
\usepackage{algorithm,algorithmic}
\usepackage{tikz}
\usepackage{array}
\newcolumntype{P}[1]{>{\centering\arraybackslash}p{#1}}

\usetikzlibrary{shapes.geometric, arrows}
\tikzstyle{startstop} = [rectangle, rounded corners, minimum width=3cm, minimum height=1cm,text centered, draw=black, fill=blue!20]
\tikzstyle{process} = [rectangle, minimum width=3cm, minimum height=1cm, text centered, draw=black, fill=orange!30]
\tikzstyle{io} = [trapezium, trapezium left angle=70, trapezium right angle=110, minimum width=3cm, minimum height=1cm, text centered, draw=black, fill=blue!30]
\tikzstyle{arrow} = [thick,->,>=stealth]
\tikzstyle{decision} = [ellipse, minimum width=3cm, minimum height=1cm, text centered, draw=black, fill=green!30]
\usetikzlibrary{quotes,angles}

\makeatletter
\def\BState{\State\hskip-\ALG@thistlm}
\makeatother

\newtheorem{theorem}{Theorem}

\newtheorem{theorem4}{Theorem}

\newtheorem{theorem6}{Theorem}
\newtheorem{theorem7}{Theorem}
\newtheorem{theorem8}{Theorem}

\newtheorem{assumption}[theorem4]{Assumption}

\newtheorem{axiom}[theorem6]{Axiom}

\newtheorem{definition}[theorem]{Definition}

\newtheorem{lemma}[theorem8]{Lemma}

\newtheorem{proposition}[theorem7]{Proposition}

\usepackage{thmtools}
\usepackage{thm-restate}
\usepackage{hyperref}
\usepackage{cleveref}

\begin{document}

	\title{\LARGE Measuring the Driving Forces of Predictive Performance:\\
\medskip
 \LARGE Application to Credit Scoring\thanks{We thank for their comments Bart Baesens, Raffaella Calabrese, Colin Cameron, Jean-Edouard Colliard, Xavier D'Haultfoeuille, Serge Darolles, Jean-François Dupuy, Emmanuel Flachaire, Thierry Foucault, Julien Grand-Clément, Emmanuel Kemel, Sébastien Laurent, Pascal Lavergne, Jean-Michel Poggi, Gilbert Saporta, Olivier Scaillet, Nicolas Vieille, participants at the 2022 CISEM workshop, 2022 Quantitative Finance and Financial Econometrics, 2022 Financial Econometrics Day (University Paris Nanterre), 2023 Risk Forum, 2023 Asian Meeting of the Econometric Society (AMES), 2023 Econometric Society Australia Meeeting (ESAM), 2024 Econometric Society European Meeting (ESEM), and workshop and seminar participants at Aix-Marseille School of Economics (AMSE), Amundi, AXA Investment Managers, COFACE, and University of Orléans. We thank the Institut Universitaire de France (IUF), the ACPR Chair in Regulation and Systemic Risk, the Excellence Initiative of Aix-Marseille University - A*MIDEX, and the French National Research Agency (AMSE ANR-17-EURE-0020, Ecodec ANR-11-LABX-0047, MLEforRisk ANR-21-CE26-0007) for supporting our research.
 } 
 \vspace{0.75cm}}
	\author{\normalsize Sullivan Hué\thanks{
           Aix Marseille University, CNRS, AMSE, Marseille, France. Email: sullivan.hue@univ-amu.fr } \and %
		\normalsize Christophe Hurlin\thanks{
			University of Orléans (LEO) and Institut Universitaire de France (IUF), Rue de Blois, 45067 Orléans, France. Email: christophe.hurlin@univ-orleans.fr} \and 
		\normalsize Christophe Pérignon\thanks{
			HEC Paris, 1 Rue de la Libération, 78350 Jouy-en-Josas,
			France. Email: perignon@hec.fr} \and \normalsize Sébastien Saurin\thanks{
			University of Orléans (LEO), Rue de Blois, 45067 Orléans, France. Email: sebastien.saurin@univ-orleans.fr}\vspace{1cm}}
	\date{\normalsize \today\vspace{0.5cm}\\
	}
	\maketitle
	\vspace{-1.cm}
	\begin{abstract}
	\noindent As they play an increasingly important role in determining access to credit, credit scoring models are under growing scrutiny from banking supervisors and internal model validators. These authorities need to monitor the model performance and identify its key drivers. To facilitate this, we introduce the XPER methodology to decompose a performance metric (e.g., AUC, $R^2$) into specific contributions associated with the various features of a forecasting model. XPER is theoretically grounded on Shapley values and is both model-agnostic and performance metric-agnostic. Furthermore, it can be implemented either at the model level or at the individual level. Using a novel dataset of car loans, we decompose the AUC of a machine-learning model trained to forecast the default probability of loan applicants. We show that a small number of features can explain a surprisingly large part of the model performance. Notably, the features that contribute the most to the predictive performance of the model may not be the ones that contribute the most to individual forecasts (SHAP). Finally, we show how XPER can be used to deal with heterogeneity issues and improve performance. 

	\end{abstract}
	
	\vspace{0.5cm}
	
	\hspace{0.2cm} \textit{Keywords: 
 Explainability; Credit scoring; Performance metrics; Shapley values \medskip }

    \baselineskip=2\normalbaselineskip\newpage

\section{Introduction}

 Why is the AUC of a given machine learning model equal to 0.7? Which features mainly explain this performance? What are the contributions of the different features to the MSE of a regression model? To answer these questions, we develop and apply a general methodology, called eXplainable PERformance (XPER), which measures the marginal contribution of a particular feature to the predictive performance of a regression or classification model.
 
Being able to identify the driving forces of the performance of a predictive model is crucial in the context of credit scoring, where complex and opaque machine learning models are increasingly being used by banks and fintechs. Such decomposition is of primary importance both for banking supervisors and internal model validation teams as they need to understand why a given model is working or not, and for which types of borrowers. Furthermore, it also permits to address heterogeneity issues by identifying groups of individuals for which the features have similar effects on performance. One can then estimate group-specific models to improve overall performance and make sure the credit scoring model operates effectively for all borrowers. In this paper, we propose an application where we use XPER to decompose the AUC of a scoring model forecasting loan defaults and to identify the features with the strongest impact on model performance. 

The XPER framework is based on Shapley values \citep{Shapley_1953}. While Shapley values decompose a \textit{payoff} among \textit{players} in a \textit{game}, XPER values decompose a \textit{performance metric} among \textit{features} in a \textit{model}. More precisely, our method breaks down the difference between a performance metric and a benchmark value among the various features of the model. Formally, an XPER value is defined as the weighted average marginal contribution of a given feature to a performance metric, obtained in a set of coalitions of other features. For instance, evaluating the XPER value of $x_{1}$ in a three-feature model implies to assess the incremental performance due to $x_{1}$ by successively considering four subsets or coalitions of features: one coalition including no features, another coalition only $x_{2}$, another one only $x_{3}$, and a last coalition including both $x_{2}$ and $x_{3}$.
 
 The application of the Shapley methodology in the context of explaining model performance is not trivial. Indeed, many Shapley values can be defined given the assumptions made on the model, the data, and the features excluded from the coalitions (see for instance \cite{strumbelj_efficient_2010}, 
 \cite{Redell2019}, \cite{Kumar_2020}, \cite{Sundararajan2020}, and \cite{Aas2021}). In this paper, we define XPER values by considering the expectation of the performance metric with respect to the joint distribution of the target variable and of the features which are excluded from the coalition. A key advantage of this definition is that the benchmark value has a meaningful interpretation: it corresponds to the performance metric that we would obtain on a hypothetical sample in which the target variable is independent from all the features included in the model, i.e., a fully misspecified model with irrelevant features. Moreover, our method does not require to re-estimate the model (\textit{\`a la} \cite{gromping_estimators_2007}) nor to pick ad-hoc values for features excluded from the coalition (\textit{\`a la} \cite{Israeli2007}). 

 XPER offers several other advantages. First, as the XPER decomposition is based on Shapley values, it is theoretically grounded and XPER values satisfy the axioms of efficiency, symmetry, linearity, and null effects. These axioms collectively ensure equitable and consistent allocation of model performance contributions among the involved features.\footnote{\textit{Efficiency} dictates that the total contribution of all features equals the overall model performance metric, minus the benchmark value. \textit{Symmetry} ensures that features making equivalent contributions receive identical XPER values. \textit{Linearity} stipulates that if model performance can be decomposed into separate parts, the XPER value of the whole model equals the sum of its parts. The \textit{null effect} axiom ensures features that do not influence model performance receive a XPER value of zero.} Second, XPER is model-agnostic as it permits to interpret the predictive performance of any econometric or machine learning model. Third, it is metric-agnostic as it can break down any performance metric: predictive accuracy (e.g., AUC, Gini, accuracy), goodness of fit (e.g., $R^2$), information criterion (e.g., AIC, BIC), loss function (e.g., MSE, MAE, Q-like), or economic performance (e.g., profit-and-loss or P\&L). Fourth, XPER can be implemented either at the global level or at the local level. At the global level, the XPER value of a given feature measures its contribution to the performance of the model. At the local level, the XPER value of a given feature measures its contribution to the model performance that comes from a given individual. Finally, as the number of coalitions grows fast with the number of features, we propose two estimation procedures: an exact one when the number of features remains moderate and an approximated one, which is a modified version of the Kernel SHAP method of \cite{Lundberg2017}. We developed a package in Python to compute the XPER values and made it available at \href{https://github.com/hi-paris/XPER}{https://github.com/hi-paris/XPER}. 

 We apply our methodology to a novel dataset of auto loans provided by an international bank. We demonstrate the usefulness of XPER for decomposing the AUC and other performance metrics of a credit scoring model. Our findings reveal that a small number of features can explain a surprisingly large part of the model performance. Furthermore, the empirical analysis confirms that XPER values differ from standard feature importance metrics. More importantly, we show how to use XPER to boost model performance. To do so, we build homogeneous groups of individuals by clustering them based on their individual XPER values. By construction, within a given group, the features tend to have similar effects on performance (same sign, same intensity). We then show that estimating group-specific models yields to a higher accuracy than with a one-fits-all model.

 Our paper contributes to the burgeoning literature on eXplainable Artificial Intelligence (XAI) \citep{Murdoch2019, Gelman2021, Molnar2024}. One well-known limitation of AI and machine learning methods comes from their opacity and lack of explainability. Most of these algorithms are considered as black boxes in the sense that the corresponding outcomes cannot be easily explained to final users nor related to the initial features \citep{Sun2021}. Recently, many explainability methods have been designed to explain black-box models by measuring which features most affect its predictions (see \cite{Zhao2021,Horel2022,Liu2023,Molnar2024}). One of the most impactful XAI methods is the Shapley additive explanation (SHAP) of \cite{Lundberg2017}, which distributes the predictions of a model among its features (see also \cite{Sundararajan2017, Lundberg2018, SundararajanDhamdereAgarwal2020, Casalicchio_2019, Bowen2020, Aas2021}). However, explaining the predictive performance of a model (XPER) is not equivalent to explaining its forecasts (SHAP). The latter approach identifies the features contributing the most to model predictions both when the model is correct and when the model is wrong. This implies that some features identified as contributors that lead the predictions may, in fact, mislead the model.

 While several other studies also decompose model performance using Shapley values (see Appendix \ref{Literature_review} for a comparison table of the main studies), XPER distinguishes itself in three ways. First, XPER stands out for its versatility, enabling the analysis of the drivers of \textit{any} performance metric for \textit{any} estimated model. In the literature, most articles focus exclusively on the decomposition of a single performance metric \citep{owen_shapley_2017,Redell2019,sutera2021global,Verdinelli2023}. Second, XPER sets itself apart from most existing methodologies, which only focus on providing a global analysis of the features influencing model performance \citep{Casalicchio_2019,ghorbani2019data,covert2020understanding,fryer2021shapley,moehle2021portfolio,Borup2022,Zhang2023}, by also enabling a local analysis. This allows us to address heterogeneity issues by identifying groups of individuals for which features exhibit similar effects on model performance. Third, XPER meaningfully decomposes the performance metric without relying on strong assumptions or re-estimating the model. In particular, many alternative approaches require re-estimating the model with different subsets of features, which may lead to omitted variable bias \citep{Lipovetsky2001,Israeli2007,bell2024efficient}.

 
\vspace{0.2cm}

\section{A primer on XPER}\label{Primer}

 In this section, we introduce XPER using a simple classification problem with $y\in \{0,1\}$ the target variable and three features $\{x_1, x_2, x_3\}$. We estimate a model $\hat{f}(x_1,x_2,x_3)$ on a training sample and evaluate its predictive performance by considering its AUC on a test sample.
 If none of the three features provides any useful information to predict the target variable, the AUC is 0.5, indicating random classification. We refer to this value as the benchmark AUC, denoted $\phi_0$. In contrast, when AUC $>0.5$, at least one feature contains some relevant information to predict the target variable.
 
 What is the relative contribution of each feature to this predictive performance? We answer this question using XPER by breaking down the difference between the AUC obtained on the test sample and the benchmark value, among the features of the model. Let us denote by $\phi_j$ the XPER contribution of feature $x_j$ to the predictive performance of the model. For instance, the AUC of the model (e.g., predicting/detecting loan default, customer churn, fraud, money laundering) is 0.78 and its XPER decomposition is:
\medskip

\begin{table}[!htbp]
  \centering
    \label{tab:subsets}%
    \begin{tabular}{lccccccccc}
          & AUC   & & $\phi_0$ & & $\phi_1$ & & $\phi_2$ & & $\phi_3$  \vspace{0.1cm}\\  \hline
    Test sample  & 0.78 & =  & 0.50 & +  & 0.14 & +  & 0.10 & + & 0.04 \\
    \end{tabular}%
\end{table}
\vspace{-0.4cm}

It shows that the feature $x_1$ is the main driver of the predictive performance of the model as it explains half of the difference between the AUC of the model and its benchmark ($=0.14/(0.78-0.50)$). The second most informative feature is $x_2$ as it explains another 35.7\% ($=0.10/(0.78-0.50)$) of the difference. Finally, $x_3$ explains the remaining 14.3\%.

 Such performance decomposition can prove handy in several contexts. XPER can help to highlight a potential \textit{heterogeneity} in the predictive performance of a model. Indeed, one can compute the XPER decomposition of a given performance measure on a subset of the test sample. For instance, let us consider two subsets of the test sample, denoted $A$ and $B$. The performance of the model can be evaluated and decomposed with XPER independently in each subset. 
 The illustrative results displayed in the table below indicate that the model performs poorly in $A$. In this case, XPER values can be used to explain the sources of such underperformance of the model. For instance, if the XPER decomposition of the AUC over the subsets $A$ and $B$ is as follows, the weak performance observed in A is only due to feature $x_1$.
 
\begin{table}[!htbp]
  \centering
    \label{tab:subsets}%
    \begin{tabular}{lccccccccc}
          & AUC   & & $\phi_0$ & & $\phi_1$ & & $\phi_2$ & & $\phi_3$  \vspace{0.1cm}\\  \hline
    Subset A & 0.65 & =  & 0.50 & +  & 0.01 & +  & 0.10 & + & 0.04 \\ \vspace{ -0.4cm}\\
     Subset B  & 0.85 & =  & 0.50 & +  & 0.21 & +  & 0.10 & + & 0.04 \\ \vspace{ -0.4cm}\\  
    \end{tabular}%
   \vspace{-0.3cm}
\end{table}

XPER can also help to understand the origin of \textit{overfitting}. Overfitting occurs when a model performs significantly better on the training sample than it does on the test sample. XPER can identify some features which contribute more to the performance of the model in the training sample than in the test sample, and thus explains overfitting. In the table below, the drop in AUC between the training sample (0.90) and the test sample (0.78) illustrates a typical overfitting issue. In this case, the overfitting problem is entirely due to $x_2$.
  
\begin{table}[!htbp]
  \centering
    \begin{tabular}{lccccccccc}
          & AUC   & & $\phi_0$ & & $\phi_1$ & & $\phi_2$ & & $\phi_3$  \vspace{0.1cm}\\  \hline
    Training & 0.90 & =  & 0.50 & +  & 0.15 & +  & 0.20 & + & 0.05 \\ \vspace{ -0.4cm}\\
    Test  & 0.78 & =  & 0.50 & +  & 0.15 & +  & 0.08 & + & 0.05 \\ \vspace{ -0.4cm}\\  \hline
    Training - Test  & 0.12 & =  & 0 & +  & 0 & +  & 0.12 & + & 0 \\
    \end{tabular}%
\end{table}
\vspace{-0.4cm}

A related application of XPER could prove useful during the \textit{variable selection phase} of model development. Indeed, by contrasting XPER values estimated in the training and validation sets, one could only keep the variables maintaining a high-enough level of importance in the validation test.

XPER can be applied to any statistical performance metrics (e.g., $R^2$, Brier Score, Gini index), but also to any \textit{economic performance metrics} (e.g., P\&L, return-on-investment, performance of a financial portfolio). As an example, we can decompose the P\&L of a credit scoring model: \vspace{-0.3cm}
\begin{equation*} \vspace{-0.3cm}
    P\&L = \sum_{i=1}^n(1-\hat{y}_i)(1-y_i)\times profit + (1-\hat{y}_i)y_i \times loss, 
\end{equation*} 
where $profit$ is the money made on any reimbursed loan ($y_i=0$) and $loss$ is the money lost on any defaulted loan ($y_i=1$). In this case, XPER values would be expressed in dollar terms, and would sum to the total P\&L. 

\section{Framework and performance metrics}
 
\noindent We consider a classification or regression problem involving a target variable, denoted $y$, taking values in $\mathcal{Y}=\{0,1\}$ or $\mathcal{Y}	\subseteq \mathbb{R}$ given the problem considered. The $q$-vector $\mathbf{x}\in \mathcal{X} \subseteq \mathbb{R}^q$ refers to the input features. We denote by $f:\mathbf{x}\rightarrow \hat{y}$ an econometric model or a machine learning algorithm, where $\hat{y}\in \mathcal{Y}$ is either a classification output or a regression output, such as $\hat{y}=f(\mathbf{x})$. 
We impose no constraint on the model: it may be parametric or not, linear or not, a weak learner or an ensemble method, etc. For simplicity, for parametric models we exclude the parameters from the notation, i.e., $f(\mathbf{x})\equiv f(\mathbf{x};\theta)$. The model is estimated (parametric model) or trained (machine learning algorithm) once for all on a training sample $ S_T=\{\mathbf{x}_j,y_j\}^T_{j=1}$. The sample size $T$ is considered as fixed. The trained model, denoted $\hat{f}(.)$, is evaluated on a test sample $S_n=\{\mathbf{x}_i,y_i,\hat{f}(\mathbf{x}_i)\}^n_{i=1}$ with a performance metric (PM).
A PM is defined as a measure used to assess the predictive performance of a model. It quantifies the accuracy and effectiveness of a model’s predictions. Standard PM for classification models include accuracy, precision, recall, F1-score, and AUC, among others. Similarly, for regression models, commonly used PM include MSE, RMSE, MAE, and $R^2$. More generally, any information criteria, loss function, or economic performance indicator can be considered as PM. See Table \ref{all_perfs} in Appendix \ref{PM_examples} for examples of PM.

\begin{definition} \label{PM_def}
    A sample performance metric $\textnormal{PM} \in \Theta \subseteq \mathbb{R}$ associated with the model $\hat{f}(.)$ and the test sample $S_n$ is defined as $\textnormal{PM} = \tilde G_n(y_1,...,y_n;\hat{f}(\mathbf{x_1}),...,\hat{f}%
    (\mathbf{x_n})) = G_n(\mathbf{y};\mathbf{X})$, where $\mathbf{y}=(y_1,...,y_n)^{'}$ and $\mathbf{X}=(\mathbf{x}_1,..,\mathbf{x}_n)^{'}$.
\end{definition}

\noindent For instance, for a linear regression model $\hat{f}(\mathbf{x})=\hat{\beta}_0+\hat{\beta}_1x_1$ and the PM defined as the opposite of the MSE, we have  $ G_n(\mathbf{y};\mathbf{X})=-\frac{1}{n}\sum_{i=1}^n(y_i-\hat{\beta}_0-\hat{\beta}_1x_{i,1})^2$.  We introduce the following three assumptions on the PM.  

\begin{assumption}
    The sample \textnormal{PM} increases over $\Theta$ with the predictive performance.
\end{assumption}

\noindent Assumption 1 simplifies the interpretation of the PM by asserting that higher values indicate more accurate model predictions. For example, since a higher $R^2$ suggests smaller differences between model predictions and target values, $R^2$ can be directly used as a PM. Conversely, if we want to decompose the MSE, we recommend considering the opposite of the MSE and defining PM = $-$MSE. Hence, higher PM values signify smaller prediction-target differences. In this case, a positive XPER value $\phi_j$ always indicates that the variable $x_j$ enhances the model predictive performance.

\begin{assumption} \label{Additive_property}
	The sample \textnormal{PM} satisfies the following additive assumption: 
	\begin{equation} \label{additive}
		G_n(\mathbf{y};\mathbf{X}) =\frac{1}{n}\sum_{i=1}^{n}G(y_i;\mathbf{x}_i;\hat{\delta}_n),
	\end{equation}
	where  $G(y_i;\mathbf{x}_i;\hat{\delta}_n)$ denotes an individual contribution to the PM, and $\hat{\delta}_n$ is a nuisance parameter which depends on the test sample $S_n$.
\end{assumption}

\noindent For simplicity, we only consider models for which the outcome $y_{i}$ for instance $i$ only depends on features $\mathbf{x}_{i}=(x_{i,1},...,x_{i,q})^{'}$. For regression or classification models with cross-sectional interactions (e.g., spatial econometrics model) or time-series dependence, notations have to be adjusted such that $\hat{y}_{i}=\hat{f}\left( \mathbf{w}_{i}\right)$, where $\mathbf{x}_{i}\subseteq \mathbf{w}_{i}$, $\exists j\neq i:\mathbf{x}_{j}\subseteq \mathbf{w}_{i}$ and/or $y_{j}\subseteq \mathbf{w}_{i}$. Then, the additive assumption becomes $G_n(\mathbf{y};\mathbf{X})=\frac{1}{n}\sum_{i=1}^{n}G\left( y_{i};\mathbf{		w}_{i};\hat{\delta}_n\right)$.

\begin{assumption} \label{assumption_convergence}
    The sample \textnormal{PM}, denoted $G_n(\mathbf{y};\mathbf{X})$, converges to the population \textnormal{PM}, denoted $\mathbb{E}_{y,\mathbf{x}}(G(y;\mathbf{x};\delta_0))$, where $\mathbb{E}_{y,\mathbf{x}}(.)$ refers to the expected value with respect to the joint distribution of $y$ and $\mathbf{x}$, and $\delta_0= \mbox{plim } \hat{\delta}_n$. In addition, $\mathbb{E}_{y,\mathbf{x}}(G(y;\mathbf{x};\delta_0))$ exists and is finite.
\end{assumption}

These three assumptions are consistent with a wide range of PMs. In Appendix \ref{PM_examples}, we detail $G\left( y;\mathbf{x};\delta _{0}\right)$ for standard PMs associated with regression or classification models. For instance, in the case of a linear regression model $\hat{f}(\mathbf{x})=\mathbf{x}\hat{\beta}$, when the $R^2$ is used as PM, we have:
\begin{align} \label{R2_def}
    R^2=G_n(\mathbf{y};\mathbf{X})=\frac{1}{n}\sum_{i=1}^{n}G(y_i;\mathbf{x}_i;\hat{\delta}_n)=1-\frac{\sum_{i=1}^{n}(y_i-\mathbf{x}_i\hat{\beta})^2}{\sum_{j=1}^{n}(y_j-\bar{y})^2},
\end{align}
with $G(y_i;\mathbf{x}_i;\hat{\delta}_n)=1-\frac{1}{\hat{\delta}_n}(y_i-\mathbf{x}_i\hat{\beta})^2$ and $\hat{\delta}_n=\frac{1}{n}\sum_{j=1}^{n}(y_j-\bar{y})^2$. 
The corresponding population $R^2$ is defined as $\mathbb{E}_{y,\mathbf{x}}(G(y;\mathbf{x};\delta_0))=1-\frac{1}{\sigma_y^2}\mathbb{E}_{y,\mathbf{x}}\left({(y-\mathbf{x}\hat{\beta})^2}\right)$, with $G(y;\mathbf{x};\delta_0)=1-\frac{1}{\hat{\delta}_0}(y-\mathbf{x}\hat{\beta})^2$ and $\delta_0=\sigma_y^2$ the variance of the target variable. 

This set of assumptions imposes minimal constraints. In particular, Assumption \ref{assumption_convergence} is satisfied under three key conditions related to the performance metric, the sampling assumptions, and the consistency of the estimator $\hat{\delta}_n$. First, regarding the performance metric, the function $G(\cdot)$ needs to be integrable with respect to the joint distribution of $y$ and $\mathbf{x}$. This ensures that $\mathbb{E}_{y, \mathbf{x}}(G(y; \mathbf{x}; \delta_0))$ exists and is finite. Second, depending on the type of data (cross-sectional or time series), specific assumptions about the joint distribution of $y$ and $\mathbf{x}$ are necessary. For cross-sectional data, we assume that the elements $G(y_i; \mathbf{x}_i; \hat{\delta}_n)$ are independently distributed or weakly dependent. This assumption is typically mild since $G(y_i; \mathbf{x}_i; \hat{\delta}_n)$ generally represents a function of the error term. For example, for the (opposite) MSE, we have $G(y_i; \mathbf{x}_i; \hat{\delta}_n) = -(y_i - \hat{f}(\mathbf{x}_i))^2$. This ensures the conditions for applying the weak law of large numbers when $n$ is large. For time series data, we assume that the elements $G(y_i; \mathbf{x}_i; \hat{\delta}_n)$ are stationary and weakly dependent. Finally, the estimator $\hat{\delta}_n$ should be consistent for $\delta_0$. This means $\hat{\delta}_n$ converges in probability to $\delta_0$ as the sample size $n$ increases.

\vspace{0.2cm}

\section{XPER values}
\label{XPER_section}
\medskip
  
\subsection{Definition}
  
Our objective is to identify the contribution of the model's features to its predictive performance, as evaluated by a PM on a given sample. We measure this contribution through Shapley values \citep{Shapley_1953}, a method used in game theory to fairly distribute a payoff $Val(x_1,...,x_q)$ among several players  $x_1,...,x_q$. The Shapley value $\phi_j$ measures the marginal impact of a player $x_j$ on the payoff by assessing the changes in $Val(x_1,...,x_q)$ when this player is added or not to a coalition of players already in the game. Denote by $\mathbf{x}^S$ the vector of players included in a coalition $S$ and $\mathbf{x}^{\overline{S}}$ the vector of excluded players, such that $ \{\mathbf{x}\} = \{\mathbf{x}^{S}\} \cup \{\mathbf{x}^{\overline{S}}\} \cup \{x_j\}$ and $\mathbf{x}=(x_1,...,x_q)$. The Shapley value is then defined as the weighted average of the marginal contributions associated with all possible coalitions $S$.

\begin{definition}[\cite{Shapley_1953}] \label{Shapley_1953_def}
	The Shapley value contribution of player $x_j$ to the payoff is:
	\begin{equation} 
		\phi _{j} = \sum_{S \subseteq \mathcal{P}(\{\mathbf{x}\} \setminus
			\left\{x_{j}\right\})}^{} \omega_{S} \left[Val(\mathbf{x}^{S} \cup \{x_j\}) - Val(\mathbf{x}^{S})\right],
   \label{Shapley_formula}
	\end{equation}
	\begin{equation} 
        \label{Equation_weight}
		\omega_{S} = \frac{|S|!\left(q-|S|-1\right)!}{q!},
	\end{equation}
    with $Val(.)$ the payoff, $S$ a coalition of players, excluding the player of interest $x_{j}$, $|S|$ the number of players in the coalition, and $\mathcal{P}(\{\mathbf{x}\}\setminus \left\{x_{j}\right\})$ the powerset of the set $\{\mathbf{x}\} \setminus \left\{x_{j}\right\}$. 
\end{definition} 
		
By analogy, we decompose the performance metric $\mathbb{E}_{y,\mathbf{x}}(G(y;\mathbf{x};\delta_0))$ (the ``payoff'') among the features $x_1,...,x_q$ (the ``players'') of the model $\hat{f}(\mathbf{x})$ (the ``game''). The main difference with the previous notations is that the payoff $Val(\mathbf{x};y)=\mathbb{E}_{y,\mathbf{x}}(G(y;\mathbf{x};\delta_0))$ depends on the joint distribution of the features $x_1,...,x_q$ and the target variable $y$. 
Formally, we define XPER as the Shapley value associated with the performance metric $\mathbb{E}_{y,\mathbf{x}}(G(y;\mathbf{x};\delta_0))$ and to the model $\hat{f}(\mathbf{x})$. 

\begin{definition}[XPER] \label{XPER_definition}
	The XPER value associated with the feature $x_j$ is defined as:
    \begin{equation*}
        \phi _{j}=\sum_{S\subseteq  \mathcal{P}(\{\mathbf{x}\}\setminus \left\{
      x_{j}\right\} )} \omega_{S}\left[ \underset{\text{averaging}}{%
        \underbrace{\mathbb{E}_{y,x_{j},\mathbf{x}^{S}}}} \hspace{0.1cm}
       \underset{\text{marginalization}}{\underbrace{%
      \mathbb{E}_{\mathbf{x}^{\overline{S}}}}}
       \left( G\left( y;\mathbf{x}
       ;\delta _{0}\right) \right) -\underset{\text{averaging}}{%
      \underbrace{\mathbb{E}_{y,\mathbf{x}^{S}}}} \hspace{0.1cm}
       \underset{\text{marginalization}}{\underbrace{%
      \mathbb{E}_{x_{j}\mathbf{x}^{\overline{S}}}}}
      \left( G\left( y;\mathbf{x}
      ;\delta _{0}\right) \right) \right] ,
       \label{DefGSHAP}
    \end{equation*}
	  with S a coalition of features, excluding the feature of interest $x_{j}$, $\mathcal{P}(\{\mathbf{x}\}\setminus \left\{x_{j}\right\})$ the powerset of the set $\{\mathbf{x}\} \setminus \left\{x_{j}\right\}$, and $\omega_{S}$ the coalition weight. The expectations are taken over the joint distribution of the features and/or target variable specified in the subscripts.
\end{definition}
 
The XPER value $\phi_j$ measures the weighted average marginal contribution of the feature $x_j$ to the PM over all features coalitions. For each coalition, the marginal contribution of $x_j$ is defined as the difference between the expected values of (1) the PM obtained while including this feature in the coalition and (2) the PM obtained while excluding this feature from the coalition. In Definition \ref{DefGSHAP}, the term  $\mathbb{E}_{\mathbf{x}^{\overline{S}}}$ refers to the marginalization with respect to the \textit{joint distribution} of the features which are excluded from the coalition $S$.\footnote{An alternative approach would be to compute the expectation over the conditional distribution of the excluded features, $\mathbf{x}^{\overline{S}}$, given the joint distribution of the target variable $y$, the feature of interest $x_j$, and the features included in the coalition $\mathbf{x}^{S}$, i.e., $\mathbb{E}_{\mathbf{x}^{\overline{S}}|y,x_{j},\mathbf{x}^{S}}$. However, the corresponding decomposition would no longer satisfy the four axioms mentioned in Section \ref{Axioms_Section}. See the discussion about the Conditional Expectation Shapley (CES) methods in \cite{Sundararajan2020}.} The second expectation $\mathbb{E}_{y,x_j,\mathbf{x}^{S}}\left(G\left(y;\mathbf{x};\delta_0\right)\right)$ refers to an averaging effect, i.e., expectation with respect to the \textit{joint distribution} of the features $\mathbf{x}^{S}$ included in the coalition with $x_j$, and the target variable $y$.

\bigskip

	\begin{table}[!htb]
		\centering
		\caption{Components of $\phi_1$ in a three-feature model}
		\begin{threeparttable} 
			\begin{tabular}{ccc}
				\hline
				$S$  & $\omega_{S}$ & $\mathbb{E}_{y,x_1,\mathbf{x}^{S}}\mathbb{E}_{\mathbf{x}^{\overline{S}}}\left(
				G\left(y;\mathbf{x};\delta_0\right)\right)- \mathbb{E}_{y,\mathbf{x}^{S}}\mathbb{E}_{x_1,\mathbf{x}^{\overline{S}}}\left(
				G\left(y;\mathbf{x};\delta_0\right)\right)$ \\
				\hline
				$\left\{\emptyset\right\}$ & $1/3$ &
				$\mathbb{E}_{y,x_1}\mathbb{E}_{x_2,x_3}\left(
				G\left(y;\mathbf{x};\delta_0\right)\right)-\mathbb{E}_{y}\mathbb{E}_{x_1,x_2,x_3}\left(
				G\left(y;\mathbf{x};\delta_0\right)\right)$\\
				$\left\{x_{2}\right\}$ & $1/6$ &
				$\mathbb{E}_{y,x_1,x_2}\mathbb{E}_{x_3}\left(
				G\left(y;\mathbf{x};\delta_0\right)\right)-\mathbb{E}_{y,x_2}\mathbb{E}_{x_1,x_3}\left(
				G\left(y;\mathbf{x};\delta_0\right)\right)$ \\
				$\left\{x_{3}\right\}$ & $1/6$ &  $\mathbb{E}_{y,x_1,x_3}\mathbb{E}_{x_2}\left(
				G\left(y;\mathbf{x};\delta_0\right)\right)-\mathbb{E}_{y,x_3}\mathbb{E}_{x_1,x_2}\left(
				G\left(y;\mathbf{x};\delta_0\right)\right)$ \\
				$\left\{x_{2},x_{3}\right\}$ & $1/3$ & $\mathbb{E}_{y,x_1,x_2,x_3}\left(
				G\left(y;\mathbf{x};\delta_0\right)\right)-\mathbb{E}_{y,x_2,x_3}\mathbb{E}_{x_1}\left(
				G\left(y;\mathbf{x};\delta_0\right)\right)$\\
				\hline
			\end{tabular}
		\end{threeparttable}
		\label{Global_analysis}
        \medskip \medskip 
        
        \justifying \noindent {\footnotesize Note: This table provides details about the XPER value computation, i.e., the coalitions (column 1), the associated weights according to Equation \eqref{Equation_weight} (column 2), and the marginal contributions (column 3).}
        \vspace{-0.25cm}
\end{table}

As an illustration, we consider a three-feature model $\hat{f}(x_1,x_2,x_3)$ and pick $x_1$ as the feature of interest. In Table \ref{Global_analysis}, we report all the coalitions among the set $\{x_2,x_3\}$ (column 1), the associated weights computed according to Equation \eqref{Equation_weight} (column 2), and the marginal contributions (column 3) used to compute the XPER value $\phi_1$. For instance, when considering a coalition with only $x_1$, we first compute the expectation of the PM with respect to the \textit{joint distribution} of excluded features $x_2$ and $x_3$, i.e., $\mathbb{E}_{x_2,x_3}(\tilde{G}(y,\hat{f}(x_1,x_2,x_3)))$. This corresponds to a marginalization of the PM with respect to the features which are excluded from the coalition. Then, we consider the expectation of the PM with respect to the \textit{joint distribution} of the features included in the coalition and the target, i.e., $y$ and $x_1$ in our example. Thus, we get an expected value $\mathbb{E}_{y,x_1}\mathbb{E}_{x_2,x_3}(\tilde{G}(y,\hat{f}(x_1,x_2,x_3)))$, where the first expectation refers to the averaging effect, whereas the second one refers to the marginalization effect. The Shapley value is then computed by summing the weighted differences in the expected PM obtained with or without the feature of interest, for all the coalitions of other features. 

One advantage of our marginalization-based approach is that there is no need to re-estimate any sub-model for each subset of features. We consider the expected value of the PM with respect to the features $\mathbf{x}^{\overline{S}}$ excluded from the coalitions, while leaving the model $\hat{f}(\mathbf{x})$ unchanged. An alternative solution would consist in estimating a submodel for each subset of features. For example, for a model with three features $x_1,x_2,x_3$, the computation of the Shapley value associated with $x_1$ would require to estimate four sub-models, namely without any feature, with $x_2$ only, with $x_3$ only, and with $x_2,x_3$, and then to re-estimate the same sub-models while including $x_1$ as an additional feature. This approach was adopted for instance by \cite{Israeli2007} to decompose the $R^2$ of a linear model. However, re-estimating a submodel with only a subset of the initial features can lead to model specification errors, e.g., omitted variable bias in the case of linear regression. This specification error necessarily distorts the Shapley value and thus may lead to an unreliable decomposition of the performance metric.

XPER is a weighted average of the expected performance improvement resulting from adding the feature of interest to a model (coalition) that ``turns off'' this feature. Therefore, the choice of the weighting scheme is crucial. To compute the XPER values, we utilize the coalition weights derived from Shapley's original definition \citep{Shapley_1953}, as defined in Equation \eqref{Equation_weight}. Given this definition, the weights of coalitions are typically uniform and based on the number of features. This choice offers a significant advantage: it guarantees compliance with the four axioms that define a Shapley value — efficiency, symmetry, linearity, and null effects. Therefore, most XAI methods relying on Shapley values adopt these coalition weights (e.g., \cite{Sundararajan2017, Lundberg2018, Casalicchio_2019, Bowen2020, SundararajanDhamdereAgarwal2020}). For instance, \cite{Lundberg2017} follows this approach when introducing the Shapley value to justify their Shapley Additive Explanation (SHAP) method.

\vspace{0.2cm}

 \subsection{Axioms} \label{Axioms_Section}
	
    The XPER values satisfy different axioms associated with Shapley values. These axioms are particularly relevant in the context of statistical performance analysis.

	\begin{axiom}[Efficiency] \label{efficiency}
     The sum of the XPER values $\phi_j$, $\forall j=1,...,q$, is equal to the difference between the performance metric $\mathbb{E}_{y,\mathbf{x}}(G(y;\mathbf{x};\delta_0))$ and a benchmark $\phi_0$ such as:
	    \begin{equation} \label{Efficiency}
			\mathbb{E}_{y,\mathbf{x}}(G(y;\mathbf{x};\delta_0)) = \phi_0 + \sum_{j=1}^{q}\phi_j,
		\end{equation}
	 where $\phi_0=\mathbb{E}_{\mathbf{x}}\mathbb{E}_{y}\left(G(y;\mathbf{x};\delta_0)\right)$ corresponds to the performance metric associated with a population where the target variable is independent from all features considered in the model.
	\end{axiom}

 \vspace{-0.5cm}
 
One of the main advantages of the XPER decomposition is that the benchmark value $\phi_0$ has an insightful interpretation: it corresponds to the PM that we would obtain on a hypothetical sample in which the target variable $y$ is independent from all model features $\mathbf{x}$, i.e., in a case where the model $\hat{f}(\mathbf{x})$ is fully misspecified. As an example, for the AUC, the benchmark $\phi_0$ corresponds to the AUC associated with a random predictor and is equal to $0.5$. For the sensitivity (true positive rate), the benchmark corresponds to the probability $\Pr(\hat{y}=1)$, for the specificity (true negative rate) the benchmark is $\Pr(\hat{y}=0)$, etc.
Thus, we can decompose any PM into two parts: (i) a base value $\phi_0$ obtained in a hypothetical case where $y$ and $\mathbf{x}$ would be independent, and (ii) a component determined by the XPER feature contributions, which depends on their dependence with the target, i.e., their relevance. 
\vspace{-0.5cm}
\begin{equation*}
\underset{\text{population performance metric}}{\underbrace{\mathbb{E}_{y,\mathbf{x}}(G(y;\mathbf{x};\delta_0)}} = \underset{\text{expected value under independence}}{\underbrace{\mathbb{E}_{y}\mathbb{E}_{\mathbf{x}}(G(y;\mathbf{x};\delta_0)}} +
\underset{\text{feature contributions}}{\underbrace{\sum_{j=1}^{q}\phi_j}.}
\end{equation*}

\noindent The XPER values also satisfy the other main axioms of the Shapley values: 

    \begin{axiom}[Symmetry] If for all subsets $S \subseteq \mathcal{P}(\{\mathbf{x}\} \setminus
			\left\{x_{j},x_{k}\right\})$
    \begin{equation*}
        \mathbb{E}_{y,x_{j},\mathbf{x}^{S}}\mathbb{E}_{\mathbf{x}^{\overline{S}}}\left( G\left( y;\mathbf{x},x_j
       ;\delta _{0}\right) \right) = \mathbb{E}_{y,x_{k},\mathbf{x}^{S}}\mathbb{E}_{\mathbf{x}^{\overline{S}}}\left( G\left( y;\mathbf{x},x_k
       ;\delta _{0}\right) \right),
    \end{equation*}
    then $\phi_{j} = \phi_k$. It means that if two features $x_j$ and $x_k$ contribute equally to the performance of the model then their effects must be the same.
    \end{axiom}

    \begin{axiom}[Linearity] For a performance metric $PM$ such that $PM=PM_1+PM_2$, then
        \begin{equation*}
            \phi_j(PM) = \phi_j(PM_1) + \phi_j(PM_2),
        \end{equation*}
        where $\phi_j(PM)$ refers to the XPER value of feature $x_j$ measuring its effect on the performance metric defined as the sum of $PM_1$ and $PM_2$.
    \end{axiom}

\begin{axiom}[Null effects]  \label{Dummy_property} If for all subsets $S \subseteq \mathcal{P}(\{\mathbf{x}\} \setminus
			\left\{x_{j}\right\})$
   \begin{equation*}\mathbb{E}_{y,x_{j},\mathbf{x}^{S}}\mathbb{E}_{\mathbf{x}^{\overline{S}}}\left( G\left( y;\mathbf{x},x_j
       ;\delta _{0}\right) \right) = \mathbb{E}_{y,\mathbf{x}^{S}}\mathbb{E}_{x_{j},\mathbf{x}^{\overline{S}}}\left( G\left( y;\mathbf{x},x_j
       ;\delta _{0}\right) \right),
   \end{equation*}
   then $\phi_j = 0$. A feature without any impact on the performance of the model $\mathbb{E}_{y,\mathbf{x}}(G(y;\mathbf{x};\delta_0))$ has an XPER value equal to 0.
\end{axiom}



To illustrate these axioms, consider a linear regression model $\hat{f}(\mathbf{x}_i)= \sum_{j=1}^{q}\hat{\beta}_jx_{i,j}$ estimated on a training set, and the $R^2$ as PM. For simplicity, we assume that the Data Generating Process (DGP)
	of the test sample $S_n=\{\mathbf{x}_i,y_i,\hat{f}(\mathbf{x}%
	_i)\}^n_{i=1}$ satisfies $	\mathbb{E}\left(\mathbf{x}\right) = \mu_q 
	\mbox{ and } \mathbb{V}(\mathbf{x}) = \Sigma$ a positive semi-definite matrix, with $\Sigma_{k,j} = \sigma_{x_k,x_j}$ the covariance between feature $x_k$ and $x_j$. Then, the XPER contribution $\phi_{j}$ of feature $x_j$ to the $R^2$ is:
\begin{equation} \label{R2_theory}
        \phi_{j}=\frac{2\hat{\beta}_j\sigma_{y,x_j}}{\sigma^2_y},  \qquad  \quad  \forall j=1,...,q,
\end{equation}
with $\sigma_y^2$ the variance of the target variable and $\sigma_{y,x_j}$ its covariance with feature $x_j$. See Appendix \ref{R2_global_proof} for the proof. Formally, $\phi_j$ depends on the estimated parameter $\hat{\beta}_j$ (training set) and the covariance  (test set) between $x_j$ and the target variable $y$, i.e., $\sigma_{y,x_j}$. If the DGP of the training and test samples are identical, model parameters $\hat{\beta}_j$ and covariance $\sigma_{y,x_j}$ have the same sign, which means $\phi_j > 0$.\footnote{
Note that the XPER values $\phi_j$ can be negative, even if the DGP of the \textit{training} and \textit{test} samples are identical. This occurs in a particular case where the true parameter $\beta_j$ of the feature of interest $x_j$ is null. In this case, the parameter estimate $\hat{\beta}_j$ obtained using the \textit{training sample} can be either positive or negative. Additionally, the estimate of the covariance between the target variable $y$ and the feature $x_j$, denoted $\hat{\sigma}_{y,x_j}$, obtained using the \textit{test sample} can also be either positive or negative, even if the true covariance $\sigma_{y,x_j}$ is null as $\beta_j=0$. Therefore, if the sign of the parameter estimate $\hat{\beta}_j$ and the covariance estimate $\hat{\sigma}_{y,x_j}$ differ, then the XPER value $\phi_j$ can be negative.} Otherwise, XPER values may be negative. Similarly, if $\hat{\beta}_j=0$ (i.e., the variable is useless in the model) or if the feature is uncorrelated with the target variable on the test sample ($\sigma_{y,x_j}=0$), then $\phi_j=0$. Finally, a feature $x_j$ has a larger contribution to the $R^2$ than a feature $x_s$ if $\hat{\beta}_j\sigma_{y,x_j} > \hat{\beta}_s\sigma_{y,x_s}$, meaning that $x_j$ is more related to the target variable than $x_s$ both in-sample (through $\hat{\beta}_j$) and out-of-sample (through $\sigma_{y,x_s}$). 

\newpage

In this simple framework, the XPER values correspond to the contribution of the features to the explained variance, as typically computed in econometrics. For ease of explanation, consider the following DGP:
\begin{equation}
    y = x_1 + 3x_2 + \varepsilon, \quad \mathbb{V}(x_1)=\mathbb{V}(x_2)=1, \quad \mathbb{V}(\varepsilon)=10.
\end{equation}
In this framework, the variance explained by the model is typically decomposed as follows:
        \begin{align} \label{R2_theory_2_}
            R^2 = \frac{\mathbb{V}(\hat{y})}{\mathbb{V}(y)}=\frac{\hat{\beta}_1 \sigma_{y,x_1}}{\sigma_y^2}+\frac{\hat{\beta}_2 \sigma_{y,x_2}}{\sigma_y^2}= \tilde{\phi}_1 + \tilde{\phi}_2 .
        \end{align}
Thus, in our example, the $R^2$ is equal to 50\%, with the contributions of features $x_1$ and $x_2$ to the $R^2$ amounting to 5\% and 45\%, respectively. Similarly, the XPER decomposition is:
\begin{align} \label{R2_theory_2__}
            R^2 = \phi_0 + \frac{2 \hat{\beta}_1\sigma_{y,x_1}}{\sigma^2_{y}}+\frac{2 \hat{\beta}_2\sigma_{y,x_2}}{\sigma^2_{y}}=\phi_0 + \phi_1 + \phi_2 ,
        \end{align}
where $\phi_j$ represents the XPER value associated with the feature $x_j$ and $\phi_0 = -R^2$ the benchmark value.\footnote{Appendix \ref{Appendix:Benchmark} presents a detailed example illustrating the interpretation of the benchmark value $\phi_0$ in the context of a standard linear regression model and the decomposition of the $R^2$.} From Equations \eqref{R2_theory_2_} and \eqref{R2_theory_2__}, we can see that the contribution $\tilde{\phi}_j$ of the feature $x_j$ to the explained variance can be derived from the XPER value by normalizing it with respect to the difference $R^2 - \phi_0$:
\begin{align}
    \tilde{\phi}_j = \frac{\phi_j}{R^2-\phi_0}.
\end{align} 
We generally report normalized XPER values, defined as $\phi_j / (\text{PM} - \phi_0)$, where $\text{PM}$ denotes the performance metric under consideration (e.g., $R^2$). Thus, in this particular setting, the XPER values exactly match the features’ contributions to the explained variance. 

In many cases, analytical expressions for the XPER values $\phi_{j}$ are not available. Indeed, the function $G\left(y;\mathbf{x};\delta_0\right)$ is in general non-linear in $y$ and $\mathbf{x}$, as the model $\hat{f}(.)$ can be highly non-linear in $\mathbf{x}$ (e.g., machine learning model) and/or the performance metric can be non-linear in $y$ and $\mathbf{x}$. 
See Appendix \ref{illustrations} for the XPER values decomposition of several regression and classification performance metrics.



\vspace{0.2cm}

\subsection{Individual XPER values}

\vspace{0.25cm}

\subsubsection{Definition}

\vspace{-0.25cm}
The XPER framework can be used to conduct a global analysis of the model predictive performance through feature contributions $\phi_{j}$, or a local analysis at the individual level. 
	
 \begin{definition} (Individual XPER) \label{def_ind_XPER}
		The individual XPER value $\phi_{i,j}$ associated with individual $i$ and model feature $j$ satisfies $\phi_{j} = \mathbb{E}_{y,\mathbf{x}}(\phi_{i,j}(y_i;\mathbf{x}_i))$, where the random variable $\phi_{i,j}(y_i;\mathbf{x}_i)$ corresponds to  individual $i$ and feature $j$  contribution to the performance metric defined as:
		\begin{equation*}  \label{DefLSHAP}
			\phi _{i,j}(y_i;\mathbf{x}_i) = \sum_{S \subseteq \mathcal{P}(\{\mathbf{x}\} \setminus
				\left\{x_{j}\right\})}^{} w_S \left[\mathbb{E}_{\mathbf{x}^{\overline{S}}}\left(
			G\left(y_i;x_{i,j},\mathbf{x}_i^S,\mathbf{x}^{\overline{S}};\delta_0\right)\right) - \mathbb{E}_{x_j,\mathbf{x}^{\overline{S}}}\left(
			G\left(y_i;x_{i,j},\mathbf{x}_i^S,\mathbf{x}^{\overline{S}};\delta_0\right)\right)\right],
		\end{equation*}
with $\mathbb{E}_{\mathbf{x}^{\overline{S}}}(\cdot)$ the expectation with respect to the joint distribution of the features $\mathbf{x}^{\overline{S}}$ and $\mathbb{E}_{x_j,\mathbf{x}^{\overline{S}}}(\cdot)$ the expectation with respected to the joint distribution of the features $\mathbf{x}^{\overline{S}}$ and $x_j$.
\end{definition}

The individual XPER value can be interpreted as the contribution of individual $i$ and model feature $j$ to the performance metric. For a given realization ($y_i,\mathbf{x}_i$), the corresponding individual contribution to the performance metric can be broken down into:
\begin{align} \label{ind_decompos}
	G(y_i;\mathbf{x}_i;\delta_0) = \phi_{i,0} + \sum_{j=1}^{q}\phi_{i,j},
\end{align}
	where $\phi_{i,j}$ is the realization of $\phi _{i,j}(y_i;\mathbf{x}_i)$ and $\phi_{i,0}$ is the realization of $\phi_{i,0}(y_i)= \mathbb{E}_{\mathbf{x}}(G(y_i;\mathbf{x};\delta_0))$. The benchmark $\phi_{i,0}$ corresponds to the contribution of an individual to the performance metric when the target variable $y_i$ is independent from the features $\mathbf{x}_i$. Therefore, the difference between the individual contribution to the performance metric $G(y_i;\mathbf{x}_i;\delta_0)$ and the individual benchmark $\phi_{i,0}$ is explained by individual XPER values $\phi_{i,j}$. 
		
    As an illustration, consider the $R^2$ of a linear regression. Then, the individual XPER value $\phi_{i,j}$ can be expressed as: 
 \begin{equation}  \label{R2_ind_theory}
    \phi_{i,j} = \sigma^{-2}_y\left(\hat{\beta}_j(x_{i,j} - \mathbb{E}(x_{j}))A -\hat{\beta}_j^2(x_{i,j}^2 - \mathbb{E}(x_{j}^2))  + \sum_{\substack{k=1\\ k \neq j}}^q\hat{\beta}_k\hat{\beta}_j\sigma_{x_k,x_j}   \right),
\end{equation}
with $A=\left( 2y_i -\sum_{\substack{k=1\\ k \neq j}}^q\hat{\beta}_k(x_{i,k} +\mathbb{E}(x_{k})) \right)$. See Appendix \ref{proof_R2_ind_theory} for the proof. The value $\phi_{i,j}$ measures the contribution of feature $x_j$ to $G(y_i;\mathbf{x}_i;\delta_0)=1-\frac{1}{\sigma_y^{2}}(y_i-\hat{f}(\mathbf{x}_i))^2$. A positive individual XPER value $\phi_{i,j}$ means that incorporating the information included in feature $x_j$ improves model prediction for individual $i$ compared to the benchmark, hence increasing $R^2$. This gives rise to several observations. First, if $\hat{\beta}_j=0$, i.e., the feature has no impact on the model outcome, the feature $x_j$ does not have any impact on the $R^2$, for all individuals. Second, the closer the realization of feature $x_{j}$ is to its expected value $\mathbb{E}(x_j)$, the lower is the contribution of this feature to the performance metric, for all individuals. A similar result occurs when $x_j^2$ is close to its expected value.
    Indeed, when the characteristics of an individual are close to the mean values in the population, his or her contribution to the predictive performance of the model is also close to the average contribution of other individuals. We include a numerical example in Appendix \ref{R2_ind_XPER} to provide further intuition.

    \subsubsection{Dealing with model heterogeneity}

    \vspace{-0.25cm}
    Individual XPER values can be particularly useful to detect and understand sources of \textit{model heterogeneity}. In machine learning and econometrics, \textit{model homogeneity} refers to a consistent relationship between features and the target variable across the entire dataset, implying that a single model can adequately capture the underlying patterns and make accurate predictions for all data points. Conversely, \textit{model heterogeneity} arises when this relationship varies across different parts of the dataset, indicating that a single model may fail to capture the diverse patterns within the data. In such cases, different models may be needed to achieve better predictive accuracy, indicating distinct feature-target relationships within subgroups of the data.
    
    To identify these subgroups, we suggest relying on the individual XPER values. Indeed, by definition, individual XPER values measure the contribution of features on individual predictive performance, hence capturing the relationship between the features and the target variable. Intuitively, if a feature's impact on performance varies between two groups, it indicates that the feature does not exert the same influence on the accuracy of the classification or regression model. 

    In practice, model heterogeneity can be revealed by applying standard clustering techniques to individual XPER values, such as the K-Medoids algorithm \citep{park2009simple}. Specifically, we propose the following three-step approach. First, using the original classification or regression model applied to the entire training sample, we calculate the individual XPER values based on the features $\mathbf{x}$ and the true target variable $y$. Second, we apply a clustering algorithm to the XPER values to identify $C$ clusters that reveal the heterogeneity in the model's performance. Third, we estimate a separate model for each of the $C$ subgroups derived from the initial sample. The empirical application presented in Section \ref{Heterogeneity_Emp_App} suggests that this approach could lead to an improvement in the model's accuracy.
 
    \section{Estimation}

    \label{estimation_small}
	In this section, we discuss the estimation procedure for the XPER values  $\phi_{j}$ and $\phi_{i,j}$, based on a test sample $%
	S_n=\{\mathbf{x}_i,y_i,\hat{f}(\mathbf{x}%
	_i)\}^n_{i=1}$. Below, we assume that the model has a small number of features, typically $q<10$.\footnote{The standard estimation framework is only feasible for a small number of features $q$. The computation of the XPER values $\phi_j$ according to Definition \ref{XPER_definition} becomes cumbersome as the number of features increases. Indeed, it requires to consider $q\times 2^{q-1}$ coalitions of features, e.g., $5,120$ for $q = 10$ and $10,485,760$ for $q = 20$, etc.} When the number of features is larger, an approximation of the XPER values is required and we propose a modified version of the Kernel SHAP method of \cite{Lundberg2017} (see Appendix \ref{feasible_XPER} for more details).
 	
	Under Assumption \ref{additive}, the XPER value $\phi_{j}$ can be estimated by a weighted average of individual contributions differences such as:
		\begin{equation}  \label{DefGSHAP_emp} 
			\hat{\phi} _{j} = \sum_{S \subseteq \mathcal{P}(\{\mathbf{x}\} \setminus
				\left\{x_{j}\right\})}^{} w_S \left[\frac{1}{n^2}\sum_{u=1}^{n}\sum_{v=1}^{n}
			G\left(y_v;x_{v,j},\mathbf{x}_v^{S},\mathbf{x}_u^{\overline{S}};\hat{\delta}_n\right) - \frac{1}{n^2}\sum_{u=1}^{n}\sum_{v=1}^{n}
			G\left(y_v;x_{u,j},\mathbf{x}_v^{S},\mathbf{x}_u^{\overline{S}};\hat{\delta}_n\right)\right],
		\end{equation}
		 with $S$ a coalition, i.e., a subset of features, excluding the
		feature of interest $x_{j}$, and $\mathcal{P}(\{\mathbf{x}\}
		\setminus \left\{x_{j}\right\})$ the powerset of the set $\{\mathbf{x}\} \setminus \left\{x_{j}\right\}$. 

\begin{table}[!htb]
		\centering
		\caption{Computation of the XPER value $\hat{\phi}_1$ in a three-feature model}
		\begin{threeparttable} 
			\begin{tabular}{ccc} 
				\hline 
				$S$  & $w_S$ & $\frac{1}{n^2}\sum_{u=1}^{n}\sum_{v=1}^{n}
				G\left(y_v;x_{v,j},\mathbf{x}_v^{S},\mathbf{x}_u^{\overline{S}};\hat{\delta}_n\right)- \frac{1}{n^2}\sum_{u=1}^{n}\sum_{v=1}^{n}
				G\left(y_v;x_{u,j},\mathbf{x}_v^{S},\mathbf{x}_u^{\overline{S}};\hat{\delta}_n\right)$ \\
				\hline \vspace{0.2cm}
				$\left\{\emptyset\right\}$ & $1/3$ &
				$\frac{1}{n^2}\sum_{u=1}^{n}\sum_{v=1}^{n}
				G\left(y_v;x_{v,1},x_{u,2},x_{u,3};\hat{\delta}_n\right)-\frac{1}{n^2}\sum_{u=1}^{n}\sum_{v=1}^{n}
				G\left(y_v;x_{u,1},x_{u,2},x_{u,3};\hat{\delta}_n\right)$ \\ \vspace{0.2cm}
				$\left\{x_{2}\right\}$ & $1/6$ & $\frac{1}{n^2}\sum_{u=1}^{n}\sum_{v=1}^{n}
				G\left(y_v;x_{v,1},x_{v,2},x_{u,3};\hat{\delta}_n\right)-\frac{1}{n^2}\sum_{u=1}^{n}\sum_{v=1}^{n}
				G\left(y_v;x_{u,1},x_{v,2},x_{u,3};\hat{\delta}_n\right)$ \\ \vspace{0.2cm}
				$\left\{x_{3}\right\}$ & $1/6$ &
				$\frac{1}{n^2}\sum_{u=1}^{n}\sum_{v=1}^{n}
				G\left(y_v;x_{v,1},x_{u,2},x_{v,3};\hat{\delta}_n\right)-\frac{1}{n^2}\sum_{u=1}^{n}\sum_{v=1}^{n}
				G\left(y_v;x_{u,1},x_{u,2},x_{v,3};\hat{\delta}_n\right)$ \\ \vspace{0.2cm}
				$\left\{x_{2},x_{3}\right\}$ & $1/3$ & $\frac{1}{n}\sum_{v=1}^{n}
				G\left(y_v;x_{v,1},x_{v,2},x_{v,3};\hat{\delta}_n\right)-\frac{1}{n^2}\sum_{u=1}^{n}\sum_{v=1}^{n}
				G\left(y_v;x_{u,1},x_{v,2},x_{v,3};\hat{\delta}_n\right)$ \\
				\hline
			\end{tabular}%
            \justifying \noindent {\footnotesize
                  \noindent Note: This table presents information on the empirical XPER value computation, i.e., the coalitions (column 1), the associated weights (column 2), and the estimated marginal contributions (column 3).}
		\end{threeparttable}
		\label{Global_analysis_emp}
	\end{table}
    
Table \ref{Global_analysis_emp} illustrates the computation of the estimated value $\hat{\phi}_1$ associated with feature $x_1$ in a model with three features $(x_1,x_2,x_3)$. For each coalition of features (column 1), we report the corresponding weight (column 2) along with the estimated marginal contribution of feature $x_1$ to the performance metric (column 3). The intuition is as follows: the sum over index $u$ refers to the marginalization effect, whereas the sum over index $v$ refers to the averaging effect. For a given instance $v$, we compute its average performance metric by replacing the features ${x}^{\overline{S}}$ which are \textit{not included} in the coalition by the corresponding values observed for all the instances of the test sample (marginalization). Then, we compute the average performance metric for all the instances $v$ (averaging).

	 Similarly to the estimation of features $\phi_{j}$, we can estimate the \textit{individual} XPER values $\phi_{i,j}$.\footnote{Performing inference on XPER values presents significant challenges. \citet{Williamson2020} proposed an approach for statistical inference on a variable importance measure called the ``Shapley Population Variable Importance Measure'' (SPVIM). However, this approach is incompatible with XPER values due to its reliance on sub-model re-estimation and its focus on global rather than individual analyses. Another potential approach involves using bootstrap methods to generate a block bootstrap distribution of XPER values, though the high computational cost remains an important limitation. Finally, conformal prediction \citep{Vovk2005, Romano2019} represents a promising avenue, as it could be adapted to construct prediction intervals for individual XPER values.}  For any individual $i$ of the test sample of $S_n$ and any feature $x_j$, a consistent estimator of the  individual XPER values $\phi_{i,j}$ is defined as:
		\begin{equation}  \label{DefLSHAP_emp} 
			\hat{\phi} _{i,j} = \sum_{S \subseteq \mathcal{P}(\{\mathbf{x}\} \setminus
				\left\{x_{j}\right\})}^{} w_S \left[\frac{1}{n}\sum_{u=1}^{n}
			G\left(y_i;x_{i,j},\mathbf{x}_i^{S},\mathbf{x}_u^{\overline{S}};\hat{\delta}_n\right) - \frac{1}{n}\sum_{u=1}^{n}
			G\left(y_i;x_{u,j},\mathbf{x}_i^{S},\mathbf{x}_u^{\overline{S}};\hat{\delta}_n\right)\right].
		\end{equation}
		 By definition, these individual XPER values satisfy:
		\begin{equation} \label{individual_contribution}
			\hat{\phi}_{j} = \frac{1}{n}\sum_{i=1}^{n}\hat{\phi}_{i,j}.
		\end{equation}

	\section{Simulations}
     In this section, we conduct a Monte Carlo simulation experiment to illustrate the XPER methodology and to highlight some of its properties. In this experiment, XPER values are used to explain the predictive performance of a probit model, measured by the AUC computed on a test set $S_n$. To understand the mechanisms of the XPER decomposition, we consider a white-box model for which the predictions are explainable. The DGP is given by a latent variable model such that $y_i = \mathbf{1}(y^*_i>0)$,	where $\mathbf{1}(.)$ is the indicator function, $y_i^{*} = \mathbf{\omega}_i\beta + \varepsilon_{i}$, $\mathbf{\omega}_i=(1:\mathbf{x}^{'}_i)$ and $\varepsilon_{i}$ an i.i.d. error term with  $\varepsilon_{i}\sim \mathcal{N}(0,1)$. We consider three i.i.d. features such that $\mathbf{x}_i=(x_{i,1},x_{i,2},x_{i,3})^{'} \sim \mathcal{N}(\mathbf{0},\mathbf{\Sigma})$ with $diag(\mathbf{\Sigma})=(1.2,1,1)$. The true vector of parameters is $\beta=(\beta_0,\beta_1,\beta_2,\beta_3)^{'} = (0.05,0.5,0.5,0)^{'}$ with $\beta_0$ the intercept. 
	
	 We simulate $K=5,000$ pseudo-samples $\lbrace y_i^k,\mathbf{x}_i^k \rbrace_{i=1}^{T+n}$ of size 1,000, for $k=1,\dots,K$. For each pseudo-sample, we use the first $T=700$ observations to estimate a probit model and the remaining ones as test set $S_n$ to compute the AUC and the corresponding XPER values according to Equation \eqref{DefGSHAP_emp}. For instance, the estimated parameters obtained for the simulation $k=1$ are equal to $\{\hat{\beta}_0^k,\hat{\beta}_1^k,\hat{\beta}_2^k,\hat{\beta}_3^k\} = \{0.0109,0.4943, 0.5234, 0.0688\}$ and the $AUC^k$ is equal to $0.7775$. For this simulation, the average feature contributions, taken over the individuals of the test set, are the following: \vspace{-1.25cm}
	\begin{equation*} 
		\underbrace{0.7775}_{AUC^k} = 	\underbrace{0.4984}_{\hat{\phi}_0^k} + \underbrace{0.1716}_{\hat{\phi}_1^k} + \underbrace{0.1098}_{\hat{\phi}_{2}^k} + \underbrace{(-0.0023)}_{\hat{\phi}_{3}^k}.
	\end{equation*} 
 
	As expected, the estimated benchmark $\hat{\phi}_{0}^k$ is close to 0.5. As a reminder, an AUC equal to 0.5 is associated with a random predictor. Hence, the benchmark $\hat{\phi}_{0}^k$ corresponds to the AUC of the model that we would obtain on a virtual test sample in which the target variable is independent from all features. The difference between the estimated AUC and this hypothetical benchmark is explained by the feature contributions. The feature with the largest variance ($x_1$) has also the largest contribution to the predictive ability of the model ($0.1716/(0.7775-0.4984) \simeq 62\%$). In contrast, the contribution of $x_3$, which is not used to generate the data ($\beta_3 = 0$), is very close to zero. 

    \bigskip \bigskip 
         \begin{figure}[!h]
            \centering
            \begin{subfigure}[b]{0.49\textwidth}
			\makebox[\linewidth]{
			\includegraphics[width=0.6%
			\textwidth,trim={0 0.8cm 0 0},clip]{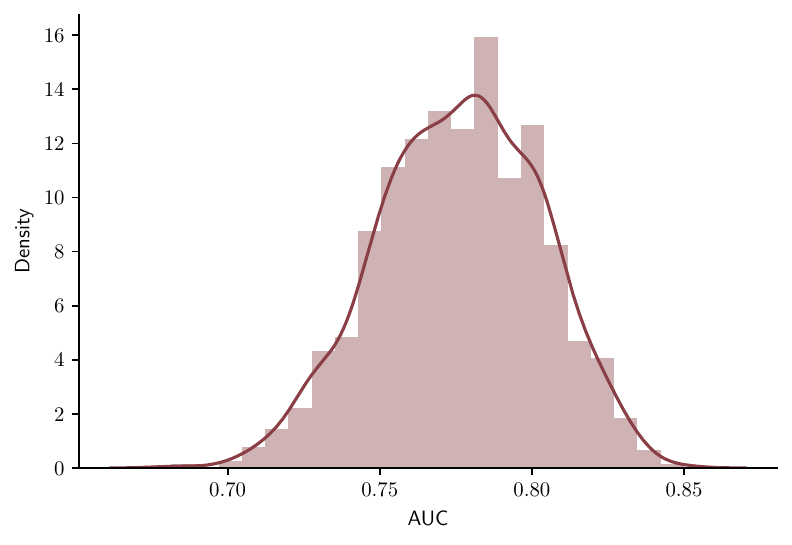}} %
			\caption{AUC}
			\end{subfigure}
			\hfill
			\begin{subfigure}[b]{0.49\textwidth}
			\makebox[\linewidth]{
			\includegraphics[width=0.6%
			\textwidth,trim={0 0.8cm 0 0},clip]{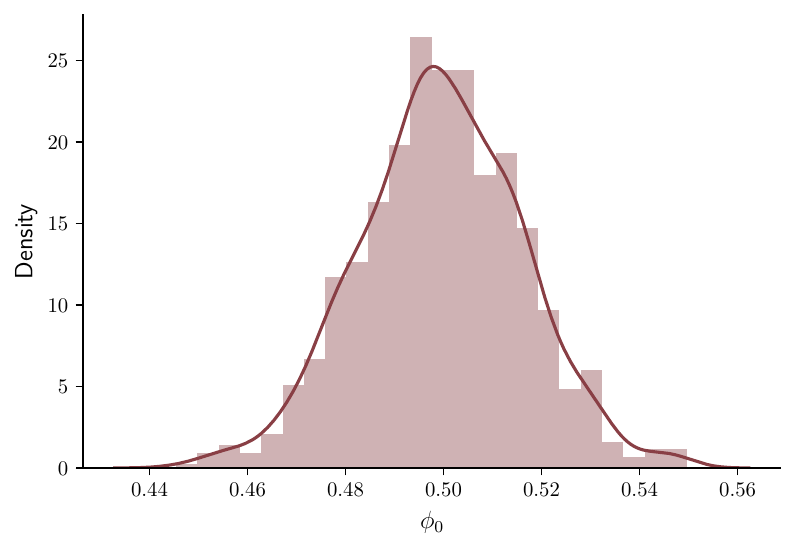}} %
			\caption{$\hat{\phi}_0$}
			\end{subfigure}
			\hfill
			\begin{subfigure}[b]{0.3\textwidth}
			\makebox[\linewidth]{
			\includegraphics[width=1%
			\textwidth,trim={0 0.8cm 0 0},clip]{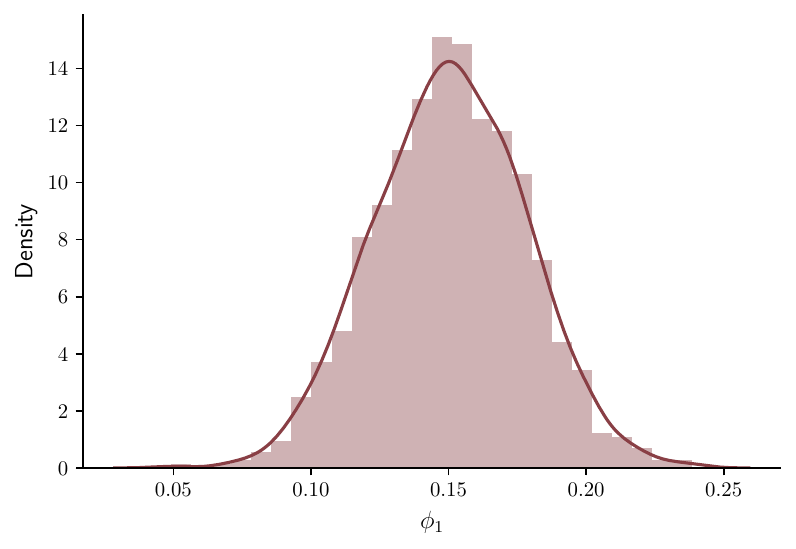}} %
			\caption{$\hat{\phi}_1$}
			\end{subfigure}
			\hfill
			\begin{subfigure}[b]{0.3\textwidth}
			\makebox[\linewidth]{
			\includegraphics[width=1%
			\textwidth,trim={0 0.8cm 0 0},clip]{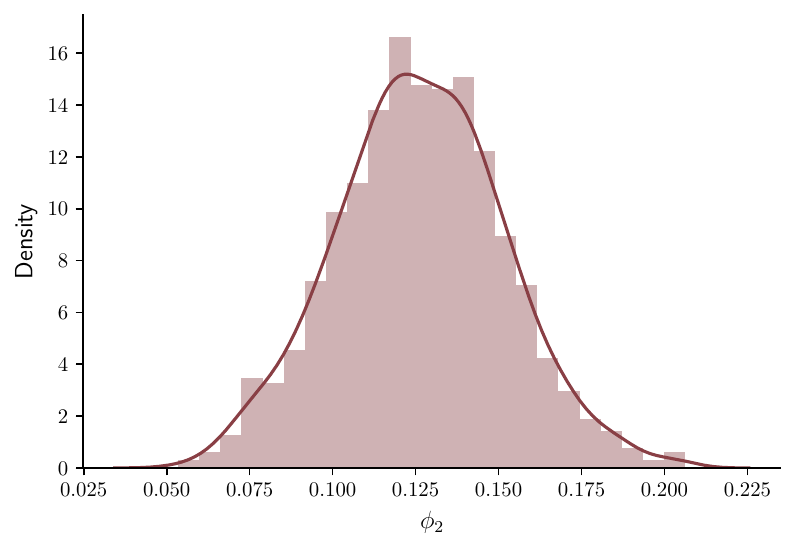}} %
			\caption{$\hat{\phi}_2$}
			\end{subfigure}
			\hfill
			\begin{subfigure}[b]{0.3\textwidth}
			\makebox[\linewidth]{
			\includegraphics[width=1%
			\textwidth,trim={0 0.8cm 0 0},clip]{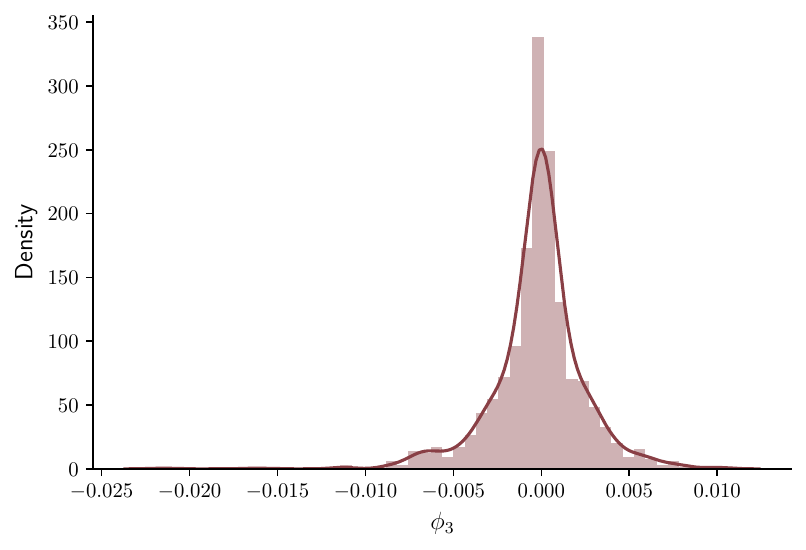}} %
			\caption{$\hat{\phi}_3$}
			\end{subfigure}
              \caption{Empirical distributions of AUC and XPER values}
              \label{illustration_1_AUC}
              \medskip \medskip
               \justifying \noindent {\footnotesize
                  \noindent Note: This figure displays the empirical distributions of the AUC and XPER values on the test sample according to the DGP detailed in Illustration 1. 
        The solid red lines refer to kernel density estimations.}
	\end{figure}
    
    \noindent Figure \ref{illustration_1_AUC} displays the empirical distributions of the AUC, the benchmark values $\hat{\phi}_0$, and the XPER values associated with features $x_1$, $x_2$, and $x_3$, computed from the $K$ simulations. It confirms the robustness of our analysis, but also illustrates the possibility to make inference on XPER values by Bootstrap or other numerical methods.
    \bigskip 
 
    \begin{table}[!h]
	\centering
	\caption{Illustration of the AUC XPER decomposition for simulation $k=1$}
	\label{AUC_phi_i_j} 
	\resizebox{0.8\textwidth}{!}{ 
		\begin{tabular}{cccccccc} \toprule \addlinespace[10pt]
            \multicolumn{8}{c}{Panel A: Local analysis (individual XPER values)} 
			\vspace{7pt} \\ \midrule \addlinespace[5pt]
			{} &  $G(y_i^k;\mathbf{x}_i^k;\hat{\delta}_n^k)$ &   $\hat{\phi}_{i,0}^k$ &   $\hat{\phi}_{i,1}^k$ &   $\hat{\phi}_{i,2}^k$ &  $\hat{\phi}_{i,3}^k$ & $y_i^k$ & $\hat{p}_i^k$
			 \\ \addlinespace[5pt]  
			\hline
    i=1   & 0.9000 & 0.5001 & 0.2450 & 0.1771 & -0.0221 & 1     & 0.6614 \\
     i=2   & 1.0000 & 0.5001 & 0.3975 & 0.1168 & -0.0144 & 1     & 0.8369 \\
    ...   & ... & ... & ... & ... & ... & ...     & ... \\
    i=297 & 0.2533 & 0.4967 & -0.1232 & -0.1237 & 0.0035 & 0     & 0.7616 \\
    ...   & ... & ... & ... & ... & ... & ...     & ... \\
    i=300 & 0.9600 & 0.4967 & 0.1316 & 0.3069 & 0.0249 & 0     & 0.1982 \\
    \toprule \addlinespace[10pt]
            \multicolumn{8}{c}{Panel B: Global analysis (XPER values)} 
			\vspace{7pt} \\ \midrule \addlinespace[7pt] 
            {} &  $AUC^k$ &   $\hat{\phi}_{0}^k$ &   $\hat{\phi}_{1}^k$ &   $\hat{\phi}_{2}^k$ &  $\hat{\phi}_{3}^k$ & $\frac{1}{n}\sum_{i=1}^{n}y_i^k$ & $\frac{1}{n}\sum_{i=1}^{n} \hat{p}_i^k$
			\\ \addlinespace[7pt] \hline
          & 0.7775 & 0.4984 & 0.1716 & 0.1098 & -0.0023 & 0.4967 & 0.4941 \\ \hline
	\end{tabular}}
    
        \medskip  \medskip 
        
        \justifying \noindent {\footnotesize Note: This table presents both the individual XPER values (local analysis, Panel A) and the aggregated XPER values (global analysis, Panel B) associated with the AUC of a Probit model for each of the three features \( x_j \), where \( j = 1, 2, 3 \). The values correspond to those obtained for the test sample (sample size \( n=300 \)) during the first simulation (\( k = 1 \)). The additional columns provide the individual contributions to the AUC, the individual benchmarks, the target variable, and the conditional probabilities $\hat{p}_i^k=\hat{\mathbb{P}}(y_i^k=1|\mathbf{x}_i^k)$ derived from the Probit model.}
 \vspace{-0.5cm}
\end{table}

	In Table \ref{AUC_phi_i_j}, we report the values of the individual XPER values obtained for the first simulation $k=1$. The third column displays the individual benchmark. For a binary classification model, the benchmark $\hat{\phi}_{i,0}\equiv \hat{\phi}_{i,0}(y_i)$ only takes two values. In our simulations, for individuals $y_i=1$ (respectively $y_i=0$), this value is equal to $0.5001$ (respectively $0.4967$). Remind that the individual benchmark corresponds to the contribution to the AUC obtained from a random predictor for an individual with a target value $y_i=1$ or $y_i=0$. Consider the first instance $i=1$ with $y_i^k=1$, $G(y^k_1;\mathbf{x}^k_1;\hat{\delta}^k_n)=0.9$ and $\hat{\phi}_{1,0}=0.5001$. As $G(y^k_1;\mathbf{x}^k_1;\hat{\delta}^k_n) > \hat{\phi}_{1,0}^k$, it means that the features allow the probit model to better predict the event $y_1$ for this instance than a random predictor. 

	 Finally, columns 4, 5, and 6 report the individual XPER values associated with features $x_1$, $x_2$, and $x_3$. First, we verify that feature $x_3$ has close to no impact on the AUC for all instances. Second, the heterogeneity of contributions to the AUC depends on individual characteristics. Remind that at the global level, the contribution of feature $x_1$ is higher than the one of feature $x_2$. However, at the local level, we observe for $i=300$ that the contribution of feature $x_1$ is smaller than the one of feature $x_2$, i.e., $0.1316 < 0.3069$. Third, some individual contributions are negative. For instance, the negative contribution of feature $x_2$ for individual $i=297$ implies that, for this instance, this feature hinders the model from predicting the true target value $y_i^k=0$.

    To further illustrate the XPER methodology, we provide in Appendix \ref{Simulations_Appendix} two additional Monte Carlo simulation experiments illustrating how XPER values can be used to detect the origin of overfitting.
    
	\section{Empirical application}\label{Emp_App}

    \medskip
 
	\subsection{Data and model}

We implement our methodology on a proprietary database of auto loans provided by an international bank. For each borrower, we know whether he or she has eventually defaulted ($y=1$) or not ($y=0$) on the loan. Given the sensitive nature of the data, we had to randomly under-sample individuals to set the default rate to an arbitrary 20\% level. Besides benefits in terms of confidentiality, setting a high arbitrary default rate also protects us against concerns arising from using an unbalanced database. After undersampling, our database includes 7,440 borrowers.\footnote
{Besides undersampling, potential techniques to deal with imbalanced datasets include oversampling (e.g., ADASYN, ROSE, SMOTE) and cost-sensitive learning. For more details, see \citet{Jian2023} and \cite{Chen2024}.} Besides the default target variable, we have access to ten features on the loan (funding amount, funded duration, vehicle price, down-payment) and on the borrower (job tenure, age, marital status, monthly payment in percentage of income, home ownership status, credit event). We divide the database into a stratified training (70\%) and test (30\%) samples to have the same default rate in both subsamples. 

 We provide in Table \ref{tab:summary_stats} some summary statistics about the features and the target variable. In our dataset, a typical loan amounts to around 11,500 euros, finances a 13,000 euro car, and lasts for 56 months. A typical borrower is 45 years old, married, not owning his or her home, has spent nine years in the same job, experienced no credit event over the past six months, provides less than a 50\% down payment, and allocates 10\% of his or her monthly income to reimburse the car loan. To get a first sense of the role of each feature on default, we display in Figure \ref{fig:FeatureDistributions_by_default} their distributions separately for defaulting and non-defaulting borrowers. This preliminary test indicates that the list of discriminating feature includes age, credit event, down payment, marital and ownership status.

 Using the training sample, we estimate an XGBoost model to predict default.
 We selected this model because it is recognized as one of the most powerful scoring engines \citep{gunnarsson_deep_2021}. Another reason for using an XGBoost is its black-box nature. Indeed, while the XPER methodology is model-agnostic, we believe it is interesting to assess its usefulness when used with a particularly complex and opaque algorithm. We select the hyperparameters of the XGBoost using stratified five-fold cross-validation based on a balanced accuracy criterion \citep{bergstra_random_nodate}.

 Table \ref{tab:XGBoost_Performances} displays the values of six performance metrics for the XGBoost model obtained on the training and test samples. Specifically, (1) the AUC measures the discriminatory ability of the model, (2) the Brier Score evaluates the accuracy of the probabilities, and (3) Accuracy, Balanced Accuracy (BA), Sensitivity, and Specificity assess the correctness of the categorical predictions. As shown in Table \ref{tab:XGBoost_Performances}, the XGBoost has an AUC of 0.7521, a Brier Score of 0.1433, and an accuracy of 79.53 on the test sample. Despite the cross-validation of the hyperparameters, we observe some over-fitting in the model as its performances drops slightly from the training sample to the test sample. Overall, the performance metrics of our model are comparable to those reported in \cite{gunnarsson_deep_2021}.

\bigskip

\begin{table}[!h]
\centering
\caption{Model performance}
\resizebox{0.8\textwidth}{!}{ 
\begin{tabular}{c|ccccccc} \hline
    Sample & Size (\%) & AUC & Brier Score  & Accuracy & BA & Sensitivity & Specificity \\ \hline
    Training & 70 & 0.8969 & 0.0958  & 86.98 & 72.43 & 48.18 & 96.69 \\ 
    Test & 30 & 0.7521 & 0.1433  & 79.53 &  58.69 & 23.99 & 93.39\mbox{ } \\ \hline
\end{tabular}}
\label{tab:XGBoost_Performances}
\medskip \medskip
        
        \justifying \noindent {\footnotesize Note: This table displays the performance metrics of the XGBoost model on the training and test samples. BA stands for Balanced Accuracy.}
\end{table}
	\subsection{XPER decomposition}
	We display in Figure \ref{fig:AUC_ESHAP_Empirical} the decomposition of the AUC among the ten features obtained with the estimation method detailed in Appendix \ref{feasible_XPER}. For ease of presentation, we express the feature contributions in percentage of the spread between the AUC and its benchmark (see Equation \eqref{Efficiency}). As shown in Table \ref{tab:XGBoost_Performances}, the AUC in the test sample is equal to 0.7521, which is significantly better than the benchmark value of 0.5 obtained for a random predictor.\footnote{In Appendix \ref{appendix:imbalanced}, we vary the average default rate in the dataset from 5\% to 30\% by randomly removing some default observations or non-default observations from the sample, respectively, to assess the impact of imbalanced data on XPER values.} We see that around 40\% of this over-performance is coming from \textit{funding amount}. The second most contributing feature is \textit{job tenure}, which accounts for another 18\%. It is interesting to note that with only two features, we can explain more than half of the performance of the model. Next, we have five features which contribute each for another 8-10\% to the performance. At the other side of the spectrum, \textit{down payment} is the feature contributing the least to the AUC. On the test sample, this feature does not help the model to better predict default than a random predictor as its XPER value is close to 0. Note that in Appendix \ref{Individual_empirical}, we also analyze the effect of the various features on the performance for each borrower individually.

\medskip
 \begin{figure}[!h]
            \begin{subfigure}[b]{0.5\textwidth}
            \makebox[\linewidth]{
			\includegraphics[width=1%
			\textwidth,trim={0 0 0 0},clip]{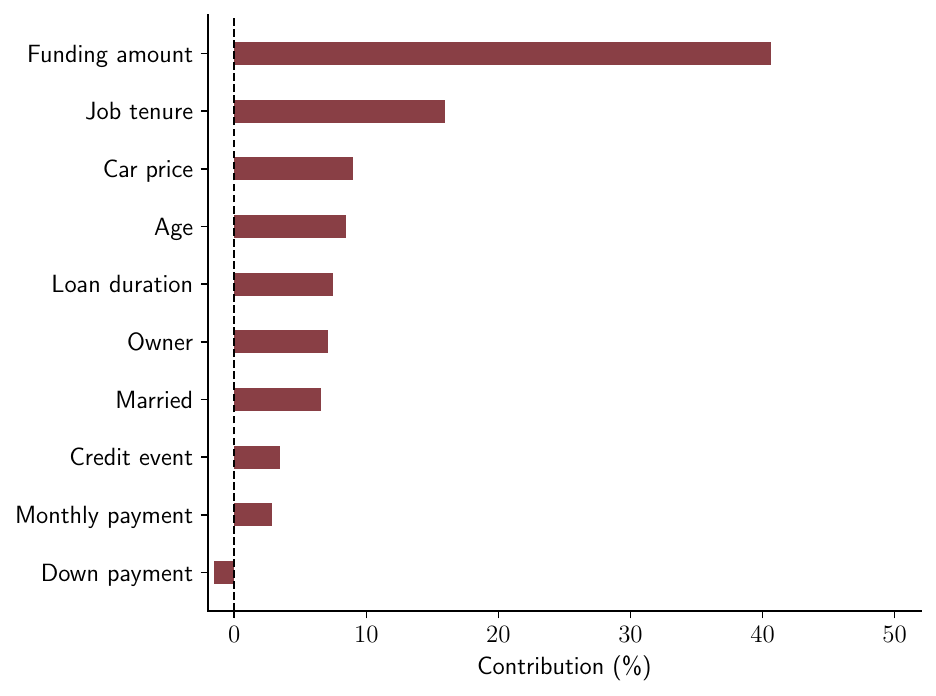}}
		\caption{XPER values}
            \label{fig:AUC_ESHAP_Empirical}
			\end{subfigure}
			\hfill
			\begin{subfigure}[b]{0.5\textwidth}
			\makebox[\linewidth]{
			\includegraphics[width=1%
        			\textwidth,trim={0 0 0 0},clip]{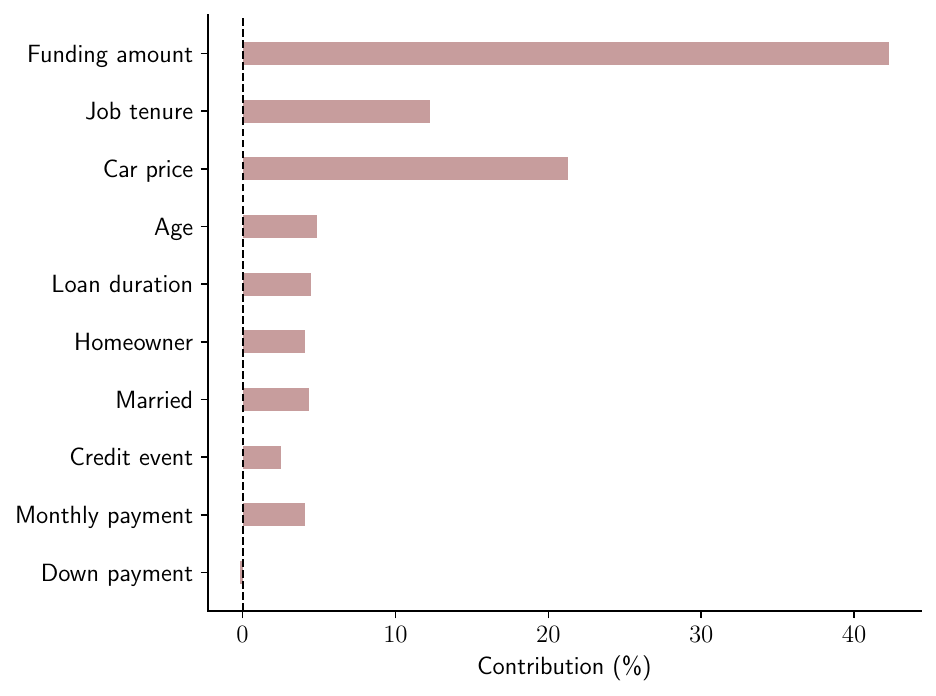}} 
        			\caption{PI-based feature contributions} 
           \label{fig:Permutation_importance}
			\end{subfigure}
   \hfill
    \caption{XPER decomposition and Permutation Importance}
     \medskip \medskip
        
        \justifying \noindent {\footnotesize Note: This figure displays in Panel (a) the XPER values for the AUC of the XGBoost model estimated on the test sample and in Panel (b) the feature contributions for the AUC based on Permutation Importance (PI).}
\vspace{-0.5cm}
	\end{figure}
  
	We now compare the XPER performance decomposition to standard feature contribution methods commonly used in machine learning to assess the impact of a feature on performance or to explain the output of a black-box model, namely Permutation Importance (PI), feature importance, and SHAP values. 

    First, we distinguish between the XPER decomposition of the AUC and the PI method introduced by \cite{breiman_random_2001}. The latter computes for a feature $x_j$, the decrease in the AUC of the model when the values of this feature are randomly reshuffled across instances. As the permutation breaks the dependency between the target variable and the feature, the resulting drop in model AUC indicates how much the model depends on the feature. In our analysis, the PI results indicate the average decrease in AUC when the values are reshuffled 200 times to obtain more robust results. As for the previous methods, we divide each feature contribution by the sum of all features contributions to compare them to the XPER values. We see in Figure \ref{fig:Permutation_importance} that XPER and PI produce different results. For instance, the contribution of the car price is twice as important with PI than with XPER.\footnote{For a more detailed comparison between XPER and PI, see Appendix \ref{XPERvsPI}. We show that PI is a special case of XPER and is only designed for a global analysis.}

 
	 Second, we contrast in Figure \ref{fig:AUC_Feature_imp} the XPER values and the XGBoost-based feature importance. The latter computes for a feature $x_j$, the average gain across all splits where the feature is utilized. 
For ease of comparison, we divide each feature contribution by the sum of the ten feature contributions. As shown in Figure \ref{fig:AUC_Feature_imp}, the result is rather striking. Indeed, the two methodologies lead to very different contributions as some dominating features in a given methodology play a minor role in the other. For instance, \textit{credit event} exhibits the highest feature importance but it is only the $9th$ most contributing feature according to XPER. Differently, funding amount plays a very important role to explain performance but does not contribute much in terms of feature importance. 

	 Third, we compare in Figure \ref{fig:AUC_test_pred_ESHAP} the XPER decomposition of the AUC with the SHAP values introduced by \cite{Lundberg2017}. In our context, the SHAP values assess the impact of the different features on the probabilities of default of each borrower. As it is the norm, we take the average absolute SHAP values for each feature to assess the feature contribution at the model level. As before, we divide each feature contribution by the sum of all features contributions to express them in percentages. In Figure \ref{fig:AUC_test_pred_ESHAP}, we see that SHAP and XPER provide for some features very different information. For instance, the \textit{car price} contribution is more than twice as important for SHAP than for XPER. Similarly, the \textit{funding amount} contribution is around 40\% for XPER whereas only 28\% for SHAP.
\footnote{For a more detailed comparison between XPER and SHAP, see Appendix \ref{XPERvsSHAP}.} 

   \medskip
    \begin{figure}[!h]
            \begin{subfigure}[b]{0.5\textwidth}
            \makebox[\linewidth]{
			\includegraphics[width=1%
			\textwidth,trim={0 0 0 0},clip]{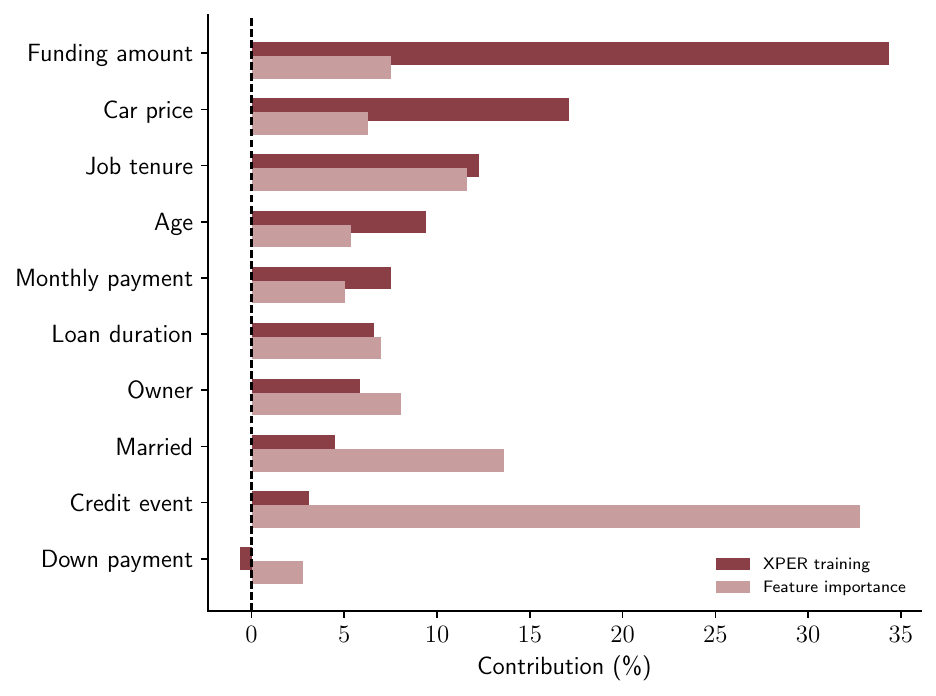}}
		\caption{XPER vs. feature importance}
            \label{fig:AUC_Feature_imp}
			\end{subfigure}
			\hfill
			\begin{subfigure}[b]{0.5\textwidth}
			\makebox[\linewidth]{
			\includegraphics[width=1%
        			\textwidth,trim={0 0 0 0},clip]{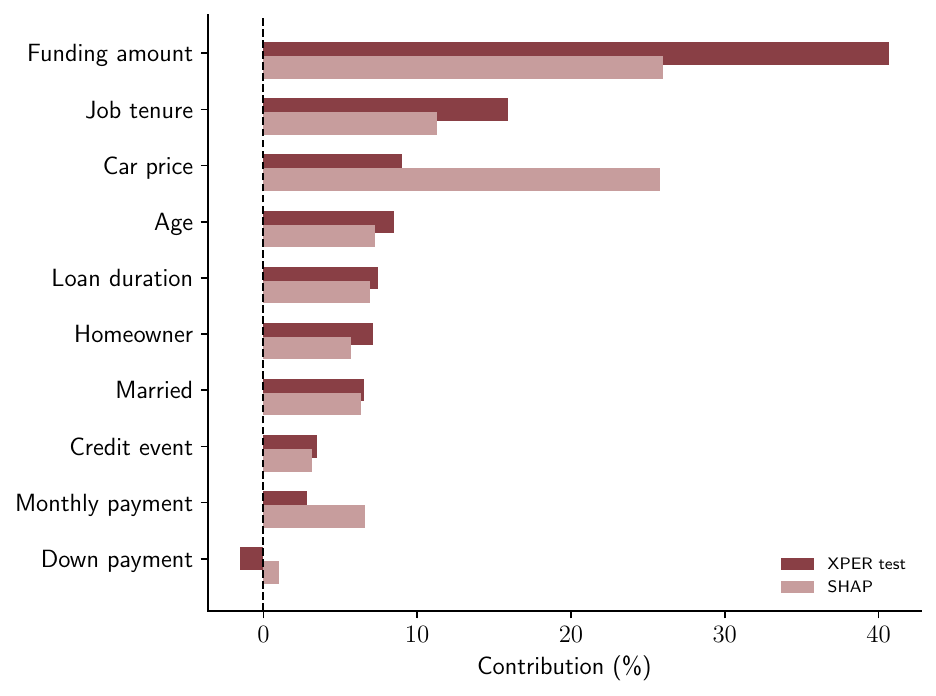}} 
        			\caption{XPER vs. SHAP} 
           \label{fig:AUC_test_pred_ESHAP}
			\end{subfigure}
   \hfill
		\label{AI_per_job_before}
    \caption{XPER vs. other feature contribution methodologies}
     \medskip \medskip
        
        \justifying \noindent {\footnotesize Note: This figure compares the XPER values of the AUC with (a) the XGBoost-based feature importances and (b) the average absolute SHAP values. 
        }
	\end{figure}

    \subsection{Using XPER to boost model performance}\label{Heterogeneity_Emp_App}

We now show how XPER can be used to deal with heterogeneity issues and improve performance. In practice, it is often challenging to estimate a single model able to correctly estimate the relationship between the target variable and the features for all individuals in the sample. In this section, we propose an alternative two-step procedure. First, we build homogeneous groups of individuals displaying similar XPER values using a clustering algorithm. Second, within each group, we estimate a group-specific model. Given the definition of XPER, in each group, the features have a similar impact on the target variable. This alternative approach is likely to increase the model performance compared to the one-fits-all model, which has been used so far in the empirical application. 

In the first step, we estimate the XPER values $\hat{\phi}_{i,j}$ on the training sample and apply the K-Medoids method \citep{park2009simple} to create two clusters of individuals with similar XPER values.\footnote{XPER values are always continuous, even when they are associated with categorical variables. To determine the optimal number of clusters, we used the ``silhouette score'' criterion. We report this score in Table \ref{Tab:Silhouette_scores} in Appendix \ref{appendix_boosting} for a number of clusters ranging from 2 to 10.} We see in Figure \ref{fig:distrib_G1_G2_train} that individuals in group 1 tend to be older, married, home-owners and ask for a mid-range loan to finance a mid-priced car. One striking difference is that in group 2, individuals tend to finance cars with extremely low prices or high prices. Overall, group 1 gathers individuals displaying a lower risk profile compared to those in group 2. This is confirmed by the fact that the default rate is much higher in group 2 (52\%) than in group 1 (13\%). Moreover, as shown in Figures \ref{fig:Feature_distrib_G1_train_default} and \ref{fig:Feature_distrib_G2_train_default} in Appendix \ref{appendix_boosting}, the feature distributions are very different for defaulters and non-defaulters in each group. We also see in Figure \ref{fig:AUC_cluster} that in group 2, using the car price deteriorates the performance of the model as its XPER value is close to -20\% vs. 20\% in group 1. Therefore, these results suggest that estimating group-specific models is likely to boost predictive performance. 

In the second step, we estimate using the training dataset one XGBoost model per group to predict default. Then, we aim to compare the resulting out-of-sample performance with the one of the one-fits-all model. To do so, we first assign each individual of the test sample to either group 1 or group 2 using the clustering rule built on the training sample. 
We see in Table \ref{tab:perf_empirical} that the performance of our strategy (column 2) is significantly higher than the performance of the initial model (column 1). For instance, the AUC of the model on the test sample increases by 16 percentage points, from 0.752 to 0.912. Another notable result is that by using two distinct models, the bank would be able to correctly detect 64\% of the defaulting borrowers against only 24\% with the initial model. Yet, it would still be able to correctly detect more than 95\% of the non-defaulting borrowers.

Finally, we benchmark our strategy with a standard clustering approach based on the features themselves. Specifically, we use the K-prototype algorithm \citep{huang1997clustering} to create two clusters of individuals with comparable feature values.\footnote{We utilize the K-Prototype algorithm instead of the K-Medoids algorithm as the former is specifically designed to handle both categorical and continuous features. More precisely, when handling continuous features, the algorithm calculates the squared Euclidean distance between the features and a representative vector (or prototype) of clusters. For categorical features, it determines the number of mismatches between the features and the cluster’s prototypes.} 
 As shown in column 3, the performance of the standard approach is close to the one of the one-fits-all model (AUC = 0.744 vs. 0.752), and is much lower than the performance of the XPER-based models. 

\clearpage

\begin{figure}[!h]
		\begin{center}
		    			\includegraphics[width=0.7
			\textwidth,trim={0 0 0 0},clip]{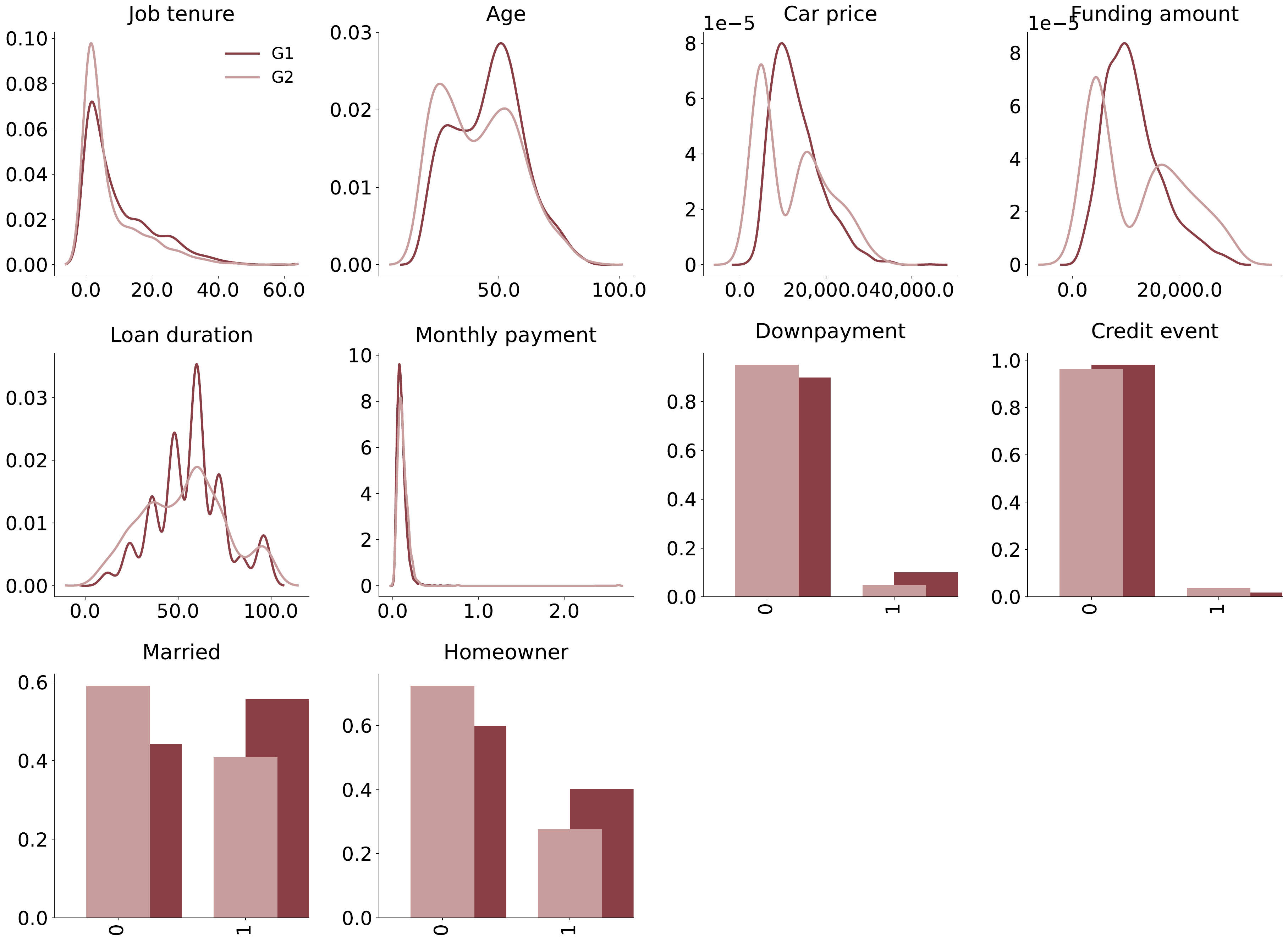}
             \caption{Features distribution by group based on XPER values}
             \label{fig:distrib_G1_G2_train} 
		\end{center}
     \medskip
        
        \justifying \noindent {\footnotesize Note: This figure displays the distribution of the features on the training sample by group created from individual XPER values using the K-Medoids methodology. For continuous features, we use kernel density estimation. Dark red refers to the first group and light red to the second group.}
	\end{figure}
 
\bigskip \bigskip 

\begin{table}[!h]
  \centering
  \caption{Model performances}
  \resizebox{0.8\textwidth}{!}{ 
    \begin{tabular}{lccc}
    \midrule
          & Initial  & Clusters on XPER values & Clusters on features \\
          & (1)  & (2) & (3) \\
    \midrule
    AUC   & 0.752 & 0.912 & 0.744\\
    Brier score & 0.143  & 0.080 & 0.151\\
    Accuracy & 79.53  & 89.11 & 79.53 \\
    Balanced Accuracy & 58.69  & 79.74 & 59.11 \\
    Sensitivity & 23.99  & 64.13 & 25.11\\
    Specificity & 93.39  & 95.35 & 93.11\\
    \midrule
    \end{tabular}}%
    \medskip \medskip
        
        \justifying \noindent {\footnotesize Note: This table displays the performances on the test sample of (1) the initial XGBoost model, (2) our two-step strategy relying on the clustering of XPER values, and (3) the alternative two-step strategy based on the clustering of feature values. For each clustering-based approach, two clusters were derived using the training sample. The performances reported in columns (2) and (3) represent the average performance metrics across the two clusters on the test sample.}
  \label{tab:perf_empirical}%
\end{table}%


\clearpage

	\section{Conclusion}
 
We have introduced XPER, a methodology designed to measure the feature contributions to the performance of any regression or classification model. XPER is built on Shapley values and interpretability tools developed in machine learning but with the distinct objective of focusing on model performance, such as AUC or $R^2$, and not on model predictions, $\hat{y}$. 

Specifically, XPER breaks down the difference between a performance metric of the model and a benchmark value. The latter corresponds to the performance metric that we would obtain on a hypothetical sample in which the target variable is independent from all the features included in the model. As such, the XPER decomposition only focuses on the part of the performance that directly originates from the features. Other advantages of the method include being implementable either at the model level or at the individual level, and not being plagued by omitted variable biases, as it does not require re-estimating the model with  different features coalitions.

We show that identifying the driving forces of the performance of a predictive model is very useful in practice. In a loan default forecasting application, XPER appears to be able to efficiently deal with heterogeneity issues and to boost out-of-sample performance. To do so, we create homogeneous groups of borrowers by clustering them based on their individual XPER values. We find that estimating group-specific models yields to a higher predictive accuracy than with a one-fits-all model.

XPER could go beyond performance. Indeed, as XPER can decompose any function of both $\hat{y}$ and $y$, one could use it to identify the features responsible for the lack of algorithmic fairness of a given machine learning model \citep{hurlin2024fairness}. Another interesting avenue for future research relies on developing inference methods for XPER values, enabling the assessment of their statistical significance.

\baselineskip=1\normalbaselineskip



\clearpage

\bibliographystyle{apalike} 

\bibliography{main} 

\baselineskip=1.5\normalbaselineskip
	\clearpage

\clearpage
	
\appendix
\pagenumbering{arabic}
\renewcommand{\thefigure}{A\arabic{figure}}

\section*{Appendices to ``Measuring the Driving Forces of Predictive Performance: Application to Credit Scoring'' }

	\setcounter{figure}{0}
	
	\renewcommand{\thetable}{A\arabic{table}}
	
	\setcounter{table}{0}

 \setcounter{equation}{0}
 \renewcommand{\theequation}{A\arabic{equation}}

    \section{Performance metric decompositions using Shapley values}

    \label{Literature_review}

    In this appendix, we provide a  literature review including numerous recent studies from computer science and statistics that specifically use Shapley values to measure the individual contributions of features to \textit{model performance}. In contrast, our review does not include articles that aim to decompose model predictions, $\hat{y}$, as this represents a different objective from ours. To facilitate the comparison of our approach with those proposed in the surveyed articles, we present a comparison in Table \ref{Table 1 Litterature Review} across several important dimensions. 
    \medskip
    
    \noindent This comprehensive literature review permits to clearly identify three main contributions of the XPER methodology: 
        \begin{itemize}
            \item XPER stands out for its versatility, enabling the analysis of the drivers of \textit{any} performance metric for \textit{any} estimated model (see paragraph starting by ``First'' below). 
            \item XPER uniquely addresses heterogeneity issues by identifying groups of individuals for which features exhibit similar effects on model performance (see paragraph starting by ``Second'' below). 
            \item XPER offers a meaningful decomposition of the performance metric, achieving this without relying on strong assumptions, unlike many other methods in the literature (see paragraph starting by ``Third'' and the following ones below).
        \end{itemize}
    
        \noindent First, we contrast papers according to the \textit{metrics they decompose}. Six out of the twelve papers considered in this survey focus exclusively on the decomposition of a single performance metric, which is often the $R^2$ for the linear regression model \citep{bell2024efficient,Israeli2007,Lipovetsky2001}. Among the remaining papers, half of them decompose loss functions, which means they are not designed to decompose accuracy metrics (e.g., AUC, Gini, accuracy), goodness of fit ($R^2$), information criterion (AIC, BIC), or economic performance metric (profit-and-loss function). Among these, \cite{Zhang2023} propose an original approach that links changes in model performance to shifts in features distributions, relying on the definition of a causal graph. In contrast, XPER enables the analysis of feature contributions for \textit{any} given performance metric. Unlike XPER, four surveyed papers are model-specific, as shown in Table \ref{Table 1 Litterature Review}.

        \medskip 
        
        \noindent Second, we distinguish existing methodologies by their \textit{level of analysis}, specifically whether they provide a global and/or a local analysis of the features influencing model performance. All but one of the approaches presented in these studies only deliver a global analysis. Indeed, \cite{sutera2021global} is the only considered article that proposes both a local and a global analysis. However, they only do so for tree-ensembles (a single model) and only for the Shannon entropy impurity measure (a single performance metric). As a result, XPER appears to be the only method able to conduct both local or global analyses for any model and any performance metric. The local analysis enabled by XPER is particularly valuable as it addresses heterogeneity issues by identifying groups of individuals for whom the features have similar effects on performance. This capability provides deeper insights and a more nuanced understanding of feature impacts across different subpopulations.
        \medskip
        
        \noindent Third, we classify the approaches based on how they handle features excluded from the feature coalitions used to compute the Shapley values. We identify two main strategies in the literature: (i) re-estimating the model without the excluded features and (ii) marginalizing over the excluded features. While intuitive, the \textit{re-estimating strategy} can result in model specification errors, such as omitted variable bias in linear regression. This concern is particularly relevant here because most of these methods are applied within a linear regression framework to analyze the drivers of $R^2$. A closely related approach is employed by \cite{moehle2021portfolio} as they use a simulator to model the investment process under scenarios where certain drivers are excluded from the decision-making. A limitation of their approach is that the reliability of the derived feature contributions depends heavily on the accuracy of the simulator. 

        \medskip
        
        \noindent The \textit{marginalization strategy} involves marginalizing over the features excluded from the model $f(\cdot)$ to avoid re-estimating it. This marginalization can be carried out in different ways, depending on the feature distribution considered (joint or conditional) and whether the performance metric \(G(\cdot)\) is evaluated at the expected value of the model’s predictions, \(G(\mathbb{E}(f(x)))\), or the expected value of the performance metric itself, \(\mathbb{E}(G(f(x)))\). These distinctions are not merely technical; they have significant implications for the interpretation of the Shapley decomposition.  For instance, \cite{covert2020understanding} and \cite{Borup2022} evaluate the performance metric at the expected value of the model’s predictions. This expected value is computed by integrating over the distribution of the excluded features while treating the included features as fixed.\footnote{Using our notations, they compute $\mathbb{E}_{\mathbf{x}^s,y}\left(\widetilde{G}\left(y;\mathbb{E}_{\mathbf{x}^{\overline{S}}}(\hat{f}(\mathbf{x}));\delta_0\right)\right)$, where $\mathbf{x}^S$ ($\mathbf{x}^{\overline{S}}$) represents the vector of features included in (excluded from) coalition $S$, $y$ is the target variable, $\delta_0$ is a nuisance parameter, and $\widetilde{G}(.)$ is the performance metric comparing the model's prediction $\hat{f}(\mathbf{x})$ to its target value $y$. In this case, the benchmark $\phi_0$ is defined as $\phi_0 = \mathbb{E}_{y}\left(\widetilde{G}\left(y;\mathbb{E}_{\mathbf{x}}(\hat{f}(\mathbf{x}));\delta_0\right)\right)$, which corresponds to the performance that would be obtained if the model were predicting the same constant value for everyone, $\mathbb{E}_{\mathbf{x}}(\hat{f}(\mathbf{x}))$.} 
        In this case, the benchmark is defined as the performance metric, as the MSE for instance, calculated for a restricted model that predicts a \textit{homogeneous} value, \(\mathbb{E}_{\mathbf{x}}(\hat{f}(\mathbf{x}))\), for all instances. The contributions of the features, \(\phi_1, \dots, \phi_q\), then explain the difference between the model's performance and this specific benchmark. For similar reasons, \cite{fryer2021shapley} examine the use of the SAGE \citep{covert2020understanding} and SHAP \citep{Lundberg2017} approaches, concluding that these methods are not well-suited for explaining performance.
         
        \medskip
        
        \noindent In our paper, we designed the marginalization approach to end up with a meaningful benchmark. Specifically, the XPER values explain the difference between the model performance and the performance of a model that would be obtained if the features were independent of the target variable.\footnote{Formally, XPER computes $\mathbb{E}_{\mathbf{x}^s,y}\mathbb{E}_{\mathbf{x}^{\overline{S}}}\left(\widetilde{G}\left(y;\hat{f}(\mathbf{x});\delta_0\right)\right)$. As a consequence, the XPER values $\phi_1,\hdots,\phi_q$ decompose the difference between the model's performance $\mathbb{E}_{\mathbf{x},y}\left(\widetilde{G}\left(y;\hat{f}(\mathbf{x});\delta_0\right)\right)$ and the benchmark value $\phi_0=\mathbb{E}_{y}\mathbb{E}_{\mathbf{x}}\left(\widetilde{G}\left(y;\hat{f}(\mathbf{x});\delta_0\right)\right)$.} Our benchmark is therefore defined as the value of the performance metric in the case where the model is completely misspecified. Our reasoning is analogous to a global Fisher test in standard linear regression, where one compares the MSE of the model to its value for a purely misspecified model in which all variables are insignificant. While we are not conducting a formal test here, the underlying idea is the same. For example, when using the AUC criterion for classification, our benchmark value is $0.5$, which corresponds to the value achieved by purely random classification. Thus, our XPER values are meaningful because they measure the improvement in predictive performance attributable to the features relative to this well-founded benchmark.

        \medskip
        
        \noindent To the best of our knowledge, the Shapley Feature IMPortance (SFIMP) method proposed by \cite{Casalicchio_2019} is the only approach in the literature that performs marginalization in the same manner as we do. However, there are several important differences between SFIMP and XPER. First, XPER enables both local and global analyses of feature contributions to performance, whereas SFIMP can only perform global analyses. Second, while SFIMP decomposes loss functions, XPER is designed to decompose any performance metric: e.g., predictive accuracy (AUC, Gini, accuracy), goodness of fit ($R^2$), information criterion (AIC, BIC), statistical loss function (MSE, MAE, Q-like), or economic performance metric (profit-and-loss function). Third, SFIMP does not specify any assumptions about the set of loss functions it can decompose. In contrast, XPER is built on several assumptions about the performance metric, such as the additivity assumption (Assumption 2). This assumption ensures that XPER can decompose any performance that can be expressed as an average of individual contributions to the performance. Crucially, the introduction of the nuisance parameter $\delta_0$, unique to XPER, extends this framework to handle even complex metrics like the AUC. By reformulating such metrics into an additive structure, XPER significantly expands its scope, allowing it to handle a broader range of performance measures than SFIMP. Fourth, as opposed to SFIMP, XPER puts emphasis on the definition of the benchmark $\phi_0$ as it plays a key role to explain the meaning of the feature contributions derived from the Shapley value decomposition. Our benchmark has a meaningful interpretation: it corresponds to the performance obtained on a hypothetical sample where the target variable $y$ is independent of all model features $\mathbf{x}$—that is, a scenario in which the model $\hat{f}(\mathbf{x})$ is fully misspecified.       
  
     \begin{sidewaystable}[h] 
     \centering 
     \caption{Literature review} 
     \label{Table 1 Litterature Review} 
     \vspace{2mm} 
     
     \resizebox{1\textwidth}{!}{ 
     \renewcommand{\arraystretch}{1.2}
    \begin{tabular}{l|c|c|c|c|c|c|c}
    \toprule
    \multirow{2}{*}{References} & \multirow{2}{*}{Metrics to be decomposed}  & \multirow{2}{*}{Model-agnostic} &  \multirow{2}{*}{Global} & \multirow{2}{*}{Local} & \multirow{2}{*}{Data} & \multirow{2}{*}{Applications} & \multicolumn{1}{c}{Management of excluded features} \\  
     & &   &  & & & & \multicolumn{1}{c}{from the model} \\ \midrule  
    \textcolor{blue}{\cite{bell2024efficient}}     & $R^2$  &  & \checkmark{} & & Synthetic & Numerical experiments &  Re-estimation\\ 
    \textcolor{blue}{\cite{Borup2022}}     & Loss functions  & \checkmark{} & \checkmark{} & & Real & Forecasting inflation & marginalization \\
    \textcolor{blue}{\cite{Casalicchio_2019}}     & Loss functions & \checkmark{}  & \checkmark{} & & Real &  Housing price prediction & marginalization \\ 
    \textcolor{blue}{\cite{covert2020understanding}}     & Loss functions  & \checkmark{} & \checkmark{} & & Real & Many (text classification, credit scoring, etc.) & marginalization\\
    \textcolor{blue}{\cite{fryer2021shapley}}     & Evaluation functions    & \checkmark{} & \checkmark{} & & Synthetic & Numerical experiments &  marginalization\\
    \textcolor{blue}{\cite{ghorbani2019data}}  & Performance metrics  & \checkmark{} &  \checkmark{} &  & Real & Many (disease prediction, spam classification, etc.) & Re-estimation \\
    \textcolor{blue}{\cite{Israeli2007}}     & $R^2$  & & \checkmark{} & & Real & Income prediction &  Re-estimation\\ 
    \textcolor{blue}{\cite{Lipovetsky2001}}     & $R^2$ &  & \checkmark{} & & Real & Customer satisfaction & Re-estimation\\ 
    \textcolor{blue}{\cite{moehle2021portfolio}}     & Performance metrics & \checkmark{} & \checkmark{} & & Real & Portfolio management & Other \\ 
    \textcolor{blue}{\cite{owen_shapley_2017}}    & Explained variance in ANOVA  & \checkmark{}  & \checkmark{} & & None & None & Re-estimation \\ 
    \textcolor{blue}{\cite{sutera2021global}} & Mean Decrease of Impurity   & & \checkmark{} & \checkmark{} & Real & Led and digits problems & Other\\
    \textcolor{blue}{\cite{Verdinelli2023}}  & Mean Squared Error & \checkmark{}   & \checkmark{} & & Synthetic & Numerical experiments & Re-estimation \\
    \textcolor{blue}{\cite{Zhang2023}}  & Loss functions  & \checkmark{} & \checkmark{} & & Real & Mortality and Tumor prediction & Other \\
    \textbf{Our paper}  & Performance metrics & \checkmark{} &  \checkmark{} &  \checkmark{} & Real & Credit Scoring & marginalization \\ 
    \bottomrule
    \end{tabular}
    }
    
 \end{sidewaystable}

    \clearpage

    \section{Examples of performance metrics} \label{PM_examples}

    \begin{table}[htbp]
      \centering
      \caption{Performance metrics}
      \normalsize
      \subcaption*{Panel A: Regression models}
      	\resizebox{1\textwidth}{!}{ 
        \begin{tabular}{lP{7.5cm}P{7.5cm}P{7cm}}
                 \midrule
               \textbf{Metrics} & \multicolumn{1}{c}{$G_n(\mathbf{y},\mathbf{x})$} & \multicolumn{1}{c}{$G(y_i;\mathbf{x}_i;\hat{\delta}_n)$} & \multicolumn{1}{c}{$\hat{\delta}_n$} \\
        \hline
        MAE   &  -$\frac{1}{n}\sum_{i=1}^{n}\left|y_i-\hat{f}(\mathbf{x}_i)\right|$     &   -$\left|y_i-\hat{f}(\mathbf{x}_i)\right|$    &  $\emptyset$ \\ \midrule 
         MSE   &  -$\frac{1}{n}\sum_{i=1}^{n}\left(y_i-\hat{f}(\mathbf{x}_i)\right)^2$     &   -$\left(y_i-\hat{f}(\mathbf{x}_i)\right)^2$    &  $\emptyset$ \\ \midrule
               $R^2$    & $1-\frac{\sum_{i=1}^{n}\left(y_i-\hat{f}(\mathbf{x}%
				_i)\right)^2}{		\sum_{j=1}^{n}(y_j-\bar{y})^2}$      &   $1-\hat{\delta}_n^{-1}\left(y_i-\hat{f}(\mathbf{x}%
		_i)\right)^2$    & $n^{-1}
		\sum_{j=1}^{n}(y_j-\bar{y})^2$ \vspace{0.2cm} \\ \hline 
         \end{tabular}}%
         \bigskip
         \normalsize
         \subcaption*{Panel B: Classification models}
         \resizebox{1\textwidth}{!}{ 
        \begin{tabular}{lP{7.5cm}P{7.5cm}P{7cm}}
                 \midrule
               \textbf{Metrics} & \multicolumn{1}{c}{$G_n(\mathbf{y},\mathbf{x})$} & \multicolumn{1}{c}{$G(y_i;\mathbf{x}_i;\hat{\delta}_n)$} & \multicolumn{1}{c}{$\hat{\delta}_n$} \\
        \hline
         Accuracy & $\frac{1}{n}\sum_{i=1}^{n}\left(y_i\hat{f}(\mathbf{x}_i) + (1-y_i)(1-\hat{f}(\mathbf{x}_i))\right)$      &  $y_i\hat{f}(\mathbf{x}_i) + (1-y_i)(1-\hat{f}(\mathbf{x}_i))$      & $\emptyset$ \\ \midrule
         \multirow{2}[0]{*}{BA}  & \multirow{2}[0]{*}{$\frac{1}{n}\sum_{i=1}^{n}\frac{1}{2}\left[\frac{y_i\hat{f}(\mathbf{x}_i)}{\frac{1}{n}\sum_{j=1}^{n}y_j} + \frac{(1-y_i)(1-\hat{f}(\mathbf{x}_i))}{\frac{1}{n}\sum_{j=1}^{n}(1-y_j)} \right]$} & \multirow{2}[0]{*}{$\frac{1}{2}\left[\hat{\delta}_{n_1}^{-1}\left(y_i\hat{f}(\mathbf{x}_i)\right) + \hat{\delta}_{n_2}^{-1}\left((1-y_i)(1-\hat{f}(\mathbf{x}_i))\right)\right]$} & $\hat{\delta}_{n_1}=\frac{1}{n}\sum_{j=1}^{n}y_j$ \vspace{0.8cm}\\  
            & & & $\hat{\delta}_{n_2}=\frac{1}{n}\sum_{j=1}^{n}(1-y_j)$ \\ \midrule
             Brier score &  -$\frac{1}{n}\sum_{i=1}^{n}\left(y_i - \hat{P}(\mathbf{x}_i)\right)^2$     &  -$\left(y_i - \hat{P}(\mathbf{x}_i)\right)^2$    & $\emptyset$ \\  \midrule
               Precision &  $\frac{1}{n}\sum_{i=1}^{n}\left(\frac{y_i\hat{f}(\mathbf{x}_i)}{\frac{1}{n}\sum_{j=1}^{n}\hat{f}(\mathbf{x}_j)}\right)$     &  $\hat{\delta}_{n}^{-1}y_i\hat{f}(\mathbf{x}_i)$    & $\frac{1}{n}\sum_{j=1}^{n}\hat{f}(\mathbf{x}_j)$ \\  \midrule
               Sensitivity & $\frac{1}{n}\sum_{i=1}^{n}\left(\frac{y_i\hat{f}(\mathbf{x}_i)}{\frac{1}{n}\sum_{j=1}^{n}y_j}\right)$      & $\hat{\delta}_{n}^{-1}y_i\hat{f}(\mathbf{x}_i)$      &  $\frac{1}{n}\sum_{j=1}^{n}y_j$\\ \midrule
               Specificity &  $\frac{1}{n}\sum_{i=1}^{n}\left(\frac{(1-y_i)(1-\hat{f}(\mathbf{x}_i))}{\frac{1}{n}\sum_{j=1}^{n}(1-y_j)}\right)$     &  $\hat{\delta}_{n}^{-1}(1-y_i)(1-\hat{f}(\mathbf{x}_i))$     & $\frac{1}{n}\sum_{j=1}^{n}(1-y_j)$ \\ \midrule
               \multirow{6}[0]{*}{AUC}   &  $\frac{\sum_{i=1}^{n}\sum_{j=1}^{n}(1-y_i)y_jI(\hat{P}(\mathbf{x}_i) < \hat{P}(\mathbf{x}_j))}{\sum_{j=1}^{n}y_j\sum_{j=1}^{n}(1-y_j)}$     & $\left((1-y_i)\times \hat{\delta}_{n_1}(\mathbf{x}_i)\right)\hat{\delta}_{n_2}^{-1}$   &  $\hat{\delta}_{n_1}(\mathbf{x}_i) = \frac{1}{n}\sum_{j=1}^{n}y_jI(\hat{P}(\mathbf{x}_i) < \hat{P}(\mathbf{x}_j))$\\ 
               & \begin{equation*} I(\hat{P}(\mathbf{x}_i) < \hat{P}(\mathbf{x}_j)) =
                    \label{start_2}
                    \left\{
                    \begin{array}{ll}
                             0  & \mbox{ if } \hat{P}(\mathbf{x}_i) > \hat{P}(\mathbf{x}_j) \\ 
                            0.5 & \mbox{ if } \hat{P}(\mathbf{x}_i) = \hat{P}(\mathbf{x}_j) \\ 
                            1  &  \mbox{ if } \hat{P}(\mathbf{x}_i) < \hat{P}(\mathbf{x}_j)
                    \end{array}
                    \right.
                    \end{equation*} & &  \begin{equation*}\hat{\delta}_{n_2} = \frac{1}{n^2}\sum_{j=1}^{n}y_j\sum_{j=1}^{n}(1-y_j)\end{equation*}\\ \hline
        \end{tabular}}%
      \label{all_perfs}%
      \medskip \medskip
        
        \justifying \noindent {\footnotesize Note: This table displays the expression of sample performance metrics $G_n(\mathbf{y},\mathbf{x})$, individual contribution to the sample performance metric $G(y_i;\mathbf{x}_i;\hat{\delta}_n)$, and the corresponding nuisance parameter $\hat{\delta}_n$. We distinguish between the performance metric associated with regression models (Panel A) and those related to classification models (Panel B).}
    \end{table}%

    \clearpage

    \section{Examples of XPER values decomposition}
	\label{illustrations}
    
 We provide several examples in Table \ref{all_dec} of the XPER decomposition of regression and classification performance metrics. For regression models, we consider a linear regression model $\hat{f}(\mathbf{x}_i)= \sum_{j=1}^{q}\hat{\beta}_jx_{i,j}$, where we assume that the DGP generating the test sample $S_n=\{\mathbf{x}_i,y_i,\hat{f}(\mathbf{x}%
	_i)\}^n_{i=1}$ satisfies $	\mathbb{E}\left(\mathbf{x}\right) = 0_q 
	\mbox{ and } \mathbb{V}(\mathbf{x}) = diag(\sigma^2_{x_j}) \mbox{ } \forall j=1,...,q$, and $\mathbb{E}(y)=0$. We denote by $\sigma_y^2$ the variance of the target variable and by $\sigma_{y,x_j}$ the covariance between the feature $x_j$ and the target variable. For classification models, we consider any binary classification model $\hat{f}(\mathbf{x})$, with $\hat{P}(\mathbf{x}) = \hat{\mathbb{P}}(y=1|\mathbf{x})$ the estimated probability of belonging to class 1 ($y=1$). We denote by $\sigma_{y,\hat{f}(\mathbf{x})}$ the covariance between the target variable and the classification output. 
    \bigskip 
    
 \begin{table}[htbp]
      \centering
      \caption{Examples of XPER values decomposition}
      \normalsize
      \subcaption*{Panel A: Regression models}
      	\resizebox{1\textwidth}{!}{ 
        \begin{tabular}{lP{5cm}P{7.5cm}P{3cm}}
                \midrule
               \textbf{Metrics} & \multicolumn{1}{c}{$\mathbb{E}_{y,\mathbf{x}}(G(y;\mathbf{x};\delta_0))$} & \multicolumn{1}{c}{$\phi_0$} & \multicolumn{1}{c}{$\phi_j$} \\
        \hline
        \multicolumn{1}{l}{\multirow{2}[0]{*}{MSE}} & $2\sum_{j=1}^{q}\hat{\beta}_j\sigma_{y,x_j}$   & \multirow{2}[0]{*}{$- \sum_{j=1}^{q}\hat{\beta}_j^2\sigma^2_{x_j} - \sigma_y^2$} & \multirow{2}[0]{*}{$2\hat{\beta}_j\sigma_{y,x_j}$ }  \\
          & $- \sum_{j=1}^{q}\hat{\beta}_j^2\sigma^2_{x_j} - \sigma_y^2$   &       & \vspace{0.2cm} \\   \midrule
          $R^2$ & $\frac{\sigma_{y,\hat{y}}}{\sigma^2_y}$      &  $-\frac{\sigma_{y,\hat{y}}}{\sigma^2_y}$ & $\frac{2\hat{\beta}_j\sigma_{y,x_j}}{\sigma^2_y}$ \vspace{0.2cm} \\  \hline
         \end{tabular}}%
         \bigskip
         \normalsize
         \subcaption*{Panel B: Classification models}
         \resizebox{1\textwidth}{!}{ 
        \begin{tabular}{lP{5cm}P{7.5cm}P{3cm}}
               \midrule
               \textbf{Metrics} & \multicolumn{1}{c}{$\mathbb{E}_{y,\mathbf{x}}(G(y;\mathbf{x};\delta_0))$} & \multicolumn{1}{c}{$\phi_0$} & \multicolumn{1}{c}{$\phi_j$} \\
        \hline
        \multicolumn{1}{l}{\multirow{2}[0]{*}{Accuracy}} & $2 \sigma_{y,\hat{f}(\mathbf{x})} + 2\mathbb{P}(y=1)\hat{P}(\mathbf{x})$   & \multirow{2}[0]{*}{$2\mathbb{P}(y=1)\hat{P}(\mathbf{x}) + 1 - \mathbb{P}(y=1) - \hat{P}(\mathbf{x})$} & \multirow{2}[0]{*}{$\mbox{No closed-form}$}  \\
          & $+ 1 - \mathbb{P}(y=1) - \hat{P}(\mathbf{x})$   &       &  \\  \midrule
               Precision &  $\frac{\sigma_{y,\hat{f}(\mathbf{x})}}{\mathbb{P}(\hat{f}(\mathbf{x})=1)} + \mathbb{P}(y=1)$     &  $\mathbb{P}(y=1)$    & $\mbox{No closed-form}$ \\  \midrule
               Sensitivity & $\frac{\sigma_{y,\hat{f}(\mathbf{x})}}{\mathbb{P}(y=1)} + \mathbb{P}(\hat{f}(\mathbf{x})=1)$ & $\mathbb{P}(\hat{f}(\mathbf{x})=1)$      &  $\mbox{No closed-form}$\\ \midrule
               Specificity &  $\frac{\sigma_{y,\hat{f}(\mathbf{x})}}{\mathbb{P}(y=0)} + \mathbb{P}(\hat{f}(\mathbf{x})=0)$    &  $\mathbb{P}(\hat{f}(\mathbf{x})=0)$     & $\mbox{No closed-form}$ \\ \midrule
               AUC &  $\mbox{No closed-form}$     &  $0.5$     & $\mbox{No closed-form}$ \\ \hline
        \end{tabular}}%
      \label{all_dec}%
      \medskip \medskip
        
        \justifying \noindent {\footnotesize Note: This table displays the expression for population performance metrics $\mathbb{E}_{y,\mathbf{x}}(G(y;\mathbf{x};\delta_0))$, benchmark values $\phi_0$, and XPER values $\phi_j$. We distinguish between the performance metric associated with regression models (Panel A) and those associated with classification models (Panel B). See Appendices \ref{proof_MSE_example}, \ref{proof_R2_example}, and \ref{proof_accuracy_example} for the proofs related to the MSE, $R^2$, and accuracy.}
    \end{table}%

   \clearpage

    \section{Estimation with a large number of features}
	\label{feasible_XPER}
    
	    \noindent Computing the individual XPER value $\phi_{i,j}$, as defined in Definition \ref{def_ind_XPER}, for a given feature $x_j$ requires evaluating all possible coalitions of the other features of the model, $S \subseteq \mathcal{P}(\{\mathbf{x}\} \setminus \{x_{j}\})$. For a model with $q$ features, this involves $2^{(q-1)}$ coalitions to compute $\phi_{i,j}$. To calculate the individual XPER values for all features $\left(\phi_{i,1}, \dots, \phi_{i,q}\right)$, a total of $q \times 2^{(q-1)}$ coalitions must be evaluated. As the number of coalitions grows exponentially with the number of features (e.g., $5,120$ for $q=10$ and $10,485,760$ for $q=20$), computing individual XPER values rapidly becomes computationally prohibitive. This scalability issue is not unique to our methodology but is a fundamental challenge for any approach based on Shapley values.
    \medskip

    \noindent To address this issue, we have adapted the Kernel SHAP methodology proposed by \citet{Lundberg2017} for estimating SHAP values. A key advantage of Kernel SHAP is its model-agnostic nature, making it applicable to any predictive model, unlike model-specific methods such as DeepSHAP or TreeSHAP \citep{ Lundberg2018}. This approach significantly reduces computational complexity by sampling only a subset $K$ of all possible feature coalitions. The approximate XPER values are then derived using a regression-based approximation, which ensures both practical feasibility and accurate estimation of feature contributions.  
    \medskip
    
    \noindent Specifically, we adapt this methodology to estimate the individual XPER values, $\phi_{i,0}, \phi_{i,1}, \dots, \phi_{i,q}$, as defined in Definition \ref{def_ind_XPER}. Recall that the individual XPER value $\phi_{i,j}$, for $j=1, \dots, q$, quantifies the contribution of model feature $j$ for individual $i$ to the performance metric, while $\phi_{i,0}$ represents the benchmark contribution for individual $i$. We approximate these values using a system of $K$ equations derived from the performance metrics associated with the $K$ selected coalitions of features. The estimation process is structured as a linear regression problem, where the individual XPER values are the parameters to be estimated:
    \begin{equation}
        G_{k}\left(y_i;\mathbf{x}_i\right) = \phi_{i,0} + \sum_{j=1}^{q}\phi_{i,j}z_{k,j} + \epsilon_k, \quad \text{for } k=1,\dots,K,
        \label{Eq_Estimation_indi}
    \end{equation}
    where $\epsilon_k$ is an error term and:
    \begin{equation}
        G_{k}\left(y_i;\mathbf{x}_i\right) = \frac{1}{n}\sum_{u=1}^{n} G\left(y_i;\mathbf{x}_i^{S_k},\mathbf{x}_u^{\overline{S}_k};\hat{\delta}_n\right),
    \end{equation}  
    represents the individual contribution to the sample performance metric associated with coalition $S_k$, and $z_{k,j}$ are binary variables indicating the presence ($z_{k,j} = 1$) or absence ($z_{k,j} = 0$) of feature $x_j$ in coalition $k$. For instance, consider a regression model $f(x)$ with $q = 5$ features indexed by $j$, denoted as $x_1, \dots, x_5$, for which we aim to estimate the XPER values associated with the mean squared error (MSE). Let $i$ index an individual in the sample, represented by $(y_i, x_{i,1}, \dots, x_{i,5})$. Now, consider a specific coalition of features, randomly chosen and indexed by $k = 1$, such that:  
    \begin{equation}
        S_1 = \{x_{i,1}, x_{i,3}, x_{i,5}\}, 
    \end{equation}  
    \begin{equation}
        z_{1,1} = 1, \quad z_{1,2} = 0, \quad z_{1,3} = 1, \quad z_{1,4} = 0, \quad z_{1,5} = 1.
    \end{equation}  
    The set of features excluded from this coalition is $\overline{S}_1 = \{x_{i,2}, x_{i,4}\}$, and the individual contribution of observation $i$ to the MSE for this coalition is computed as:  
    \begin{equation}
        G_{1}(y_i, \mathbf{x}_{i}) = \frac{1}{n} \sum_{u=1}^{n} \left( y_i - f(x_{i,1}, x_{u,2}, x_{i,3}, x_{u,4}, x_{i,5}) \right)^2.
    \end{equation}  
    For a second coalition, $S_2 = \{x_{i,2}, x_{i,3}\}$, we have $z_{2,2} = 1$ and $z_{2,3} = 1$, with all other dummy variables set to 0. The individual contribution for this coalition is computed as:
    \begin{equation}
        G_{2}(y_i, \mathbf{x}_{i}) = \frac{1}{n} \sum_{u=1}^{n} \left( y_i - f(x_{u,1}, x_{i,2}, x_{i,3}, x_{u,4}, x_{u,5}) \right)^2.
    \end{equation} 
    By repeating this process for $K$ coalitions, we obtain a series of individual contributions, $G_{1}(y_i,\mathbf{x}_{i})$, $\dots$, $G_{K}(y_i,\mathbf{x}_{i})$, along with $q$ series of dummy variables $\{z_{1,1}, \dots, z_{K,1}\}$, $\dots$, $\{z_{1,q}, \dots, z_{K,q}\}$.
    \medskip
      
    \noindent In a vectorial form, the system in Equation \eqref{Eq_Estimation_indi} can be written as:
    \begin{equation}
        \mathbf{G}\left(y_i;\mathbf{x}_i\right)=\mathbf{Z}\phi_i+\epsilon,
    \end{equation} 
    \begin{equation}
    \underset{\left(q+1\right) \times 1}{\phi_i}=\begin{pmatrix}  
    \phi_{i,0} \\ \phi_{i,1} \\ \vdots 
    \\ \phi_{i,q}
    \end{pmatrix}
     ,~~~
 	\underset{K \times \left(q+1\right)}{\mathbf{Z}} = \begin{pmatrix}
	1 & z_{1,1} & \hdotsfor{1} & z_{1,q} \\ 
	\vdots & \vdots & \ddots & \vdots \\
	1 & z_{K,1} & \hdotsfor{1} & z_{K,q} \\ \end{pmatrix}, ~~~
    \underset{K \times 1}{\mathbf{G\left(y_i;\mathbf{x}_i\right)}} = 
    \begin{pmatrix}
	G_{1}\left(y_i;\mathbf{x}_i\right) \\ 
	\vdots \\ 
	G_{K}\left(y_i;\mathbf{x}_i\right)
    \end{pmatrix}.
    \end{equation}
    
    \noindent The vector of individual XPER values, $\phi_i$, can be estimated using a least squares estimator. However, to accurately approximate the individual XPER values, it is necessary to account for (1) the weights of each coalition of features and (2) the fact that the estimation relies on a subset of coalitions of size $K$. Hence, we estimate Equation \eqref{Eq_Estimation_indi} using a Weighted Least Squares (WLS) estimator as follows:  
    \begin{align*}
        \hat{\boldsymbol{\phi}_i} &= \underset{\phi_i}{\text{argmin}} \sum_{k=1}^{K} \omega_{S_k} \left( G_{k}\left(y_i; \mathbf{x}_i\right) - \left(\phi_{i,0} + \sum_{j=1}^{q} \phi_{i,j} z_{k,j}\right)\right)^2,
    \end{align*}  
    where the weights $\omega_{S_k}$ are defined as:  
    \begin{equation*}
        \omega_{S_k} = \frac{q-1}{\frac{q!}{|S_k|!\left(q-|S_k|\right)!}|S_k|\left(q-|S_k|\right)},
    \end{equation*}  
    with $|S_k|$ denoting the number of features included in coalition $S_k$. The estimates of the individual XPER values are then computed as: 
    \begin{align}
        \hat{\boldsymbol{\phi}_i} = \left(\mathbf{Z}'\boldsymbol{\Omega}\mathbf{Z}\right)^{-1}\mathbf{Z}'\boldsymbol{\Omega}\mathbf{G}\left(y_i;\mathbf{x}_i\right),
    \end{align} 
    where  
    \begin{align*}
        \underset{\left(K \times K\right)}{\boldsymbol{\Omega}} &= \begin{pmatrix}
        \omega_{S_1} & 0 & \hdotsfor{1} & 0 \\ 
        0 & \omega_{S_2} & \hdotsfor{1} & 0 \\
        \vdots & \vdots & \ddots & \vdots \\
        0 & 0 & \hdotsfor{1} & \omega_{S_K} \\ \end{pmatrix}.
    \end{align*} 
    \noindent Estimating the individual XPER values using WLS significantly reduces computational time by limiting the number of coalitions considered and allowing for the simultaneous estimation of contributions for all features. The steps for implementing this approximation are detailed in Algorithm \ref{pseudo_code_individual}.

    \begin{algorithm}
	\caption{Pseudo-code to estimate individual XPER values with a large number of features} \label{pseudo_code_individual}
	\begin{algorithmic}[1]
			\STATE Train the model $f(.)$ using the training sample $\{\mathbf{x}_i, y_i\}^T_{i=1}$. \vspace{0.1cm}
            \STATE Randomly select a coalition $S_k$ from the set of all coalitions $\widetilde{\mathcal{P}}(\{\mathbf{x}\})$, and determine its complement $\overline{S}_k$. \vspace{0.1cm}
            \STATE  Store the dummy variables $z_{k,1}, \dots, z_{k,q}$. \vspace{0.1cm}
            \STATE Select an observation $i = 1$ from the test sample $\{\mathbf{x}_i, y_i, \hat{f}(\mathbf{x}_i)\}^n_{i=1}$. \vspace{0.1cm}
            \STATE Compute the individual contribution of observation $i$ to the performance metric $G_{k}\left(y_i; \mathbf{x}_i\right)$ using $S_k$ and $\overline{S}_k$. \vspace{0.1cm}
            \STATE Calculate the weight $\omega_{S_k}$ for the selected coalition $S_k$. \vspace{0.1cm}
            \STATE Repeat steps 2 through 6 $K$ times to obtain $K$ values for the individual contribution $G_{k}\left(y_i; \mathbf{x}_i\right)$ and the corresponding weights $\omega_{S_k}$, for $k = 1, \dots, K$. \vspace{0.1cm}
            \STATE Estimate the individual XPER values $\phi_{i,j}$ for $j = 1, \dots, q$ for observation $i$ using Weighted Least Squares. \vspace{0.1cm}
            \STATE Repeat steps 3 through 8 for all observations $i = 2, \dots, n$.
		\end{algorithmic}
\end{algorithm}
\medskip
\newpage
\noindent Finally, the global XPER value $\hat{\phi}_j$ for feature $x_j$, as defined in Definition \ref{XPER_definition}, is estimated as the empirical mean of the individual XPER values $\hat{\phi}_{i,j}$, as shown in Equation \eqref{individual_contribution}.

    \clearpage

    \section{XPER vs. Permutation Importance}

    \label{XPERvsPI}
    In this section, we compare the XPER method with the Permutation Importance (PI) method \citep{breiman_random_2001,fisher2019}. Both methods aim to measure the effect of the features on model performance. Therefore, it is important to choose between PI and XPER. In our opinion, there are two key differences between these approaches. First, XPER satisfies the axioms of the Shapley value, unlike PI. These axioms, such as efficiency and the null effects axiom, simplify the interpretation of XPER and ensure its relevance. For example, the sum of the contributions obtained from PI lacks a specific meaning, except in very restricted cases. Second, in the performance context, PI is designed only at the global level and does not account for individual-level analysis, thus excluding the possibility of conducting heterogeneity analysis as done in Section \ref{Heterogeneity_Emp_App}.

    \smallskip
    
    \noindent More fundamentally, XPER encompasses PI. Specifically, it can be shown that XPER values can be expressed as the sum of two terms, one of which is the (normalized) PI. To derive this result, let us formally define PI. Initially developed by \citet{breiman_random_2001} and later generalized by \citet{fisher2019}, PI corresponds to the difference between the model's performance with the original feature values (``original'' performance) and its performance when the feature values are randomly permuted (``permuted'' performance). Formally, the ``original'' performance of the model is given by the population performance metric $PM=\mathbb{E}_{y,\mathbf{x}}\left(G(y,\mathbf{x};\delta_0)\right)$. The ``permuted'' performance metric associated with a feature $x_j$ is $PM_{switch,j}=\mathbb{E}_{y,\mathbf{x}_{-j}}\mathbb{E}_{x_j}\left(G(y,\mathbf{x};\delta_0)\right)$, with $\mathbf{x}_{-j}$ the vector of features excluding the feature $x_j$. This metric represents the expected performance of the model when $x_j$ is replaced with random noise, rendering $x_j$ uninformative about $y$, while keeping the marginal distribution of $x_j$ unchanged \citep{fisher2019}. By taking the difference between the two terms, we obtain the PI value for feature $x_j$:
    \begin{equation} \label{PI_def}
        PI_j \equiv PM - PM_{switch} = \mathbb{E}_{y,\mathbf{x}}\left(G(y,\mathbf{x};\delta_0)\right) - \mathbb{E}_{y,x^S,x^{\overline{S}}}\mathbb{E}_{x_j}\left(G(y,\mathbf{x};\delta_0)\right).
    \end{equation}
    Given this definition, it is possible to show that the XPER value encompasses the PI value, as we have (see Appendix \ref{proof_PI} for the proof):  
    \begin{equation}
       \phi _{j}=\sum_{S\subseteq  \mathcal{P}(\{\mathbf{x}_{-j}\}) \setminus \{\mathbf{x}_{-j}\}} \omega_{S}\left[ \mathbb{E}_{y,x_{j},\mathbf{x}^{S}} \mathbb{E}_{\mathbf{x}^{\overline{S}}}
       \left( G\left( y;\mathbf{x} ;\delta _{0}\right) \right) -\mathbb{E}_{y,\mathbf{x}^{S}}
      \mathbb{E}_{x_{j},\mathbf{x}^{\overline{S}}} \left( G\left( y;\mathbf{x}
      ;\delta _{0}\right) \right) \right] + \frac{PI_j}{q}, 
      \label{Eq : PI and XPER}
    \end{equation}        
    where $q$ denotes the number of model features, and $\mathcal{P}(\{\mathbf{x}_{-j}\}) \setminus \{\mathbf{x}_{-j}\}$ represents the set of all possible feature coalitions excluding the feature $x_j$ and discarding the coalition with all features except $x_j$. For example, in a model with three variables, and for $j=1$, this set includes the coalitions $\{\emptyset\}$, $\{x_2\}$, and $\{x_3\}$. The PI corresponds to the gain in performance associated with adding feature $x_j$ to the coalition of all the model features (excluding $x_j$), corresponding to $\{x_2, x_3\}$ in this example.

    \smallskip

    \noindent Equation \eqref{Eq : PI and XPER} illustrates the similarity between PI and XPER values. Both metrics assess the impact of adding a given feature $x_j$ to a pre-existing coalition of features on the model's performance metric. The key difference is that PI evaluates this impact for only \textit{one specific coalition} that includes all features except $x_j$. In contrast, XPER takes into account all possible feature coalitions, satisfying the axioms of a Shapley value. As a result, PI and XPER coincide in a linear regression model, but generally differ in nonlinear models, as illustrated in Figure \ref{fig:Permutation_importance} of the empirical application.

    \smallskip

    \noindent To illustrate this simple case where XPER and PI deliver the same results, we conduct a Monte Carlo experiment. The DGP is defined by $y_i = x_{i,1} \beta_1 + x_{i,2}^3\beta_2 + \varepsilon_i$, with $\varepsilon_i$ is an i.i.d. error term, and $\mathbf{x}_{i} = (x_{i,1}, x_{i,2})'$ two i.i.d. features. We assume that $\varepsilon_i \sim \mathcal{N}(0,1)$, $\mathbf{x}_{i}\sim \mathcal{N}(\mathbf{0},\mathbf{\Sigma})$ with $diag(\Sigma) = (1,1)$, and $\beta = (2, 1)'$. 
        \noindent We simulate $K = 1,000$ pseudo-samples $\lbrace y_i^s,\mathbf{x}_i^s \rbrace_{i=1}^{T+n}$ of size $4,000$ for $s = 1, \dots, K$. For each pseudo-sample, we use the first $T = 2,000$ observations to estimate a linear regression model involving $x_{i,1}$ and $x_{i,2}^3$, and the remaining $n=2,000$ observations to compute the $R^2$ and the corresponding XPER and PI values.
        Over the $1,000$ replications, the average $R^2$ obtained is equal to $0.9492$, which is extremely close to its theoretical value ($0.95$). A similar result is obtained for estimated coefficients as the average values are equal to $\{\hat{\beta_1}, \hat{\beta_2}\} = \{2.0010, 0.9999\}$. On average, we obtain the following XPER values:
        \begin{equation*}
            \phi_0 = -0.9496, \ \phi_1 = 0.4059, \ \phi_2 = 1.4930.
        \end{equation*}
        We obtain similar contributions with the PI, as on average: 
        \begin{equation*}
            PI_1 = 0.4053, \ PI_2 = 1.4925.
        \end{equation*}

    \clearpage
    
    \section{XPER vs. SHAP}

    \label{XPERvsSHAP}
    
    In this section, we compare the XPER method with the Shapley additive explanation (SHAP) method of \cite{Lundberg2017}. As SHAP has now become ubiquitous in machine learning, we believe it is important to clearly show the added value of XPER over SHAP. In the latter, the contribution of a feature $x_j$ to the predicted value $\hat{f}(\mathbf{x}_i)$ for individual $i$, denoted $\phi _{i,j}^{SHAP}$, is defined as:
    \begin{align}  \label{DefSHAP}
			\phi _{i,j}^{SHAP} = \sum_{S \subseteq \mathcal{P}(\{\mathbf{x}\} \setminus
				\left\{x_{j}\right\})}^{} w_S \left[\mathbb{E}_{\mathbf{x}^{\overline{S}}}\left(
			\hat{f}(x_{i,j},\mathbf{x}_i^S,\mathbf{x}^{\overline{S}})\right) - \mathbb{E}_{x_j,\mathbf{x}^{\overline{S}}}\left(
			\hat{f}(x_{j},\mathbf{x}_i^S,\mathbf{x}^{\overline{S}})\right)\right],
    \end{align}
    \begin{equation} \label{pred_SHAP}
        \hat{f}(\mathbf{x}_i) = \phi_{i,0}^{SHAP} + \sum_{j=1}^q \phi_{i,j}^{SHAP},
    \end{equation}
    \begin{equation} \label{SHAP_bench}
       \phi_{i,0}^{SHAP} = \mathbb{E}\left(\hat{f}(\mathbf{x}_i)\right).
    \end{equation}

    \begin{restatable}{prop}{XPERSHAP}
\label{prop:SHAP_particular_XPER}
SHAP is a particular case of XPER where the individual contribution to the performance metric is equal to the predicted value of the model, $G\left(y_i;\mathbf{x}_i;\delta_0\right) = \hat{f}(\mathbf{x}_i)$.
\end{restatable}

    See Appendix \ref{proof_SHAP_particular_XPER} for the proof.
    As stated in Proposition \ref{prop:SHAP_particular_XPER}, we can show that SHAP is a particular case of XPER where the performance metric does not take into account the target variable. However, as performance metrics generally include at least the target variable to compare the predictions to their true value, SHAP and individual XPER values will differ in most cases.
    However, one may still wonder whether SHAP and individual XPER values provide the same information. If this were to be true, we should find that for a given feature $x_j$ (1) a positive (negative) SHAP value means that this feature has a positive (negative) effect on the performance, and (2) a large SHAP value implies that this feature has a strong impact on performance. Below, we show that none of these statements are true. Indeed, in the former case, depending on the value of the target variable, a positive XPER value is not necessarily associated with a positive SHAP value. Intuitively, if the target variable is positive, a variable can contribute to increase the predicted value of the model ($\phi_{i,j}^{SHAP} > 0$) which can reduce the spread between the predicted value and the target value ($\phi_{i,j} > 0$). However, if the target variable is negative, a variable which contributes to decrease the predicted value of the model ($\phi_{i,j}^{SHAP} < 0$) can also reduce the prediction error of the model ($\phi_{i,j} > 0$). For instance, when the model only includes one variable, if the performance metric is defined as $G(y_i;\hat{f}(\mathbf{x}_i);\delta_0) = -(y_i - \hat{f}(\mathbf{x}_i))^2$, and if we assume that $\mathbb{E}(\hat{f}(x_1))=0$, we can show that:
    \begin{equation} \label{XPER_SHAP_example}
        \phi_{i,1} = 2\hat{\varepsilon}_i\phi_{i,1}^{SHAP} + \left(\phi_{i,1}^{SHAP}\right)^2 + \mathbb{V}\left(\hat{f}(x_{i,1})\right) ,
    \end{equation}
    where $\hat{\varepsilon}_i = y_i - \hat{f}(x_{i,1})$ is the prediction error of the model for individual $i$. See Appendix \ref{proof_XPER_SHAP_example} for the proof.
    As we can see, a positive XPER value can be associated with either a positive or negative SHAP value. If the prediction error and the SHAP value are both positive (negative), it means that this variable contributes to reduce the spread between the predicted value and the target value, which results in a positive XPER value. Therefore, a positive and a negative SHAP value can both lead to a positive XPER value. Moreover, since the individual performance and the model predictions do not share the same domain, except in the case where $G\left(y_i;\mathbf{x}_i;\delta_0\right) = \hat{f}(\mathbf{x}_i)$, the magnitude of SHAP and XPER values can differ significantly.
    
    Although designed at the individual level, the SHAP method is also used to assess the effect of the features at the global level by taking the mean of the absolute SHAP values for each feature. We have:
    \begin{equation}
        \phi_{j}^{SHAP} = \mathbb{E}\left( \left|\phi_{i,j}^{SHAP}\right|\right),
    \end{equation}
    where $\phi_{j}^{SHAP}$ refers to the Mean Absolute SHAP values for feature $j$. This raises again the question of how similar these values are from the XPER values. One major difference between SHAP and XPER is that the SHAP values cannot be negative at the global level. This difference turns out to be important as negative XPER values allow to red-flag features that deteriorate the performance of the model on a given sample. For instance, this could help to explain the origin of the overfitting of a model as mentioned in Section \ref{Primer}. Note that according to Equations \eqref{pred_SHAP} and \eqref{SHAP_bench}, it would not be appropriate to take the mean of the SHAP values without taking the absolute value, as we would end up decomposing a zero:
    \begin{align}
        \mathbb{E}\left(\hat{f}(\mathbf{x}_i)\right) &= \phi_{i,0}^{SHAP} + \sum_{j=1}^q \mathbb{E}\left(\phi_{i,j}^{SHAP}\right), \nonumber \\
          \Leftrightarrow \sum_{j=1}^q \mathbb{E}\left(\phi_{i,j}^{SHAP}\right) &= \mathbb{E}\left(\hat{f}(\mathbf{x}_i)\right) - \mathbb{E}\left(\hat{f}(\mathbf{x}_i)\right) = 0.
    \end{align}
    
    Moreover, as both methods provide different results at the individual level, there is no reason to think that they would be equivalent at the global level. In the empirical study in Section \ref{Emp_App}, we confirm that the differences between XPER and SHAP can be substantial in practice.

    \clearpage

    \section{Proofs}
    
    \subsection{Proof of Equation \eqref{R2_theory}} \label{R2_global_proof}
    \begin{lemma} \label{sum_weight}
    The sum of the weights across all coalitions $S$ is equal to 1, 
 $\sum_{S \subseteq \mathcal{P}(\{\mathbf{x}\} \setminus
				\left\{x_{j}\right\})}^{}\omega_S = 1$.
\end{lemma}
\begin{proof} According to the definition of $\omega_{S}$ in Equation \eqref{Equation_weight} and knowing that $\mathcal{P}(\{\mathbf{x}\} \setminus
				\left\{x_{j}\right\}) = \bigcup_{k=0}^{q-1} \mathcal{P}_k(\{\mathbf{x}\} \setminus
				\left\{x_{j}\right\})$, where $\mathcal{P}_k(\{\mathbf{x}\} \setminus
				\left\{x_{j}\right\})$ refers to the collection of all subsets of size $k$ that can be formed from the powerset $\mathcal{P}(\{\mathbf{x}\} \setminus
				\left\{x_{j}\right\})$, we have: 
\begin{align*}
    \sum_{S \subseteq \mathcal{P}(\{\mathbf{x}\} \setminus
				\left\{x_{j}\right\})}^{}\omega_{S} &= \sum_{S \subseteq \bigcup_{k=0}^{q-1} \mathcal{P}_k(\{\mathbf{x}\} \setminus
				\left\{x_{j}\right\})}^{}\frac{1}{q \times C^{|S|}_{q-1}}.
\end{align*}
with $C^{|S|}_{q-1}$ the number of $|S|\mbox{-combinations}$ of a set with $q-1$ elements.
As $\mathcal{P}_k(\{\mathbf{x}\} \setminus
				\left\{x_{j}\right\}) \cap \mathcal{P}_l(\{\mathbf{x}\} \setminus
				\left\{x_{j}\right\}) = \emptyset, ~ \forall k \neq l$ we derive that:
\begin{align*}
    \sum_{S \subseteq \mathcal{P}(\{\mathbf{x}\} \setminus
				\left\{x_{j}\right\})}^{}\omega_{S} &= \sum_{k=0}^{q-1} \sum_{S \subseteq \mathcal{P}_k(\{\mathbf{x}\} \setminus
				\left\{x_{j}\right\})}^{} \frac{1}{q \times C^{|S|}_{q-1}}.
\end{align*}
For each $k$, $\mathcal{P}_k(\{\mathbf{x}\} \setminus
				\left\{x_{j}\right\})$ is composed of $C^{k}_{q-1}$ subsets of size $k$. As $S \subseteq \mathcal{P}_k(\{\mathbf{x}\} \setminus
				\left\{x_{j}\right\}$, we know that $|S| = k$, which implies that $C^{|S|}_{q-1} = C^{k}_{q-1}$. Thus,
\begin{align*}
    \sum_{S \subseteq \mathcal{P}(\{\mathbf{x}\} \setminus
				\left\{x_{j}\right\})}^{}\omega_{S} = \sum_{k=0}^{q-1} C^{k}_{q-1} \frac{1}{q \times C^{k}_{q-1}} = \sum_{k=0}^{q-1}\frac{1}{q} = 1.
\end{align*}    
\end{proof}
\vspace{-1cm}
\begin{lemma} \label{lemma_R2}
    Consider a linear regression model $\hat{f}(\mathbf{x})= \sum_{j=1}^{q}\hat{\beta}_jx_{j}$ where we assume that the DGP
	of the test sample $S_n=\{\mathbf{x}_i,y_i,\hat{f}(\mathbf{x}%
	_i)\}^n_{i=1}$ satisfies $	\mathbb{E}\left(\mathbf{x}\right) = \mu_q 
	\mbox{ and } \mathbb{V}(\mathbf{x}) = \Sigma$ a positive semi-definite matrix, with $\Sigma_{k,j} = \sigma_{x_k,x_j}$ the covariance between feature $x_k$ and $x_j$. The individual contribution to the $R^2$, for a coalition $S \subseteq \mathcal{P}(\{\mathbf{x}\} \setminus\left\{x_{j}\right\})$, can be expressed as:  
        \begin{align*}
    G(y;\mathbf{x}) &= 1 - \sigma^{-2}_y\left[y^2 + \sum_{\substack{l=1 \\ l \in S}}^qx_l^2\hat{\beta}_l^2 + \sum_{\substack{k=1 \\ k \in \overline{S}}}^qx_k^2\hat{\beta}_k^2 + x_j^2\hat{\beta}_j^2 + 2\sum_{\substack{1 \le k < l \le q\\ k,l \in S}}\hat{\beta}_k\hat{\beta}_lx_kx_l + 2\sum_{\substack{1 \le k < l \le q\\ k,l \in \overline{S}}}\hat{\beta}_k\hat{\beta}_lx_kx_l\right] \\
    &- \sigma^{-2}_y\left[2\sum_{\substack{k=1\\ k \in S}}^q\hat{\beta}_k\hat{\beta}_jx_kx_j + 2\sum_{\substack{k=1\\ k \in \overline{S}}}^q\hat{\beta}_k\hat{\beta}_jx_kx_j + 2 \sum_{\substack{k=1 \\ k \in S}}^q \sum_{\substack{l=1 \\ l \in \overline{S}}}^q\hat{\beta}_k\hat{\beta}_lx_kx_l\right] \\ 
    &- \sigma^{-2}_y\left[- 2\sum_{\substack{k=1\\ k \in S}}^qyx_k\hat{\beta}_k - 2\sum_{\substack{k=1\\ k \in \overline{S}}}^qyx_k\hat{\beta}_k - 2yx_j\hat{\beta}_j \right].
\end{align*}
\end{lemma}
\begin{proof} Consider a linear regression model $\hat{f}(\mathbf{x})= \sum_{j=1}^{q}\hat{\beta}_jx_{j}$ where we assume that the DGP of the test sample $S_n=\{\mathbf{x}_i,y_i,\hat{f}(\mathbf{x}%
	_i)\}^n_{i=1}$ satisfies $	\mathbb{E}\left(\mathbf{x}\right) = \mu_q 
	\mbox{ and } \mathbb{V}(\mathbf{x}) = \Sigma$ a positive semi-definite matrix, with $\Sigma_{k,j} = \sigma_{x_k,x_j}$ the covariance between feature $x_k$ and $x_j$.
 
    Reminds that for a linear regression model $\hat{f}(\mathbf{x})= \sum_{j=1}^{q}\hat{\beta}_jx_{j}$, the individual contribution to the $R^2$ is defined as (see Equation \eqref{R2_def}):
    \begin{align}
        G(y;\mathbf{x}) &= 1 - \frac{(y - \mathbf{x}\hat{\beta})^2}{\sigma^2_y}\nonumber \\
        & = 1 - \sigma^{-2}_y\left[y^2 + \sum_{k=1}^qx_k^2\hat{\beta}_k^2 + 2\sum_{1 \le k < l \le q}\hat{\beta}_k\hat{\beta}_lx_kx_l - 2\sum_{k=1}^qyx_k\hat{\beta}_k\right]. \label{R2_individual}
    \end{align}
    Considering a coalition $S \subseteq \mathcal{P}(\{\mathbf{x}\} \setminus\left\{x_{j}\right\})$ of features, the vector of features $\mathbf{x}$ is composed of three sub-vectors: $\mathbf{x}^S$ the vector of features in the coalition $S$, $\mathbf{x}^{\overline{S}}$ the vector of features apart from the coalition, and $x_j$ the remaining feature of interest, such that $\mathbf{x} = (\mathbf{x}^S,\mathbf{x}^{\overline{S}},x_j)$.  Therefore, we can rewrite Equation \eqref{R2_individual} as:
    \begin{align*}
    G(y;\mathbf{x}) &= 1 - \sigma^{-2}_y\left[y^2 + \sum_{\substack{l=1 \\ l \in S}}^qx_l^2\hat{\beta}_l^2 + \sum_{\substack{k=1 \\ k \in \overline{S}}}^qx_k^2\hat{\beta}_k^2 + x_j^2\hat{\beta}_j^2 + 2\sum_{\substack{1 \le k < l \le q\\ k,l \in S}}\hat{\beta}_k\hat{\beta}_lx_kx_l + 2\sum_{\substack{1 \le k < l \le q\\ k,l \in \overline{S}}}\hat{\beta}_k\hat{\beta}_lx_kx_l\right] \\
    &- \sigma^{-2}_y\left[2\sum_{\substack{k=1\\ k \in S}}^q\hat{\beta}_k\hat{\beta}_jx_kx_j + 2\sum_{\substack{k=1\\ k \in \overline{S}}}^q\hat{\beta}_k\hat{\beta}_jx_kx_j + 2 \sum_{\substack{k=1 \\ k \in S}}^q \sum_{\substack{l=1 \\ l \in \overline{S}}}^q\hat{\beta}_k\hat{\beta}_lx_kx_l\right] \\ 
    &- \sigma^{-2}_y\left[- 2\sum_{\substack{k=1\\ k \in S}}^qyx_k\hat{\beta}_k - 2\sum_{\substack{k=1\\ k \in \overline{S}}}^qyx_k\hat{\beta}_k - 2yx_j\hat{\beta}_j \right].
\end{align*}
\end{proof}
\vspace{-1cm}
\begin{lemma} \label{big_summation}
For a given set of features $\{\mathbf{x}\} \setminus
				\left\{x_{j}\right\}$, we have:
\begin{align*}
    2\sum_{S \subseteq \mathcal{P}(\{\mathbf{x}\} \setminus
				\left\{x_{j}\right\})}^{}\omega_{S}\left(\sum_{\substack{k = 1\\ k \in S}}^qh(x_k) + \sum_{\substack{k=1\\ k \in \overline{S}}}^qg(x_k)\right) = \sum_{\substack{k=1\\ k \neq j}}\left(h(x_k) +g(x_k)\right),
\end{align*}
with $h(.) \mbox{ and } g(.)$ some unknown linear or non-linear functions.
\end{lemma}

\begin{proof} Let consider a quantity $A$ defined as follows:
\begin{align} \label{big_sum}
    A = 2\sum_{S \subseteq \mathcal{P}(\{\mathbf{x}\} \setminus
				\left\{x_{j}\right\})}^{}\omega_{S}\left(\sum_{\substack{k = 1\\ k \in S}}^qh(x_k) + \sum_{\substack{k=1\\ k \in \overline{S}}}^qg(x_k)\right), 
\end{align}
with $h(.) \mbox{ and } g(.)$ some unknown functions.
As $\mathcal{P}(\{\mathbf{x}\} \setminus
				\left\{x_{j}\right\}) = \bigcup_{l=0}^{q-1} \mathcal{P}_l(\{\mathbf{x}\} \setminus
				\left\{x_{j}\right\})$, where $\mathcal{P}_l(\{\mathbf{x}\} \setminus
				\left\{x_{j}\right\})$ refers to the collection of all subsets of size $l$ that can be formed from the powerset $\mathcal{P}(\{\mathbf{x}\} \setminus
				\left\{x_{j}\right\})$, we obtain from Equation \eqref{big_sum}:
\begin{align} \label{sum_1}
    A  = 2\sum_{S \subseteq \bigcup_{l=0}^{q-1} \mathcal{P}_l(\{\mathbf{x}\} \setminus
				\left\{x_{j}\right\})}^{}\omega_{S}\left(\sum_{\substack{k = 1\\ k \in S}}^qh(x_k) + \sum_{\substack{k=1\\ k \in \overline{S}}}^qg(x_k)\right).
\end{align}
    Note that for an even number of features $q$, we have:
    \begin{align} \label{even}
        \bigcup_{l=0}^{q-1} \mathcal{P}_l(\{\mathbf{x}\} \setminus
				\left\{x_{j}\right\})=\bigcup_{l=0}^{(q-2)/2}  \mathcal{P}_{q-1-l}(\{\mathbf{x}\} \setminus
				\left\{x_{j}\right\}) \cup \mathcal{P}_{l}(\{\mathbf{x}\} \setminus
				\left\{x_{j}\right\}),
    \end{align}
    whereas for an odd number of features, we end up with:
    \begin{align}
        \bigcup_{l=0}^{q-1} \mathcal{P}_l(\{\mathbf{x}\} \setminus
				\left\{x_{j}\right\}) = \bigcup_{l=0}^{(q-1)/2}  \mathcal{P}_{q-1-l}(\{\mathbf{x}\} \setminus
				\left\{x_{j}\right\}) \cup \mathcal{P}_{l}(\{\mathbf{x}\} \setminus
				\left\{x_{j}\right\}).
    \end{align}
    Therefore, we now distinguish between the two cases to complete the proof.
    \medskip
    
    \noindent When $q$ is even, substituting the expression of $\bigcup_{l=0}^{q-1} \mathcal{P}_l(\{\mathbf{x}\} \setminus
				\left\{x_{j}\right\})$ from Equation \eqref{even} into Equation \eqref{sum_1} renders Equation \eqref{sum_1} equivalent to:
    \begin{align} \label{intermediate_even}
        A = 2\sum_{S \subseteq \bigcup_{l=0}^{(q-2)/2}  \mathcal{P}_{q-1-l}(\{\mathbf{x}\} \setminus
				\left\{x_{j}\right\}) \cup \mathcal{P}_{l}(\{\mathbf{x}\} \setminus
				\left\{x_{j}\right\})}^{}\omega_{S}\left(\sum_{\substack{k = 1\\ k \in S}}^qh(x_k) + \sum_{\substack{k=1\\ k \in \overline{S}}}^qg(x_k)\right).
    \end{align}
    As $\bigcap_{l=0}^{(q-2)/2}  \mathcal{P}_{q-1-l}(\{\mathbf{x}\} \setminus
				\left\{x_{j}\right\}) \cup \mathcal{P}_{l}(\{\mathbf{x}\} \setminus
				\left\{x_{j}\right\}) = \emptyset$, we can rewrite Equation \eqref{intermediate_even} as follows:
    \begin{align} \label{intermediate_even_2}
        A = 2\sum_{l=0}^{(q-2)/2}\sum_{S \subseteq \mathcal{P}_{q-1-l}(\{\mathbf{x}\} \setminus
				\left\{x_{j}\right\}) \cup \mathcal{P}_{l}(\{\mathbf{x}\} \setminus
				\left\{x_{j}\right\})}^{}\omega_{S}\left(\sum_{\substack{k = 1\\ k \in S}}^qh(x_k) + \sum_{\substack{k=1\\ k \in \overline{S}}}^qg(x_k)\right).
    \end{align}
    Note that each coalition $S \subseteq \mathcal{P}_{q-1-l}(\{\mathbf{x}\} \setminus
				\left\{x_{j}\right\}) \cup \mathcal{P}_{l}(\{\mathbf{x}\} \setminus
				\left\{x_{j}\right\})$ is either composed of $|S| =q-1-l$ or $|S| = l$ elements obtained from the set of features $\{\mathbf{x}\} \setminus
				\left\{x_{j}\right\}$ of size $q-1$. Therefore, according to Lemma \ref{weight_lemma}, all of the coalitions $S \subseteq \mathcal{P}_{q-1-l}(\{\mathbf{x}\} \setminus
				\left\{x_{j}\right\}) \cup \mathcal{P}_{l}(\{\mathbf{x}\} \setminus
				\left\{x_{j}\right\})$ have the same weight.  We refer to this weight as $\omega_{l}$. Thus, Equation \eqref{intermediate_even_2} simplifies to:
        \begin{align} 
        A = 2\sum_{l=0}^{(q-2)/2}\omega_{l}\sum_{S \subseteq \mathcal{P}_{q-1-l}(\{\mathbf{x}\} \setminus
				\left\{x_{j}\right\}) \cup \mathcal{P}_{l}(\{\mathbf{x}\} \setminus
				\left\{x_{j}\right\})}^{}\left(\sum_{\substack{k = 1\\ k \in S}}^qh(x_k) + \sum_{\substack{k=1\\ k \in \overline{S}}}^qg(x_k)\right).
    \end{align}
    By construction, each feature $x_k \in \{\mathbf{x}\} \setminus
				\left\{x_{j}\right\}$ is included in half of the coalitions $S \subseteq \mathcal{P}_{q-1-l}(\{\mathbf{x}\} \setminus
				\left\{x_{j}\right\}) \cup \mathcal{P}_{l}(\{\mathbf{x}\} \setminus
				\left\{x_{j}\right\})$. Note that the set $\mathcal{P}_{q-1-l}(\{\mathbf{x}\} \setminus
				\left\{x_{j}\right\}) \cup \mathcal{P}_{l}(\{\mathbf{x}\} \setminus
				\left\{x_{j}\right\})$ is composed of $2\times C^{q-1-l}_{q-1} = 2\times C^{l}_{q-1}$ coalitions. Therefore, each feature $x_k$ is included in $C^{l}_{q-1}$ coalitions. Thus, we can write that: 
    \begin{align}
        \sum_{S \subseteq \mathcal{P}_{q-1-l}(\{\mathbf{x}\} \setminus
				\left\{x_{j}\right\}) \cup \mathcal{P}_{l}(\{\mathbf{x}\} \setminus
				\left\{x_{j}\right\})}^{}\sum_{\substack{k = 1\\ k \in S}}^qh({x}_k) = \sum_{\substack{k = 1\\ k \neq j}}^qC^{l}_{q-1}h({x}_k).
    \end{align}
    As each feature $x_k \in \{\mathbf{x}\} \setminus
				\left\{x_{j}\right\}$ is included in half of the coalitions $S \subseteq \mathcal{P}_{q-1-l}(\{\mathbf{x}\} \setminus
				\left\{x_{j}\right\}) \cup \mathcal{P}_{l}(\{\mathbf{x}\} \setminus
				\left\{x_{j}\right\})$, each of them is excluded from the coalitions $S$ half of the time. Therefore, each feature $x_k$ is included in $C^{l}_{q-1}$ coalitions $\overline{S} \subseteq \mathcal{P}_{q-1-l}(\{\mathbf{x}\} \setminus
				\left\{x_{j}\right\}) \cup \mathcal{P}_{l}(\{\mathbf{x}\} \setminus
				\left\{x_{j}\right\})$, such that: 
    \begin{align}
        \sum_{S \subseteq \mathcal{P}_{q-1-l}(\{\mathbf{x}\} \setminus
				\left\{x_{j}\right\}) \cup \mathcal{P}_{l}(\{\mathbf{x}\} \setminus
				\left\{x_{j}\right\})}^{}\sum_{\substack{k = 1\\ k \in \overline{S}}}^qg({x}_k) = \sum_{\substack{k = 1\\ k \neq j}}^qC^{l}_{q-1}g({x}_k).
    \end{align}
    As a consequence, Equation \eqref{intermediate_even_2} simplifies to:
    \begin{align}
        A = 2\sum_{l=0}^{(q-2)/2}\omega_{l}C^l_{q-1}\left(\sum_{\substack{k=1\\ k \neq j}}\left(h(x_k) +g(x_k)\right)\right).
    \end{align}
    Moreover, according to the definition of $\omega_{l}$ in Equation \eqref{Equation_weight}, we then have:
        \begin{align}
        A = 2\sum_{l=0}^{(q-2)/2}\frac{1}{q}\left(\sum_{\substack{k=1\\ k \neq j}}\left(h(x_k) +g(x_k)\right)\right).
    \end{align}
    Finally, for an even number of features $q$, we obtain:
            \begin{align}
        2\sum_{S \subseteq \mathcal{P}(\{\mathbf{x}\} \setminus
				\left\{x_{j}\right\})}^{}\omega_{S}\left(\sum_{\substack{k = 1\\ k \in S}}^qh(x_k) + \sum_{\substack{k=1\\ k \in \overline{S}}}^qg(x_k)\right) = \sum_{\substack{k=1\\ k \neq j}}\left(h(x_k) +g(x_k)\right).
    \end{align}
    \noindent Similarly, when $q$ is odd, we can write that:
    \begin{align} \label{odd_sum}
            A = 2\sum_{l=0}^{(q-1)/2}\omega_{l}\sum_{S \subseteq \mathcal{P}_{q-1-l}(\{\mathbf{x}\} \setminus
    				\left\{x_{j}\right\}) \cup \mathcal{P}_{l}(\{\mathbf{x}\} \setminus
    				\left\{x_{j}\right\})}^{}\left(\sum_{\substack{k = 1\\ k \in S}}^qh(x_k) + \sum_{\substack{k=1\\ k \in \overline{S}}}^qg(x_k)\right).
        \end{align}
    However, when $l=(q-1)/2$,  
    the set $\mathcal{P}_{q-1-l}(\{\mathbf{x}\} \setminus
				\left\{x_{j}\right\}) \cup \mathcal{P}_{l}(\{\mathbf{x}\} \setminus
				\left\{x_{j}\right\})$ is only composed of $C^{q-1-l}_{q-1} = C^{l}_{q-1}$ coalitions as $\mathcal{P}_{(q-1)/2}(\{\mathbf{x}\} \setminus
				\left\{x_{j}\right\}) \cup \mathcal{P}_{(q-1)/2}(\{\mathbf{x}\} \setminus
				\left\{x_{j}\right\}) = \mathcal{P}_{(q-1)/2}(\{\mathbf{x}\} \setminus
				\left\{x_{j}\right\}) $. Therefore, for $l=(q-1)/2$, each feature $x_k \in S \subseteq \mathcal{P}_{q-1-l}(\{\mathbf{x}\} \setminus
				\left\{x_{j}\right\}) \cup \mathcal{P}_{l}(\{\mathbf{x}\} \setminus
				\left\{x_{j}\right\})$ is included in $C^{l}_{q-1}/2$ coalitions.
    To take into account this specificity, we rewrite Equation \eqref{odd_sum} as follows:
    \begin{align} \label{odd_sum_1}
            A &= 2\sum_{l=0}^{(q-3)/2}\omega_{l}\sum_{S \subseteq \mathcal{P}_{q-1-l}(\{\mathbf{x}\} \setminus
    				\left\{x_{j}\right\}) \cup \mathcal{P}_{l}(\{\mathbf{x}\} \setminus
    				\left\{x_{j}\right\})}^{}\left(\sum_{\substack{k = 1\\ k \in S}}^qh(x_k) + \sum_{\substack{k=1\\ k \in \overline{S}}}^qg(x_k)\right) \\
         &-2\omega_{(q-1)/2}\sum_{S \subseteq \mathcal{P}_{(q-1)/2}(\{\mathbf{x}\} \setminus
    				\left\{x_{j}\right\}) }^{}\left(\sum_{\substack{k = 1\\ k \in S}}^qh(x_k) + \sum_{\substack{k=1\\ k \in \overline{S}}}^qg(x_k)\right),
        \end{align}
        which is equal to:
        \begin{align} 
            A &= 2\sum_{l=0}^{(q-3)/2}\omega_{l}C^l_{q-1}\left(\sum_{\substack{k=1\\ k \neq j}}\left(h(x_k) +g(x_k)\right)\right) -2w_{|(q-1)/2|}\frac{C^{(q-1)/2}_{q-1}}{2}\left(\sum_{\substack{k=1\\ k \neq j}}\left(h(x_k) +g(x_k)\right)\right).
        \end{align}
        According to the definition of $\omega_{l}$ in Equation \eqref{Equation_weight}, we then have:
       \begin{align} 
            A &= 2\sum_{l=0}^{(q-3)/2}\frac{1}{q}\left(\sum_{\substack{k=1\\ k \neq j}}\left(h(x_k) +g(x_k)\right)\right) -\left(\sum_{\substack{k=1\\ k \neq j}}\left(h(x_k) +g(x_k)\right)\right).
        \end{align}
         Finally, for an odd number of features $q$, we obtain:
            \begin{align}
        A = 2\sum_{S \subseteq \mathcal{P}(\{\mathbf{x}\} \setminus
				\left\{x_{j}\right\})}^{}\omega_{S}\left(\sum_{\substack{k = 1\\ k \in S}}^qh(x_k) + \sum_{\substack{k=1\\ k \in \overline{S}}}^qg(x_k)\right) = \sum_{\substack{k=1\\ k \neq j}}\left(h(x_k) +g(x_k)\right).
    \end{align}
    As we obtain the same expression for $A$ with an odd and an even number of features $q$, we conclude that for all $q$: 
    \begin{align}
        2\sum_{S \subseteq \mathcal{P}(\{\mathbf{x}\} \setminus
				\left\{x_{j}\right\})}^{}\omega_{S}\left(\sum_{\substack{k = 1\\ k \in S}}^qh(x_k) + \sum_{\substack{k=1\\ k \in \overline{S}}}^qg(x_k)\right) = \sum_{\substack{k=1\\ k \neq j}}\left(h(x_k) +g(x_k)\right),
    \end{align}    
    with $h(.) \mbox{ and } g(.)$ some unknown functions.
\end{proof}

\begin{proposition} \label{R2_prop}
    Consider a linear regression model $\hat{f}(\mathbf{x})= \sum_{j=1}^{q}\hat{\beta}_jx_{j}$ where we assume that the DGP
	of the test sample $S_n=\{\mathbf{x}_i,y_i,\hat{f}(\mathbf{x}%
	_i)\}^n_{i=1}$ satisfies $	\mathbb{E}\left(\mathbf{x}\right) = \mu_q 
	\mbox{ and } \mathbb{V}(\mathbf{x}) = \Sigma$ a positive semi-definite matrix, with $\Sigma_{k,j} = \sigma_{x_k,x_j}$ the covariance between feature $x_k$ and $x_j$. Then, the XPER contribution $\phi_{j}$ of feature $x_j$ to the $R^2$ is:
\begin{equation} 
    \phi_{j}=\frac{2\hat{\beta}_j\sigma_{y,x_j}}{\sigma^2_y},  \qquad  \quad  \forall j=1,...,q,
\end{equation}
with $\sigma_y^2$ the variance of the target variable and $\sigma_{y,x_j}$ its covariance with feature $x_j$.
\end{proposition}

\begin{proof} Consider a linear regression model $\hat{f}(\mathbf{x})= \sum_{j=1}^{q}\hat{\beta}_jx_{j}$ where we assume that the DGP	of the test sample $S_n=\{\mathbf{x}_i,y_i,\hat{f}(\mathbf{x}%
	_i)\}^n_{i=1}$ satisfies $	\mathbb{E}\left(\mathbf{x}\right) = \mu_q 
	\mbox{ and } \mathbb{V}(\mathbf{x}) = \Sigma$, a positive semi-definite matrix, with $\Sigma_{k,j} = \sigma_{x_k,x_j}$ the covariance between feature $x_k$ and $x_j$.
 
 From Lemma \ref{lemma_R2}, we can derive that, for a coalition $S \subseteq \mathcal{P}(\{\mathbf{x}\} \setminus\left\{x_{j}\right\})$, we have:
\begin{align} \label{equation_R2}
\mathbb{E}_{y,\mathbf{x}^S,x_j}\mathbb{E}_{\mathbf{x}^{\overline{S}}}\left(G(y;\mathbf{x};\delta_0)\right) - \mathbb{E}_{y,\mathbf{x}^S}\mathbb{E}_{\mathbf{x}^{\overline{S}},x_j}\left(G(y;\mathbf{x};\delta_0)\right) &= \sigma^{-2}_y\left[-2\sum_{\substack{k=1\\ k \in S}}^q\hat{\beta}_k\hat{\beta}_j\sigma_{x_k,x_j} + 2\sum_{\substack{k=1\\ k \in \overline{S}}}^q\hat{\beta}_k\hat{\beta}_j\sigma_{x_k,x_j}\right]  \nonumber \\
&+ \sigma^{-2}_y\left[2\hat{\beta}_j\sigma_{y,x_j}\right],
\end{align}
with $\sigma_y^2$ the variance of the target variable and $\sigma_{y,x_j}$ its covariance with feature $x_j$.
Thus, according to Definition \ref{Equation_weight} and Equation \eqref{equation_R2}, the XPER contribution $\phi_{j}$ to the $R^2$ is equal to:
\begin{align}
\phi_{j} &= \sum_{S \subseteq \mathcal{P}(\{\mathbf{x}\} \setminus
				\left\{x_{j}\right\})}^{}\omega_S\left(\mathbb{E}_{y,\mathbf{x}^S,x_j}\mathbb{E}_{\mathbf{x}^{\overline{S}}}\left(G(y;\mathbf{x};\delta_0)\right) - \mathbb{E}_{y,\mathbf{x}^S}\mathbb{E}_{\mathbf{x}^{\overline{S}},x_j}\left(G(y;\mathbf{x};\delta_0)\right)\right) \nonumber \\
    &= \sum_{S \subseteq \mathcal{P}(\{\mathbf{x}\} \setminus
				\left\{x_{j}\right\})}^{}\omega_S\left(\sigma^{-2}_y\left[-2\sum_{\substack{k=1\\ k \in S}}^q\hat{\beta}_k\hat{\beta}_j\sigma_{x_k,x_j} + 2\sum_{\substack{k=1\\ k \in \overline{S}}}^q\hat{\beta}_k\hat{\beta}_j\sigma_{x_k,x_j}\right]\right) \nonumber \\
&+ \sum_{S \subseteq \mathcal{P}(\{\mathbf{x}\} \setminus
				\left\{x_{j}\right\})}^{}\omega_S\sigma^{-2}_y\left[2\hat{\beta}_j\sigma_{y,x_j}\right]. 
\end{align}
As according to Lemma \ref{sum_weight}, we know that $\sum_{S \subseteq \mathcal{P}(\{\mathbf{x}\} \setminus
				\left\{x_{j}\right\})}^{}\omega_S = 1$,  we obtain:
\begin{align}
\phi_{j} &= \sum_{S \subseteq \mathcal{P}(\{\mathbf{x}\} \setminus
				\left\{x_{j}\right\})}^{}\omega_S\left(\sigma^{-2}_y\left[-2\sum_{\substack{k=1\\ k \in S}}^q\hat{\beta}_k\hat{\beta}_j\sigma_{x_k,x_j} + 2\sum_{\substack{k=1\\ k \in \overline{S}}}^q\hat{\beta}_k\hat{\beta}_j\sigma_{x_k,x_j}\right]\right) \nonumber \\
&+ \sigma^{-2}_y\left[2\hat{\beta}_j\sigma_{y,x_j}\right]. \label{phi_j_inter}
\end{align}
\noindent According to Lemma \ref{big_summation}, for $h(x_k)=-\hat{\beta}_k\hat{\beta}_j\sigma_{x_k,x_j}$ and $g(x_k)=\hat{\beta}_k\hat{\beta}_j\sigma_{x_k,x_j}$ we obtain:
\begin{align} \label{lemma_2_final}
    2\sum_{S \subseteq \mathcal{P}(\{\mathbf{x}\} \setminus
				\left\{x_{j}\right\})}^{}\omega_{S}\left(\sum_{\substack{k = 1\\ k \in S}}^q-\hat{\beta}_k\hat{\beta}_j\sigma_{x_k,x_j} + \sum_{\substack{k=1\\ k \in \overline{S}}}^q\hat{\beta}_k\hat{\beta}_j\sigma_{x_k,x_j}\right) = \sum_{\substack{k=1\\ k \neq j}}\left(-\hat{\beta}_k\hat{\beta}_j\sigma_{x_k,x_j} + \hat{\beta}_k\hat{\beta}_j\sigma_{x_k,x_j}\right) = 0
\end{align}
Therefore, from Equations \eqref{phi_j_inter} and \eqref{lemma_2_final}, we deduce that the XPER contribution $\phi_{j}$ of feature $x_j$ to the $R^2$ is:
\begin{align*}
    \phi _{j} = \frac{2\hat{\beta}_j\sigma_{y,x_j}}{\sigma^{2}_y}.
\end{align*}

\end{proof}
    \vspace{-1cm}
    \subsection{Proof of Equation \eqref{R2_ind_theory}}
    \label{proof_R2_ind_theory}
    \begin{lemma} \label{weight_lemma}
    The weight associated with a coalition $S$ built from a set of features of size $q-1$ is equal to the weight of the coalition $\Tilde{S}$, where $|\Tilde{S}| = q-1- |S|$, i.e., $\omega_{S} = \omega_{\Tilde{S}}$.
\end{lemma}

\begin{proof} According to Equation \eqref{Equation_weight}, $\omega_{S}$ is defined as:
\begin{align*}
    \omega_{S} = \frac{1}{q \times C^{|S|}_{q-1}} = \frac{1}{q \times \frac{(q-1)!}{|S|! (q-1-|S|)!}}.
\end{align*}
Similarly, as $|\Tilde{S}| = q - 1 - |S|$, $\omega_{\Tilde{S}}$ is expressed as:
\begin{align*}
    \omega_{\Tilde{S}}  = \frac{1}{q \times C^{|\Tilde{S}|}_{q-1}} =  \frac{1}{q \times C^{q - 1 - |S|}_{q-1}} = \frac{1}{q \times \frac{(q-1)!}{(q - 1 - |S|)! (q-1-(q - 1 - |S|))!}} = \frac{1}{q \times \frac{(q-1)!}{|S|! (q-1-|S|)!}} = \omega_{S}.
\end{align*}
\end{proof}
\vspace{-1cm}
\begin{proposition}
    Consider a linear regression model $\hat{f}(\mathbf{x})= \sum_{j=1}^{q}\hat{\beta}_jx_{j}$ where we assume that the DGP
	of the test sample $S_n=\{\mathbf{x}_i,y_i,\hat{f}(\mathbf{x}%
	_i)\}^n_{i=1}$ satisfies $	\mathbb{E}\left(\mathbf{x}\right) = \mu_q 
	\mbox{ and } \mathbb{V}(\mathbf{x}) = \Sigma$, a positive semi-definite matrix, with $\Sigma_{k,j} = \sigma_{x_k,x_j}$ the covariance between feature $x_k$ and $x_j$. The individual XPER contribution $\phi_{i,j}$ to the $R^2$ is:
\begin{equation*}  
    \phi_{i,j} = \sigma^{-2}_y\left[\hat{\beta}_j(x_{i,j} - \mathbb{E}(x_{j}))A -\hat{\beta}_j^2(x_{i,j}^2 - \mathbb{E}(x_{j}^2))  + \sum_{\substack{k=1\\ k \neq j}}^q\hat{\beta}_k\hat{\beta}_j\sigma_{x_k,x_j}   \right].
\end{equation*}
with $A=\left( 2y_i -\sum_{\substack{k=1\\ k \neq j}}^q\hat{\beta}_k(x_{i,k} +\mathbb{E}(x_{k})) \right)$, $\sigma_y^2$ the variance of the target variable, and $\sigma_{x_k,x_j}$ the covariance between the feature $x_k$ and $x_j$.
\end{proposition}
\begin{proof} Consider a linear regression model $\hat{f}(\mathbf{x})= \sum_{j=1}^{q}\hat{\beta}_jx_{j}$ where we assume that the DGP
	of the test sample $S_n=\{\mathbf{x}_i,y_i,\hat{f}(\mathbf{x}%
	_i)\}^n_{i=1}$ satisfies $	\mathbb{E}\left(\mathbf{x}\right) = \mu_q 
	\mbox{ and } \mathbb{V}(\mathbf{x}) = \Sigma$, a positive semi-definite matrix, with $\Sigma_{k,j} = \sigma_{x_k,x_j}$ the covariance between feature $x_k$ and $x_j$.
 
    From Lemma \ref{lemma_R2}, we can derive that for a coalition $S \subseteq \mathcal{P}(\{\mathbf{x}\} \setminus\left\{x_{j}\right\})$ and for an individual $i$, we have: 
\begin{align}
\mathbf{E}_{\mathbf{x}^{\overline{S}}}\left(G(y_i;\mathbf{x}_i;\delta_0)\right) - \mathbf{E}_{\mathbf{x}^{\overline{S}},x_j}\left(G(y_i;\mathbf{x}_i;\delta_0)\right) &= \sigma^{-2}_y\left[2y_i\hat{\beta}_j(x_{i,j}- \mathbb{E}(x_j)) -\hat{\beta}_j^2(x_{i,j}^2 - \mathbb{E}(x_j^2)) \right]\nonumber \\ 
&+ \sigma^{-2}_y\left[2\sum_{\substack{k=1\\k \in \overline{S}}}\hat{\beta}_k\hat{\beta}_j\sigma_{x_k,x_j} \right] \nonumber \\ 
    &- \sigma^{-2}_y\left[ 2\hat{\beta}_j(x_{i,j}- \mathbb{E}(x_j))\left(\sum_{\substack{k=1\\k \in S}}\hat{\beta}_kx_{i,k} + \sum_{\substack{k=1\\k \in \overline{S}}}\hat{\beta}_k\mathbb{E}(x_k)\right)  \right],
    \label{R2_ind_1}
\end{align}
with $\sigma_y^2$ the variance of the target variable and $\sigma_{x_k,x_j}$ the covariance between the feature $x_k$ and $x_j$.
Thus, according to Definition \ref{def_ind_XPER} and Equation \eqref{R2_ind_1}, the XPER contribution $\phi_{i,j}$ to the $R^2$ is equal to:
\begin{align}
\phi_{i,j} &= \sum_{S \subseteq \mathcal{P}(\{\mathbf{x}\} \setminus
				\left\{x_{j}\right\})}^{}\omega_S\left(\mathbf{E}_{\mathbf{x}^{\overline{S}}}\left(G(y_i;\mathbf{x}_i;\delta_0)\right) - \mathbf{E}_{\mathbf{x}^{\overline{S}},x_j}\left(G(y_i;\mathbf{x}_i;\delta_0)\right)\right) \nonumber \\
    &= \sum_{S \subseteq \mathcal{P}(\{\mathbf{x}\} \setminus
				\left\{x_{j}\right\})}^{}\omega_S\sigma^{-2}_y\left[2y_i\hat{\beta}_j(x_{i,j}- \mathbb{E}(x_j)) -\hat{\beta}_j^2(x_{i,j}^2 - \mathbb{E}(x_{j}^2)) + 2\sum_{\substack{k=1\\k \in \overline{S}}}\hat{\beta}_k\hat{\beta}_j\sigma_{x_k,x_j} \right] \nonumber \\ 
    &- \sum_{S \subseteq \mathcal{P}(\{\mathbf{x}\} \setminus
				\left\{x_{j}\right\})}^{}\omega_S\sigma^{-2}_y\left[2\hat{\beta}_j(x_{i,j}- \mathbb{E}(x_j))\left(\sum_{\substack{k=1\\k \in S}}\hat{\beta}_kx_{i,k} + \sum_{\substack{k=1\\k \in \overline{S}}}\hat{\beta}_k\mathbb{E}(x_k)\right)  \right]. 
\end{align}
As according to Lemma \ref{sum_weight}, we know that $\sum_{S \subseteq \mathcal{P}(\{\mathbf{x}\} \setminus
				\left\{x_{j}\right\})}^{}\omega_S = 1$,  we obtain:
\begin{align}
\phi_{i,j} &=  \sigma^{-2}_y\left[2y_i\hat{\beta}_j(x_{i,j}- \mathbb{E}(x_j)) -\hat{\beta}_j^2(x_{i,j}^2 - \mathbb{E}(x_j^2)) + 2\sum_{S \subseteq \mathcal{P}(\{\mathbf{x}\} \setminus
				\left\{x_{j}\right\})}^{}\omega_S\sum_{\substack{k=1\\k \in \overline{S}}}\hat{\beta}_k\hat{\beta}_j\sigma_{x_k,x_j} \right] \nonumber \\ 
    &- \sigma^{-2}_y\left[\hat{\beta}_j(x_{i,j}- \mathbb{E}(x_j))\times 2\sum_{S \subseteq \mathcal{P}(\{\mathbf{x}\} \setminus
				\left\{x_{j}\right\})}^{}\omega_S\left(\sum_{\substack{k=1\\k \in S}}\hat{\beta}_kx_{i,k} + \sum_{\substack{k=1\\k \in \overline{S}}}\hat{\beta}_k\mathbb{E}(x_k)\right)  \right].  \label{r2_ind_finall}
\end{align} 
\noindent According to Lemma \ref{big_summation}, for $h(x_k)=0$ and $g(x_k)=\hat{\beta}_k\hat{\beta}_j\sigma_{x_k,x_j}$ we find that:
\begin{align}  \label{a}
    2\sum_{S \subseteq \mathcal{P}(\{\mathbf{x}\} \setminus
				\left\{x_{j}\right\})}^{}\omega_S\sum_{\substack{k=1\\k \in \overline{S}}}\hat{\beta}_k\hat{\beta}_j\sigma_{x_k,x_j} = \sum_{\substack{k=1\\ k \neq j}}^q\hat{\beta}_k\hat{\beta}_j\sigma_{x_k,x_j}
\end{align}
Similarly, for $h(x_k)=\hat{\beta}_kx_{i,k}$ and $g(x_k)=\hat{\beta}_k\mathbb{E}(x_k)$:
\begin{align} \label{b}
    2\sum_{S \subseteq \mathcal{P}(\{\mathbf{x}\} \setminus
				\left\{x_{j}\right\})}^{}\omega_S\left(\sum_{\substack{k=1\\k \in S}}\hat{\beta}_kx_{i,k} + \sum_{\substack{k=1\\k \in \overline{S}}}\hat{\beta}_k\mathbb{E}(x_{i,k})\right) = \sum_{\substack{k=1\\ k \neq j}}^q\hat{\beta}_k(x_{i,k} +\mathbb{E}(x_k))  
\end{align}
From Equations \eqref{r2_ind_finall}, \eqref{a}, and \eqref{b}, we obtain:
\begin{align}
\phi_{i,j} &=  \sigma^{-2}_y\left[2y_i\hat{\beta}_j(x_{i,j}- \mathbb{E}(x_j)) -\hat{\beta}_j^2(x_{i,j}^2 - \mathbb{E}(x_j^2)) + \sum_{\substack{k=1\\ k \neq j}}^q\hat{\beta}_k\hat{\beta}_j\sigma_{x_k,x_j} \right] \nonumber \\ 
    &- \sigma^{-2}_y\left[\hat{\beta}_j(x_{i,j}- \mathbb{E}(x_j))\sum_{\substack{k=1\\ k \neq j}}^q\hat{\beta}_k(x_{i,k} +\mathbb{E}(x_k))    \right].  
\end{align} 
Finally, after rearranging the terms, we obtain:
\begin{equation*}  
    \phi_{i,j} = \sigma^{-2}_y\left[\hat{\beta}_j(x_{i,j} - \mathbb{E}(x_{j}))A -\hat{\beta}_j^2(x_{i,j}^2 - \mathbb{E}(x_{j}^2))  + \sum_{\substack{k=1\\ k \neq j}}^q\hat{\beta}_k\hat{\beta}_j\sigma_{x_k,x_j}   \right].
\end{equation*}
with $A=\left( 2y_i -\sum_{\substack{k=1\\ k \neq j}}^q\hat{\beta}_k(x_{i,k} +\mathbb{E}(x_{k})) \right)$, $\sigma_y^2$ the variance of the target variable, and $\sigma_{x_k,x_j}$ the covariance between the feature $x_k$ and $x_j$.
\end{proof}
\vspace{-0.5cm}
\subsection{Proof of the MSE example in Table \ref{all_perfs}}
\label{proof_MSE_example}
\begin{proposition}
    Consider a linear regression model $\hat{f}(\mathbf{x})= \sum_{j=1}^{q}\hat{\beta}_jx_{j}$, where $\mathbb{E}\left(\mathbf{x}\right) = 0_q
\mbox{ and } \mathbb{V}(\mathbf{x}) = diag(\sigma^2_{x_j}), \forall j=1,...,q$, and $\mathbb{E}(y)=0$. The contributions $\phi_{j}$ of features $x_j$ to the (opposite of the) MSE satisfy the efficiency axiom such that:
	\begin{equation}
		\underbrace{2\sum_{j=1}^{q}\hat{\beta}_j\sigma_{y,x_j} - \sum_{j=1}^{q}\hat{\beta}_j^2\sigma^2_{x_j} - \sigma_y^2}_{\mathbb{E}_{y,\mathbf{x}}(G(y;\mathbf{x};\delta_0))} = \underbrace{- \sum_{j=1}^{q}\hat{\beta}_j^2\sigma^2_{x_j} - \sigma_y^2}_{\phi_0} + \sum_{j=1}^{q}\underbrace{2\hat{\beta}_j\sigma_{y,x_j}}_{\phi_j}.
	\end{equation}
\end{proposition}
\begin{proof} Consider a linear regression model $\hat{f}(\mathbf{x})= \sum_{j=1}^{q}\hat{\beta}_jx_{j}$, where $\mathbb{E}\left(\mathbf{x}\right) = 0_q
\mbox{ and } \mathbb{V}(\mathbf{x}) = diag(\sigma^2_{x_j}), \forall j=1,...,q$.
As shown in Table \ref{all_perfs}, the individual contribution to the (opposite) MSE is defined as $G(y;\mathbf{x};\delta_0)) = -(y - \mathbf{x}\hat{\beta})^2$, where $\mathbf{x}\hat{\beta} = \hat{f}(\mathbf{x})$. Similarly, $G(y;\mathbf{x};\delta_0))$ can be expressed as:
\begin{align}
    G(y;\mathbf{x};\delta_0)) = - \left[y^2 + \sum_{j=1}^qx_j^2\hat{\beta}_j^2 + 2\sum_{1 \le j < l \le q}\hat{\beta}_j\hat{\beta}_lx_jx_l - 2\sum_{j=1}^qyx_j\hat{\beta}_j \right]. \label{MSE_proof}
\end{align}

\noindent To complete the proof, we first start by proving that the (opposite) MSE can be written as:
\begin{equation}
    \mathbb{E}_{y,\mathbf{x}}(G(y;\mathbf{x};\delta_0)) = 2\sum_{j=1}^{q}\hat{\beta}_j\sigma_{y,x_j} - \sum_{j=1}^{q}\hat{\beta}_j^2\sigma^2_{x_j} - \sigma_y^2.
\end{equation}

\noindent Indeed, by taking the expected value of Equation \eqref{MSE_proof} with respect to the joint distribution of the target variable and the features, we obtain:
\begin{align}
    \mathbb{E}_{y,\mathbf{x}}(G(y;\mathbf{x};\delta_0)) &= -\mathbb{E}(y^2) - \sum_{j=1}^q\hat{\beta}_j^2\mathbb{E}(x_j^2) - 2\sum_{1 \le j < l \le q}\hat{\beta}_j\hat{\beta}_l\mathbb{E}(x_jx_l) + 2\sum_{j=1}^q\hat{\beta}_j\mathbb{E}(yx_j).
    \end{align}
As $\mathbb{E}\left(\mathbf{x}\right) = 0_q 
\mbox{ and } \mathbb{V}(\mathbf{x}) = diag(\sigma^2_{x_j}), \forall j=1,...,q$ we have:
\begin{align}
    \mathbb{E}_{y,\mathbf{x}}(G(y;\mathbf{x};\delta_0)) &= -\mathbb{E}(y^2) - \sum_{j=1}^q\hat{\beta}_j^2\sigma^2_{x_j} + 2\sum_{j=1}^q\hat{\beta}_j\sigma_{y,x_j}. \nonumber 
\end{align}
As we assume that $\mathbb{E}(y) = 0$, we have:
\begin{align} \label{PM_MSE}
    \mathbb{E}_{y,\mathbf{x}}(G(y;\mathbf{x};\delta_0)) &= -\sigma^2_y- \sum_{j=1}^q\hat{\beta}_j^2\sigma^2_{x_j} + 2\sum_{j=1}^q\hat{\beta}_j\sigma_{y,x_j}.
\end{align}

\noindent Second, we prove that the benchmark value of the (opposite) MSE can be written as:
\begin{equation}
    \phi_0 = \mathbb{E}_{y}\mathbb{E}_{\mathbf{x}}(G(y;\mathbf{x};\delta_0)) = -\sigma^2_y - \sum_{j=1}^q\hat{\beta}_j^2\sigma^2_{x_j}.
\end{equation}

\noindent From Equation \eqref{MSE_proof}, by taking the expected value with respect to the target variable and the expected value with respect to the joint distribution of the features, we obtain:
\begin{align}
    \phi_0 = \mathbb{E}_{y}\mathbb{E}_{\mathbf{x}}(G(y;\mathbf{x};\delta_0)) &= -\mathbb{E}(y^2) - \sum_{j=1}^q\hat{\beta}_j^2\mathbb{E}(x_j^2) - 2\sum_{1 \le j < l \le q}\hat{\beta}_j\hat{\beta}_l\mathbb{E}(x_jx_l) + 2\sum_{j=1}^q\hat{\beta}_k\mathbb{E}(y)\mathbb{E}(x_j). 
\end{align}
As $\mathbb{E}\left(\mathbf{x}\right) = 0_q 
\mbox{ and } \mathbb{V}(\mathbf{x}) = diag(\sigma^2_{x_j}), \forall j=1,...,q,$ we have:
\begin{align} \label{phi_0_MSE}
    \phi_0 = \mathbb{E}_{y}\mathbb{E}_{\mathbf{x}}(G(y;\mathbf{x};\delta_0)) &= -\sigma^2_y - \sum_{j=1}^q\hat{\beta}_j^2\sigma^2_{x_j}.
\end{align}

\medskip
\noindent Third, we show that the XPER value associated with the feature $x_j$ for the (opposite) MSE can be expressed as:
\begin{equation}
    \phi_j = 2\hat{\beta}_j\sigma_{y,x_j}.
\end{equation}

\noindent Note that the individual contribution to the $R^2$, defined as $G^{R^2}\left( y;\mathbf{x};\delta _{0}\right) = 1 - \sigma^{-2}_y(y - \mathbf{x}\hat{\beta})$ according to Equation \eqref{R2_def}, can be expressed as $G^{R^2}\left( y;\mathbf{x};\delta _{0}\right) = 1 + \sigma^{-2}_yG\left( y;\mathbf{x};\delta _{0}\right)$, with $G\left( y;\mathbf{x};\delta _{0}\right)$ the individual contribution to the MSE. According to Definition \ref{XPER_definition}, the XPER value associated with the feature $x_j$ for the $R^2$:
\begin{align}
    \phi_j^{R^2} &= \sum_{S\subseteq  \mathcal{P}(\{\mathbf{x}\}\setminus \left\{
      x_{j}\right\} )}\omega_S\left[\mathbb{E}_{y,x_{j},\mathbf{x}^{S}}\mathbb{E}_{\mathbf{x}^{\overline{S}}}\left(1 + \sigma^{-2}_yG\left( y;\mathbf{x};\delta _{0}\right) \right) - \mathbb{E}_{y,\mathbf{x}^{S}}\mathbb{E}_{\mathbf{x}^{x_{j},\overline{S}}}\left(1 + \sigma^{-2}_y G\left( y;\mathbf{x}
       ;\delta _{0}\right) \right)\right] \nonumber \\ 
     \phi_j^{R^2}  &= \sigma^{-2}_y\left[\sum_{S\subseteq  \mathcal{P}(\{\mathbf{x}\}\setminus \left\{
      x_{j}\right\} )}\omega_S\left[\mathbb{E}_{y,x_{j},\mathbf{x}^{S}}\mathbb{E}_{\mathbf{x}^{\overline{S}}}\left(G\left( y;\mathbf{x};\delta _{0}\right) \right) - \mathbb{E}_{y,\mathbf{x}^{S}}\mathbb{E}_{\mathbf{x}^{x_{j},\overline{S}}}\left(G\left( y;\mathbf{x}
       ;\delta _{0}\right) \right)\right]\right] \nonumber \\
    \phi_j^{R^2} &= \frac{\phi_j}{\sigma^{2}_y}. \label{R2_MSE}
\end{align}
Therefore, from Equations \eqref{R2_theory} and \eqref{R2_MSE}, we obtain:
\begin{align} \label{phi_MSE}
    \phi_j = 2\hat{\beta}_j\sigma_{y,x_j}.
\end{align}
Finally, from Equations \eqref{PM_MSE}, \eqref{phi_0_MSE}, and \eqref{phi_MSE}, we conclude that:
\begin{equation}		\underbrace{2\sum_{j=1}^{q}\hat{\beta}_j\sigma_{y,x_j} - \sum_{j=1}^{q}\hat{\beta}_j^2\sigma^2_{x_j} - \sigma_y^2}_{\mathbb{E}_{y,\mathbf{x}}(G(y;\mathbf{x};\delta_0))} = \underbrace{- \sum_{j=1}^{q}\hat{\beta}_j^2\sigma^2_{x_j} - \sigma_y^2}_{\phi_0} + \sum_{j=1}^{q}\underbrace{2\hat{\beta}_j\sigma_{y,x_j}}_{\phi_j}.
\end{equation}
    
\end{proof}

    \subsection{Proof of the $R^2$ example in Table \ref{all_perfs}}
    \label{proof_R2_example}
    \begin{proposition}
    Consider a linear regression model $\hat{y} = \hat{f}(\mathbf{x})= \sum_{j=1}^{q}\hat{\beta}_jx_{j}$, where $\mathbb{E}\left(\mathbf{x}\right) = 0_q
\mbox{ and } \mathbb{V}(\mathbf{x}) = diag(\sigma^2_{x_j}), \forall j=1,...,q$, and the features are uncorrelated to the residuals $\hat{\varepsilon} = y - \hat{y}$. The contributions $\phi_{j}$ of features $x_j$ to the $R^2$ satisfy the efficiency axiom such that:
	\begin{align} 		\underbrace{\frac{\sigma_{y,\hat{y}}}{\sigma^2_y}}_{\mathbb{E}_{y,\mathbf{x}}(G(y;\mathbf{x};\delta_0))} = \underbrace{-\frac{\sigma_{y,\hat{y}}}{\sigma^2_y}}_{\phi_0} + \sum_{j=1}^{q}\underbrace{\frac{2\hat{\beta}_j\sigma_{y,x_j}}{\sigma^2_y}}_{\phi_j}.
	\end{align}
\end{proposition}

\begin{proof}
    Consider a linear regression model $\hat{f}(\mathbf{x})= \sum_{j=1}^{q}\hat{\beta}_jx_{j}$, where $\mathbb{E}\left(\mathbf{x}\right) = 0_q
\mbox{ and } \mathbb{V}(\mathbf{x}) = diag(\sigma^2_{x_j}), \forall j=1,...,q$.
As shown in Equation \eqref{R2_def}, the individual contribution to the $R^2$ is defined as $G(y;\mathbf{x};\delta_0)) = 1 -\sigma^2_y(y - \mathbf{x}\hat{\beta})^2$, where $\mathbf{x}\hat{\beta} = \hat{f}(\mathbf{x})$. Similarly, $G(y;\mathbf{x};\delta_0))$ can be expressed as:
\begin{align}
    G(y;\mathbf{x};\delta_0)) &= 1 - \sigma^{-2}_y\left[y^2 + \sum_{j=1}^qx_j^2\hat{\beta}_j^2 + 2\sum_{1 \le j < l \le q}\hat{\beta}_j\hat{\beta}_lx_jx_l - 2\sum_{j=1}^qyx_j\hat{\beta}_j \right]. \label{R2_proof}
\end{align}

\noindent To complete the proof, we first start by proving that the (opposite) MSE can be written as:
\begin{equation}
    \mathbb{E}_{y,\mathbf{x}}(G(y;\mathbf{x};\delta_0)) = \frac{\sigma_{y,\hat{y}}}{\sigma^2_y}.
\end{equation}
\noindent Indeed, by taking the expected value of Equation \eqref{R2_proof} with respect to the joint distribution of the target variable and the features, we obtain:
\begin{align}
    \mathbb{E}_{y,\mathbf{x}}(G(y;\mathbf{x};\delta_0)) &= 1 - \sigma^{-2}_y\left[\mathbb{E}(y^2) + \sum_{j=1}^q\hat{\beta}_j^2\mathbb{E}(x_j^2)  + 2\sum_{1 \le j < l \le q}\hat{\beta}_j\hat{\beta}_l\mathbb{E}(x_jx_l) - 2\sum_{j=1}^q\hat{\beta}_j\mathbb{E}(yx_j) \right].
\end{align}
We assume that $\mathbb{E}\left(\mathbf{x}\right) = 0_q 
\mbox{ and } \mathbb{V}(\mathbf{x}) = diag(\sigma^2_{x_j}), \forall j=1,...,q$. Therefore, we have:
\begin{align}
    \mathbb{E}_{y,\mathbf{x}}(G(y;\mathbf{x};\delta_0)) &= 1 - \sigma^{-2}_y\left[\mathbb{E}(y^2) + \sum_{j=1}^q\hat{\beta}_j^2\sigma^2_{x_j} - 2\sum_{j=1}^q\hat{\beta}_j\sigma_{y,x_j} \right]. \nonumber  
\end{align}
In a linear model without intercept, we know that $\mathbb{E}(y) = 0$, therefore:
\begin{align} 
    \mathbb{E}_{y,\mathbf{x}}(G(y;\mathbf{x};\delta_0)) &= - \sum_{j=1}^q\frac{\hat{\beta}_j^2\sigma^2_{x_j}}{\sigma^{2}_y} + \sum_{j=1}^q\frac{2\hat{\beta}_j\sigma_{y,x_j} }{\sigma^{2}_y}. \label{R2_fin}
\end{align}
\noindent Moreover, in a linear regression model, the target variable can be expressed as $y = \hat{y} + \hat{\varepsilon}$, with $\hat{y}$ the estimated model and $\hat{\varepsilon}$ the residuals. As the features are uncorrelated from each other, if we also assume that they are uncorrelated to the residuals we can show that:
\begin{equation} \label{corr_y_pred}
    \sigma_{y,\hat{y}} = \sum_{j=1}^q\hat{\beta}_j^2\sigma^2_{x_j} = \sum_{j=1}^q\hat{\beta}_j\sigma_{y,x_j},
\end{equation}
with $\sigma_{y,\hat{y}}$ the covariance between the target variable and its prediction. Therefore, the $R^2$ as expressed in Equation \eqref{R2_fin} is also written as:
\begin{align} \label{R2_theory_1}
    \mathbb{E}_{y,\mathbf{x}}(G(y;\mathbf{x};\delta_0)) &= \frac{\sigma_{y,\hat{y}}}{\sigma^2_y}. 
\end{align}
\noindent Second, we prove that the benchmark value of $R^2$ can be written as:
\begin{equation} 
    \phi_0 = \mathbb{E}_{y}\mathbb{E}_{\mathbf{x}}(G(y;\mathbf{x};\delta_0)) = \frac{\sigma_{y,\hat{y}}}{\sigma^2_y}. 
\end{equation}
\noindent From Equation \eqref{MSE_proof}, by taking the expected value with respect to the target variable and the expected value with respect to the joint distribution of the features, we obtain:
\begin{align}
    \phi_0 = \mathbb{E}_{y}\mathbb{E}_{\mathbf{x}}(G(y;\mathbf{x};\delta_0)) &= 1 - \sigma^{-2}_y\left[\mathbb{E}(y^2) + \sum_{j=1}^q\hat{\beta}_j^2\mathbb{E}(x_j^2) + 2\sum_{1 \le j < l \le q}\hat{\beta}_j\hat{\beta}_l\mathbb{E}(x_jx_l) \right] \nonumber \\
    &- \sigma^{-2}_y\left[- 2\sum_{j=1}^q\hat{\beta}_j\mathbb{E}(y)\mathbb{E}(x_j) \right].
\end{align}
We assume that $\mathbb{E}\left(\mathbf{x}\right) = 0_q 
\mbox{ and } \mathbb{V}(\mathbf{x}) = diag(\sigma^2_{x_j}), \forall j=1,...,q$. Therefore, we have:
\begin{align} \label{last_R2}
     \phi_0 = \mathbb{E}_{y}\mathbb{E}_{\mathbf{x}}(G(y;\mathbf{x};\delta_0))  &= -\frac{\sum_{j=1}^q\hat{\beta}_j^2\sigma^2_{x_j}}{\sigma^{2}_y}.
\end{align}
From Equation \eqref{corr_y_pred}, we can rewrite Equation \eqref{last_R2} as:
\begin{align} \label{R2_theory_2}
     \phi_0 = \mathbb{E}_{y}\mathbb{E}_{\mathbf{x}}(G(y;\mathbf{x};\delta_0))  &= - \frac{\sigma_{y,\hat{y}}}{\sigma^{2}_y}.
\end{align}
\noindent Third, as stated in Proposition \ref{R2_prop}, the XPER value associated with the feature $x_j$ for the $R^2$ can be expressed as:
\begin{equation} \label{R2_theory_3}
    \phi_j = \frac{2\hat{\beta}_j\sigma_{y,x_j}}{\sigma^2_y}.
\end{equation}
Finally, from Equations \eqref{R2_theory_1}, \eqref{R2_theory_2}, and \eqref{R2_theory_3}, we conclude that:
\begin{align} 		\underbrace{\frac{\sigma_{y,\hat{y}}}{\sigma^2_y}}_{\mathbb{E}_{y,\mathbf{x}}(G(y;\mathbf{x};\delta_0))} = \underbrace{-\frac{\sigma_{y,\hat{y}}}{\sigma^2_y}}_{\phi_0} + \sum_{j=1}^{q}\underbrace{\frac{2\hat{\beta}_j\sigma_{y,x_j}}{\sigma^2_y}}_{\phi_j}.
\end{align}

\end{proof}

\subsection{Proof of the accuracy example in Table \ref{all_perfs}}
\label{proof_accuracy_example}

\begin{proposition}
    Consider any binary classification model $\hat{f}(\mathbf{x})$, with $\hat{P}(\mathbf{x}) = \hat{\mathbb{P}}(y=1|\mathbf{x})$ the estimated probability of belonging to class 1 (y=1). The contributions $\phi_{j}$ of features $x_j$ to the $accuracy$ satisfy the efficiency axiom such that:
\begin{align}
		\underbrace{2 \sigma_{y,\hat{f}(\mathbf{x})} + 2\mathbb{P}(y=1)\hat{P}(\mathbf{x}) + 1 - \mathbb{P}(y=1) - \hat{P}(\mathbf{x})}_{\mathbb{E}_{y,\mathbf{x}}(G(y;\mathbf{x};\delta_0))} &= \underbrace{2\mathbb{P}(y=1)\hat{P}(\mathbf{x}) + 1 - \mathbb{P}(y=1) - \hat{P}(\mathbf{x})}_{\phi_0} \nonumber \\ 
  &+ \underbrace{2\sigma_{y,\hat{f}(\mathbf{x})}}_{\sum_{j=1}^{q}\phi_j}, 
	\end{align}
 with $\hat{P}(\mathbf{x}) = \hat{\mathbb{P}}(y=1|\mathbf{x})$ and $\sigma_{y,\hat{f}(\mathbf{x})}$ the covariance between the target variable and the classification output.
\end{proposition}

\begin{proof}
    Consider any binary classification model $\hat{f}(\mathbf{x})$, with $\hat{P}(\mathbf{x}) =\hat{\mathbb{P}}(y=1|\mathbf{x})$ the estimated probability of belonging to class 1 (y=1). As shown in Table \ref{all_perfs}, the individual contribution to the accuracy is defined as:
    \begin{equation} \label{accuracy_def}
        G(y;\mathbf{x};\delta_0)) = y\hat{f}(\mathbf{x}) + (1-y)(1-\hat{f}(\mathbf{x})).
    \end{equation}

    \noindent To complete the proof, we first start by showing that the accuracy can be written as:
\begin{equation}
    \mathbb{E}_{y,\mathbf{x}}(G(y;\mathbf{x};\delta_0)) = 2 \sigma_{y,\hat{f}(\mathbf{x})} + 2\mathbb{P}(y=1)\hat{P}(\mathbf{x}) + 1 - \mathbb{P}(y=1) - \hat{P}(\mathbf{x}).
\end{equation}
\noindent Indeed, by taking the expected value of Equation \eqref{accuracy_def} with respect to the joint distribution of the target variable and the features, we obtain:
\begin{align}
    \mathbb{E}_{y,\mathbf{x}}(G(y;\mathbf{x};\delta_0)) &= 2\mathbb{E}\left(y\hat{f}(\mathbf{x})\right) + 1 - \mathbb{E}(y) - \mathbb{E}\left(\hat{f}(\mathbf{x})\right) \nonumber \\ &= 2 \sigma_{y,\hat{f}(\mathbf{x})} + 2\mathbb{E}(y)\mathbb{E}\left(\hat{f}(\mathbf{x})\right) + 1 - \mathbb{E}(y) - \mathbb{E}\left(\hat{f}(\mathbf{x})\right),
\end{align}
with $\sigma_{y,\hat{f}(\mathbf{x})}$ the covariance between the target variable and the classification output. As the target variable is binary, we know that $\mathbb{E}(y)=\mathbb{P}(y=1)$ and $\mathbb{E}\left(\hat{f}(\mathbf{x})\right) = \hat{\mathbb{P}}(y=1|\mathbf{x}) = \hat{P}(\mathbf{x})$. Therefore, we obtain:
\begin{align} \label{accuracy_PM}
    \mathbb{E}_{y,\mathbf{x}}(G(y;\mathbf{x};\delta_0)) &= 2 \sigma_{y,\hat{f}(\mathbf{x})} + 2\mathbb{P}(y=1)\hat{P}(\mathbf{x}) + 1 - \mathbb{P}(y=1) - \hat{P}(\mathbf{x}).
\end{align}

\noindent Second, from Equation \eqref{accuracy_def}, we can see that by taking the expected value with respect to the target variable and the expected value with respect to the joint distribution of the features, the benchmark value of accuracy can be written as:
\begin{align}
    \phi_0 = \mathbb{E}_{y}\mathbb{E}_{\mathbf{x}}(G(y;\mathbf{x};\delta_0)) &= 2\mathbb{E}\left(y\right)\mathbb{E}\left(\hat{f}(\mathbf{x})\right) + 1 - \mathbb{E}(y) - \mathbb{E}\left(\hat{f}(\mathbf{x})\right) \nonumber \\
    &= 2\mathbb{P}(y=1)\hat{P}(\mathbf{x}) + 1 - \mathbb{P}(y=1) - \hat{P}(\mathbf{x}). \label{phi_0_accuracy}
\end{align}
Third, according to the axiom \ref{efficiency}, we deduce that the sum of the XPER values associated with the features $x_j$ for the accuracy is equal to:
\begin{align} \label{sum_accuracy}
    \sum_{j=1}^{q}\phi_j =  2 \sigma_{y,\hat{f}(\mathbf{x})}.
\end{align}
Finally, from Equations \eqref{accuracy_PM}, \eqref{phi_0_accuracy}, and \eqref{sum_accuracy}, we conclude that:
\begin{align}
		\underbrace{2 \sigma_{y,\hat{f}(\mathbf{x})} + 2\mathbb{P}(y=1)\hat{P}(\mathbf{x}) + 1 - \mathbb{P}(y=1) - \hat{P}(\mathbf{x})}_{\mathbb{E}_{y,\mathbf{x}}(G(y;\mathbf{x};\delta_0))} &= \underbrace{2\mathbb{P}(y=1)\hat{P}(\mathbf{x}) + 1 - \mathbb{P}(y=1) - \hat{P}(\mathbf{x})}_{\phi_0} \nonumber \\ 
  &+ \underbrace{2\sigma_{y,\hat{f}(\mathbf{x})}}_{\sum_{j=1}^{q}\phi_j}. 
	\end{align}
\end{proof}

    \subsection{Proof of Equation \eqref{XPER_SHAP_example}} \label{proof_XPER_SHAP_example}
    \begin{proposition}
        Consider a regression model $\hat{f}(\mathbf{x})$ including only one feature $x_1$ such as $\mathbf{x} = x_1$. We can show that for the performance metric $G(y_i;\hat{f}(\mathbf{x}_i);\delta_0) = -(y_i - \hat{f}(\mathbf{x}_i))^2$, if we assume that $\mathbb{E}(\hat{f}(x_{i,1})) = 0$, then the corresponding individual XPER value $\phi_{i,1}$ is equal to:
        \begin{equation}
            \phi_{i,1} = 2\hat{\varepsilon}_i\phi_{i,1}^{SHAP} + \left(\phi_{i,1}^{SHAP}\right)^2 + \mathbb{V}\left(\hat{f}(x_{i,1})\right),
        \end{equation}
        where $\hat{\varepsilon}_i = y_i - \hat{f}(x_{i,1})$ is the prediction error of the model for individual $i$, and $\phi_{i,1}^{SHAP}$ refers to the SHAP value of the feature $x_1$ for this individual.
    \end{proposition} \vspace{-0.5cm}
    \begin{proof}
        Consider a regression model $\hat{f}(\mathbf{x})$ including only one feature $x_1$ such as $\mathbf{x} = x_1$, and the performance metric $G(y_i;\hat{f}(\mathbf{x}_i);\delta_0) = -(y_i - \hat{f}(\mathbf{x}_i))^2$.

        \noindent According to Equation \eqref{ind_decompos}, the XPER value $\phi_{i,1}$ decomposes $G(y_i;\hat{f}(\mathbf{x}_i);\delta_0)$ such as:
        \begin{align} 
            G(y_i;\hat{f}(\mathbf{x}_i);\delta_0) = \phi_{i,0} + \phi_{i,1},
        \end{align}
        with $\phi_{i,0}$ the benchmark value of the performance metric. Therefore, if we replace $G(y_i;\hat{f}(\mathbf{x}_i);\delta_0)$ by its expression in the previous equation, we can see that:
        \begin{equation} \label{intermediaire}
            \phi_{i,1} = - y_i^2 - \hat{f}(x_{i,1})^2 + 2y_i\hat{f}(x_{i,1}) - \phi_{i,0}.
        \end{equation} \vspace{-0.5cm}
        Moreover, the benchmark value $\phi_{i,0}$ is equal to:
        \begin{equation}
        \phi_{i,0} = \mathbb{E}_{x_{1}}\left(G(y_i;\hat{f}(\mathbf{x}_i);\delta_0)\right) = -y_i^2 - \mathbb{E}\left(\hat{f}(x_{i,1})^2\right) + 2y_i\mathbb{E}\left(\hat{f}(x_{i,1})\right).
        \end{equation}
        As we assume that $\mathbb{E}\left(\hat{f}(x_{i,1})\right) = 0$, we obtain:
        \begin{equation} \label{bench_example_shap}
        \phi_{i,0}= -y_i^2 - \mathbb{V}\left(\hat{f}(x_{i,1})\right).
        \end{equation}
        Replacing the expression of $\phi_{i,0}$ in Equation \eqref{intermediaire}, we obtain:
        \begin{equation} 
            \phi_{i,1} = 2y_i\hat{f}(x_{i,1}) + \mathbb{V}\left(\hat{f}(x_{i,1})\right) - \hat{f}(x_{i,1})^2.
        \end{equation}
        As the prediction error $\varepsilon_i$ can be expressed as the difference between the target variable $y_i$ and the prediction $\hat{f}(x_{i,1})$, i.e.,  $\hat{\varepsilon}_i = y_i - \hat{f}(x_{i,1})$, $\phi_{i,1}$ is then equal to:
        \begin{equation}  \label{last}
            \phi_{i,1} = 2\hat{\varepsilon}_i\hat{f}(x_{i,1}) + \hat{f}(x_{i,1})^2 + \mathbb{V}\left(\hat{f}(x_{i,1})\right). 
        \end{equation}
        Now, according to Equations \eqref{pred_SHAP} and \eqref{SHAP_bench}, as $\mathbb{E}\left(\hat{f}(x_{i,1})\right) = 0$, we can see that:
        \begin{equation} \label{f_pred}
            \hat{f}(x_{i,1}) = \phi_{i,1}^{SHAP}.
        \end{equation}
        with $\phi_{i,1}^{SHAP}$ the SHAP value associated with feature $x_1$ for individual $i$.

        \noindent Finally, from Equations \eqref{last} and \eqref{f_pred}, we obtain:
        \begin{equation}
            \phi_{i,1} = 2\hat{\varepsilon}_i\phi_{i,1}^{SHAP} + \left(\phi_{i,1}^{SHAP}\right)^2 + \mathbb{V}\left(\hat{f}(x_{i,1})\right). 
        \end{equation}
    \end{proof}

    \subsection{Proof of Proposition \ref{prop:SHAP_particular_XPER}} \label{proof_SHAP_particular_XPER}
   
    \XPERSHAP*

    \begin{proof}
    According to Definition \ref{def_ind_XPER}, the individual XPER value $\phi_j$ associated with the performance metric $G\left(y_i;\mathbf{x}_i;\delta_0\right) = \hat{f}(\mathbf{x}_i)$ is equal to:
        \begin{equation}  \label{XPER_proof}
			\phi _{i,j} = \sum_{S \subseteq \mathcal{P}(\{\mathbf{x}\} \setminus
				\left\{x_{j}\right\})}^{} w_S \left[\mathbb{E}_{\mathbf{x}^{\overline{S}}}\left(\hat{f}\left(x_{i,j},\mathbf{x}_i^S,\mathbf{x}^{\overline{S}}\right)\right) - \mathbb{E}_{x_j,\mathbf{x}^{\overline{S}}}\left(\hat{f}\left(x_{j},\mathbf{x}_i^S,\mathbf{x}^{\overline{S}}\right)\right)\right].
		\end{equation}
    Moreover, according Equation \eqref{DefSHAP}, the SHAP value $\phi_j^{SHAP}$ associated with $\hat{f}(\mathbf{x}_i)$ is equal to:
    \begin{align} \label{SHAP_proof}
			\phi _{i,j}^{SHAP} = \sum_{S \subseteq \mathcal{P}(\{\mathbf{x}\} \setminus
				\left\{x_{j}\right\})}^{} w_S \left[\mathbb{E}_{\mathbf{x}^{\overline{S}}}\left(
			\hat{f}(x_{i,j},\mathbf{x}_i^S,\mathbf{x}^{\overline{S}})\right) - \mathbb{E}_{x_j,\mathbf{x}^{\overline{S}}}\left(
			\hat{f}(x_{j},\mathbf{x}_i^S,\mathbf{x}^{\overline{S}})\right)\right].
    \end{align}
    Therefore, in the particular case where $G\left(y_i;\mathbf{x}_i;\delta_0\right) = \hat{f}(\mathbf{x}_i)$, according to Equations \eqref{XPER_proof} and \eqref{SHAP_proof}, the individual XPER value $\phi_j$ is equal to the SHAP value $\phi_j^{SHAP}$.
    \end{proof}

    \subsection{Proof of Equation \eqref{Eq : PI and XPER}}

    \label{proof_PI}
    \begin{proposition}
        The XPER value $\phi_j$ (see Definition \ref{Shapley_1953_def}) can be expressed as:
            \begin{equation}
       \phi _{j}=\sum_{S\subseteq  \mathcal{P}(\{\mathbf{x}_{-j}\}) \setminus \{\mathbf{x}_{-j}\}} \omega_{S}\left[ \mathbb{E}_{y,x_{j},\mathbf{x}^{S}} \mathbb{E}_{\mathbf{x}^{\overline{S}}}
       \left( G\left( y;\mathbf{x} ;\delta _{0}\right) \right) -\mathbb{E}_{y,\mathbf{x}^{S}}
      \mathbb{E}_{x_{j},\mathbf{x}^{\overline{S}}} \left( G\left( y;\mathbf{x}
      ;\delta _{0}\right) \right) \right] + \frac{PI_j}{q}, 
    \end{equation}
    where $PI_j$ corresponds to the Permutation Importance (PI) value of feature $x_j$, as defined in Equation \eqref{PI_def}.
    \end{proposition}

    \begin{proof}
        To establish the link between $PI_j$ (as defined in Equation \eqref{PI_def}) and $\phi_j$ (see Definition \ref{Shapley_1953_def}), let us introduce some additional notations. Let denote $\{\mathbf{x}_{-j}\}= \{\mathbf{x}\}\setminus \left\{
          x_{j}\right\}$ as the set of explanatory variables excluding the feature $x_j$. The powerset of the set $\{\mathbf{x}_{-j}\}$ is defined as $\mathcal{P}(\{\mathbf{x}_{-j}\})$. According to Equation \eqref{Equation_weight} in the paper, for $S = \{\mathbf{x}_{-j}\}$ then $w_S=1/q$, where $q$ denotes the number of model features. Note that for $S = \{\mathbf{x}_{-j}\}$ we have $\mathbf{x}^{\overline{S}}=\{\varnothing\}$. Reminds that the XPER value associated with feature $x_j$ is defined as follows:
         \begin{equation}
            \phi _{j}=\sum_{S\subseteq  \mathcal{P}(\{\mathbf{x}_{-j}\})} \omega_{S}\left[ \mathbb{E}_{y,x_{j},\mathbf{x}^{S}}
          \mathbb{E}_{\mathbf{x}^{\overline{S}}}
           \left( G\left( y;\mathbf{x}
           ;\delta _{0}\right) \right) -\mathbb{E}_{y,\mathbf{x}^{S}}
          \mathbb{E}_{x_{j},\mathbf{x}^{\overline{S}}}
          \left( G\left( y;\mathbf{x}
          ;\delta _{0}\right) \right) \right].
        \end{equation}
        \noindent Taking into account that $\mathcal{P}(\{\mathbf{x}_{-j}\})=(\mathcal{P}(\{\mathbf{x}_{-j}\}) \setminus \{\mathbf{x}_{-1}\}) \cup \{\mathbf{x}_{-1}\}$, we can rearrange the terms to obtain:
        \begin{align} 
            \begin{split}
            \phi _{j}=\sum_{S\subseteq  \mathcal{P}(\{\mathbf{x}_{-j}\}) \setminus \{\mathbf{x}_{-j}\}} &\omega_{S}\left[ \mathbb{E}_{y,x_{j},\mathbf{x}^{S}} 
          \mathbb{E}_{\mathbf{x}^{\overline{S}}}
           \left( G\left( y;\mathbf{x}
           ;\delta _{0}\right) \right) -\mathbb{E}_{y,\mathbf{x}^{S}}
          \mathbb{E}_{x_{j},\mathbf{x}^{\overline{S}}}
          \left( G\left( y;\mathbf{x}
          ;\delta _{0}\right) \right) \right] \\
          &+ \frac{1}{q}\left[\mathbb{E}_{y,\mathbf{x}}
           \left( G\left( y;\mathbf{x}
           ;\delta _{0}\right) \right) - \mathbb{E}_{y,\mathbf{x}_{-j}}\mathbb{E}_{x_j}\left(G(y,\mathbf{x};\delta_0)\right)\right]. 
           \end{split}
        \end{align}
        \noindent Therefore, we can see that the XPER value $\phi_j$ can be expressed as the sum of two terms, one of which is the (normalized) $PI_j$:
        \begin{equation}
           \phi _{j}=\sum_{S\subseteq  \mathcal{P}(\{\mathbf{x}_{-j}\}) \setminus \{\mathbf{x}_{-j}\}} \omega_{S}\left[ \mathbb{E}_{y,x_{j},\mathbf{x}^{S}}
          \mathbb{E}_{\mathbf{x}^{\overline{S}}}
           \left( G\left( y;\mathbf{x}
           ;\delta _{0}\right) \right) -\mathbb{E}_{y,\mathbf{x}^{S}}
          \mathbb{E}_{x_{j},\mathbf{x}^{\overline{S}}}
          \left( G\left( y;\mathbf{x}
          ;\delta _{0}\right) \right) \right] 
          + \frac{1}{q}\left[PI_j\right].
        \end{equation}
    \end{proof}
    
        \noindent Thus, the XPER value $\phi_j$ can be expressed as the permutation importance (PI) of feature $x_j$ (normalized by the number of model features) plus an additional term. This additional term vanishes, making $\phi_j$ exactly equal to the PI, \textit{if and only if} the marginal impact of $x_j$ on performance is identical across all coalitions, i.e., $\mathbb{E}_{y,x_{j},\mathbf{x}^{S}}$ must be consistent across all subsets of features. This condition holds for linear regression models. However, for non-linear models, this condition is generally not valid, leading to different PI and XPER values, as illustrated in Figure \ref{fig:Permutation_importance} in the empirical application.
        
    \clearpage

    \subsection{Proof of equivalence between between sensitivity and specificity decomposition}

    \label{sensitivity_specificity}
    \begin{proposition}
        When decomposing sensitivity and specificity with XPER: (1) the nominal XPER values $\phi_j$ are identical,  up to a normalization factor, and  (2) the XPER values in percentage, i.e., $\phi_j/(\mathbb{E}_{y,\mathbf{x}}(G(y;\mathbf{x};\delta_0))-\phi_0)$, are identical.
    \end{proposition}
    \begin{proof}
    First, by definition, the population sensitivity and specificity are expressed as follows:
            \begin{align*}
            \begin{cases}
                \mathbb{E}_{y,\mathbf{x}}\left(G_{sensi}(y_i;\mathbf{x}_i;\delta_0)\right) &= \frac{1}{\mathbb{P}(Y=1)}\times \mathbb{E}_{y,\mathbf{x}}\left(y_i\hat{f}(\mathbf{x}_i )\right), \\
                \mathbb{E}_{y,\mathbf{x}}\left(G_{speci}(y_i;\mathbf{x}_i;\delta_0)\right) &= \frac{1}{\mathbb{P}(Y=0)}\times \mathbb{E}_{y,\mathbf{x}}\left(y_i\hat{f}(\mathbf{x}_i ) + 1-y_i - \hat{f}(\mathbf{x}_i) \right).
            \end{cases}
            \end{align*}
            According to the linearity axiom, we know that:
            \begin{align*}
                \phi_j\left(\mathbb{E}_{y,\mathbf{x}}\left(y_i\hat{f}(\mathbf{x}_i ) + 1-y_i - \hat{f}(\mathbf{x}_i)\right)\right) &= \phi_j\left(\mathbb{E}_{y,\mathbf{x}}\left(y_i\hat{f}(\mathbf{x}_i )\right)\right) \\
                &+ \phi_j\left(\mathbb{E}_{y,\mathbf{x}}\left(1-y_i\right)\right) \\
                &- \phi_j\left(\mathbb{E}_{y,\mathbf{x}}\left( \hat{f}(\mathbf{x}_i)\right)\right).
            \end{align*}
            Thus, we can see that if $\phi_j\left(\mathbb{E}_{y,\mathbf{x}}\left(1-y_i\right)\right)=0$ and $\phi_j\left(\mathbb{E}_{y,\mathbf{x}}\left(\hat{f}(\mathbf{x}_i)\right)\right)=0$, then the contribution of the feature $x_j$ to sensitivity will be equal to its contribution to specificity, up to a normalization factor. 
            As there are no features in the performance metric $PM = \mathbb{E}_{y,\mathbf{x}}\left(1-y_i\right)$, by applying XPER on this performance metric, it is straightforward to see that $\phi_j\left(\mathbb{E}_{y,\mathbf{x}}\left(1-y_i\right)\right)=0$. Regarding the second performance measure $PM=\mathbb{E}_{y,\mathbf{x}}\left(\hat{f}(\mathbf{x}_i)\right)$, by applying XPER, we can see that the benchmark equals the latter. Intuitively, since there is no difference between the performance and its benchmark value, this means that these features do not improve performance, so their contribution must be null. In a simple case with three explanatory variables, we can see that this holds true and understand why it does. Applying XPER, the feature contribution of feature $x_1$ to this performance is:
            \begin{align*}
                \phi_1 &= \frac{1}{3}\left(\mathcolor{red}{\mathbb{E}_{x_1}\mathbb{E}_{x_2,x_3}\left(
				\hat{f}(\mathbf{x}_i)\right)}-\mathbb{E}_{x_1,x_2,x_3}\left(\hat{f}(\mathbf{x}_i)\right)\right) \\
                &+ \frac{1}{6}\left(\mathcolor{blue}{\mathbb{E}_{x_1,x_2}\mathbb{E}_{x_3}\left(
				\hat{f}(\mathbf{x}_i)\right)}-\mathcolor{green}{\mathbb{E}_{x_2}\mathbb{E}_{x_1,x_3}\left(
				\hat{f}(\mathbf{x}_i)\right)}\right) \\
                &+ \frac{1}{6}\left(\mathcolor{green}{\mathbb{E}_{x_1,x_3}\mathbb{E}_{x_2}\left(
				\hat{f}(\mathbf{x}_i)\right)}-\mathcolor{blue}{\mathbb{E}_{x_3}\mathbb{E}_{x_1,x_2}\left(
				\hat{f}(\mathbf{x}_i)\right)}\right) \\
                &+ \frac{1}{3}\left(\mathbb{E}_{x_1,x_2,x_3}\left(
				\hat{f}(\mathbf{x}_i)\right)-\mathcolor{red}{\mathbb{E}_{x_2,x_3}\mathbb{E}_{x_1}\left(
				\hat{f}(\mathbf{x}_i)\right)}\right), \\
                \phi_1 &= 0.
            \end{align*}
            Similarly for $x_2$ and $x_3$, we can easily see that $\phi_2=\phi_3=0$. In a way, in this situation, the marginal effects calculated for the different coalitions cancel out. Thus, the contribution of each feature to the performance is null. Therefore, to decompose either the sensitivity or specificity, we only need to apply XPER on $PM=\mathbb{E}_{y,\mathbf{x}}\left(y_i\hat{f}(\mathbf{x}_i )\right)$, and divide the contributions by $\mathbb{P}(Y=1)$ or $\mathbb{P}(Y=0)$.        
            Let denote by $\tilde{\phi}_j$ the XPER value associated with feature $x_j$, measuring its contribution to $PM=\mathbb{E}_{y,\mathbf{x}}\left(y_i\hat{f}(\mathbf{x}_i )\right)$. Then, the feature contribution $x_j$ to the sensitivity and specificity are:
            \begin{align*}
                \begin{cases}
                    \phi_j^{sensi} = \frac{\tilde{\phi}_j}{\mathbb{P}(Y=1)}, \\
                    \phi_j^{speci} = \frac{\tilde{\phi}_j}{\mathbb{P}(Y=0)}.
                \end{cases}
            \end{align*}
            Second, to ease the interpretation of the XPER values, we usually divide their contribution by the difference between the performance metric and its benchmark. For the sensitivity and specificity, the difference between the performance metric and its benchmark is:
            \begin{align*}
                \begin{cases}
                    \mathbb{E}_{y,\mathbf{x}}\left(G_{sensi}(y_i;\mathbf{x}_i;\delta_0)\right) - \phi_0^{sensi} = \frac{\mathbb{E}_{y,\mathbf{x}}\left(y_i\hat{f}(\mathbf{x}_i )\right) - \mathbb{E}_{y}\mathbb{E}_{\mathbf{x}}\left(y_i\hat{f}(\mathbf{x}_i )\right)}{\mathbb{P}(Y=1)}, \\
                    \mathbb{E}_{y,\mathbf{x}}\left(G_{speci}(y_i;\mathbf{x}_i;\delta_0)\right) - \phi_0^{speci} = \frac{\mathbb{E}_{y,\mathbf{x}}\left(y_i\hat{f}(\mathbf{x}_i )\right) - \mathbb{E}_{y}\mathbb{E}_{\mathbf{x}}\left(y_i\hat{f}(\mathbf{x}_i )\right)}{\mathbb{P}(Y=0)},
                \end{cases}
            \end{align*}
            with
            \begin{align*}
                \begin{cases}
                    \phi_0^{sensi} = \frac{\mathbb{E}_{y}\mathbb{E}_{\mathbf{x}}\left(y_i\hat{f}(\mathbf{x}_i )\right)}{\mathbb{P}(Y=1)}, \\
                    \phi_0^{speci} = \frac{\mathbb{E}_{y}\mathbb{E}_{\mathbf{x}}\left(y_i\hat{f}(\mathbf{x}_i )\right) + 1 - \mathbb{E}(y_i) -\mathbb{E}(\hat{f}(\mathbf{x}_i))}{\mathbb{P}(Y=0)}. \\
                \end{cases}
            \end{align*}
            Therefore, the contribution (expressed in percentage) of the features to the sensitivity and specificity are identical:
            \begin{align*}
                    \frac{\phi_j^{sensi}}{\mathbb{E}_{y,\mathbf{x}}\left(G_{sensi}(y_i;\mathbf{x}_i;\delta_0)\right)- \phi_0^{sensi}} 
                    &= \frac{\tilde{\phi}_j}{\mathbb{E}_{y,\mathbf{x}}\left(y_i\hat{f}(\mathbf{x}_i )\right) - \mathbb{E}_{y}\mathbb{E}_{\mathbf{x}}\left(y_i\hat{f}(\mathbf{x}_i )\right)} \\
                    &= \frac{\phi_j^{speci}}{\mathbb{E}_{y,\mathbf{x}}\left(G_{speci}(y_i;\mathbf{x}_i;\delta_0)\right)- \phi_0^{speci}}. 
            \end{align*}
            \end{proof}
            
    \clearpage

    \section{XPER decomposition of the $R^2$ in a linear regression model}

    \subsection{Interpretation of the benchmark} \label{Appendix:Benchmark}

    One of the main advantages of the XPER decomposition is the intuitive interpretation of the benchmark $\phi_0$. This value represents the expected performance metric in a hypothetical scenario where the model is applied (without re-estimation) to a dataset in which the target variable $y$ is independent of all the features $\mathbf{x}$. 
        
    \noindent When using the $R^2$ performance metric for a linear regression model, this benchmark value corresponds to $\phi_0 = -R^2$. This result arises from the fact that the model $\hat{f}(\mathbf{x})$ is treated as fixed and pre-estimated on a training sample, while the XPER decomposition is applied to the test sample. If we were to \textit{estimate} the model $\hat{f}(\mathbf{x})$ on a hypothetical sample where $y$ is independent of $\mathbf{x}$, the benchmark $R^2$ would be zero asymptotically. However, in our case, the model is estimated on the training set where $y$ and $\mathbf{x}$ are correlated. We then apply this \textit{pre-trained} model to a hypothetical test set with spurious regressors (i.e., where $y$ is independent of $\mathbf{x}$), without re-estimating the model. This procedure yields a negative $R^2$, which reflects the opposite of the $R^2$ obtained on the test sample with real (correlated) data.
        
        \smallskip

        \noindent To illustrate this point, consider the linear regression model given by $y_i = \beta_0 + \beta_1 x_{1,i} + \beta_2 x_{2,i} + \varepsilon_i$, where the parameters are set to $\{\beta_0, \beta_1, \beta_2\} = \{0, 1, 3\}$. The features $\mathbf{x}_i=(x_{1,i},x_{2,i})$ are i.i.d. with $\mathbb{V}(x_j) = 1$, and the errors $\varepsilon_i$ are i.i.d. following a normal distribution $\mathcal{N}(0, \sigma^2_\varepsilon)$ with $\sigma^2_\varepsilon = 10$. We simulate a sample of size $n = 4,000$, using the first $3,000$ observations as the training sample and the remaining ones as the test sample. We estimate the model $y_i = \mathbf{x}_i \beta + \varepsilon_i$ using the training sample, resulting in the following parameter estimates $\{\hat{\beta}_0, \hat{\beta}_1, \hat{\beta}_2\} = \{-0.0500,0.9767, 3.0166\}$. The corresponding $R^2$ value for the training sample is $R^2_{\text{train}}=0.4994$.
        
        \smallskip

        \noindent Next, we use the estimated model to make predictions $\hat{y}_i$ on the test sample. We obtain an $R^2$ equal to $R^2_{\text{test}} = 0.5544$. Then, we compute the XPER values on the test sample, yielding the following decomposition:
        \begin{equation*}
            \underbrace{0.5544}_{R^2} = \underbrace{-0.4761}_{\hat{\phi}_0} + \underbrace{0.0948}_{\hat{\phi}       _1} + \underbrace{0.9357}_{\hat{\phi}_2}.
        \end{equation*}
        To understand the interpretation of the benchmark value $\hat{\phi}_0$, we propose the following experiment. We simulate two spurious regressors $z_1$ and $z_2$, with $\mathbb{V}(z_j) = 1$, that are uncorrelated with the target variable $y$, and with the features $x_1$ and $x_2$. 
        \begin{enumerate}
            \item We \textit{estimate} a linear regression model $y_i = \gamma_0 + \gamma_1 z_{1,i} + \gamma_2 z_{2,i} + \mu_i$, where $y_i$ corresponds to the simulated variable generated from the true DGP $y_i=\beta_0 + \beta_1 x_{1,i} + \beta_2 x_{2,i} + \varepsilon_i$, on a sample of size $n=3000$. We then apply this model on the test set of size $n=1000$ and compute the $R^2$. As expected, we obtain a value close to 0, specifically  $R^2_{\text{test}} = -0.0030$. The corresponding parameter estimates are close to zero with $\{\hat{\gamma}_0, \hat{\gamma}_1, \hat{\gamma}_2\} = \{-0.0202, -0.0424, 0.0939\}$.

            \smallskip
            
            \item We \textit{apply} the estimated model obtained from the training sample, i.e., $y_i=\hat{\beta}_0 + \hat{\beta}_1 x_{1,i} + \hat{\beta}_2 x_{2,i}$, with $\{\hat{\beta}_0, \hat{\beta}_1, \hat{\beta}_2\} = \{-0.0500,0.9767, 3.0166\}$ to a hypothetical test set where the model features are uncorrelated with the target. Formally, we set $x_1 = z_1$ and $x_2 = z_2$ and compute the fitted values $\hat{y}_i = -0.0500 + 0.9767 z_{1,i} + 3.0166 z_{2,i}$ using the parameters estimated from the training sample with the ``true'' regressors $x_1$ and $x_2$. It is as if the variables used to estimate the model in the training sample had suddenly become irrelevant in the test sample. The corresponding $R^2$ value is $-0.4538$, which is close to the benchmark estimate $\hat{\phi}_0$, and nearly the opposite of the test $R^2$ initially obtained on the test set. Indeed, the fact that the constant term in the regression is not re-estimated induces a negative $R^2$.
        \end{enumerate}

        \subsection{Interpretation of the individual XPER values}

        \label{R2_ind_XPER}
        Consider a linear regression model and the $R^2$ as sample performance metric such that $y_i = \mathbf{x}_i^T\beta + \varepsilon_i$,	with $\varepsilon_{i}$ i.i.d $ \mathcal{N}(0,4)$. We consider three i.i.d. features such that $\mathbf{x}_i^T=(x_{i,1},x_{i,2},x_{i,3})$ and $\mathbf{x} \sim \mathcal{N}(\mathbf{0},\mathbf{\Sigma})$ with $diag(\mathbf{\Sigma})=(8,4,1)$. The true vector of parameters is $\{\beta_0,\beta_1,\beta_2,\beta_3\} = \{0.2,1,1,1\}$ with $\beta_0$ the intercept. 

        \begin{table}[htbp]
          \centering
          \caption{Illustration of $R^2$ XPER values in a three-fold standard linear model}
            \begin{tabular}{ccccccccc}
            \hline
            & $G(y_i; \mathbf{x}_i; \hat{\delta}_n)$ & $\hat{\phi}_0$ & $\hat{\phi}_1$ & $\hat{\phi}_2$ & $\hat{\phi}_3$ & $y_i$ & $\hat{y}_i$ & $(y_i - \hat{y}_i)^2$  \\
            \hline
            i = 1 & 0.7875 & -0.3388 & 0.9287 & -0.2303 &  0.4279 &  3.2871 &  1.3815  &  3.6312 \\
            i = 2  &  0.0151 & 0.2491 & -0.2697 & 0.0454 & -0.0097 & -0.3773 & -4.4793  & 16.8266 \\
            i = 3  &  0.5367 & 0.0830 & -0.0161 & 0.1953 & 0.2745 & -1.6622 &  1.1513  &  7.9157 \\
            i = 4  &  0.9992 & -1.0697 & 1.8839 & -0.0317 & 0.2166 & 4.8534 &  4.7354  &  0.0139 \\
            i = 5  &  0.8478 & 0.1834 & 0.4011 & 0.0691 & 0.1942 & 1.2401  & -0.3725  &  2.6007 \\
            ...   &   ...   &  ...   &  ...  &   ...  &   ...  &   ...   &   ...   &    ... \\
            i = 996 &  0.4523 & 0.2511 & 0.1402 & 0.0947 & -0.0336 & -0.3395 & -3.3984  &  9.3570 \\
            i = 997 & 0.8812 & 0.2219 & 0.2834 & 0.2266 & 0.1494 & -0.7394 &  0.6851  &  2.0290 \\
            i = 998 & 0.9627 & -1.1944 & 0.9645 & 0.6049 & 0.5876 & -4.9032 & -4.1044  &  0.6380 \\
            i = 999 & 0.9607 & 0.2520 & 0.3705 & 0.2684 & 0.0698 & -0.3202  & 0.4997  &  0.6722 \\
            i = 1,000 & 0.9203 & 0.2449 & 0.5788 & 0.0221 & 0.0745 & 0.6174 & -0.5495 &  1.3617 \\
            \hline
            & 0.7629 & -0.7385 & 0.9649 & 0.4202 & 0.1163 & 0.1138 & 0.0843 & 4.0504 \\
            \hline
            \end{tabular}
          \label{R2_phi_i_j}
        \end{table}%

        \smallskip
        
	    \noindent We simulate a sample of size $2,000$. We use the first $T=1,000$ observations to estimate the model and the remaining ones as test sample $S_n$. Consider a simulation for which the estimated parameters are equal to $ \{\hat{\beta}_0,\hat{\beta}_1,\hat{\beta}_2,\hat{\beta}_3\} = \{0.1395, 1.0208, 0.9743, 1.0434\}$. The objective is to evaluate the contribution of each feature to the $R^2$ of the model $\hat{f}(\mathbf{x}_i)=\mathbf{x}_i^T\hat{\beta}$, computed on the test sample $S_n$. The estimated $R^2$ and feature contributions are the following:
    	\begin{equation*}
    		\underbrace{0.7629}_{R^2} =  \underbrace{-0.7385}_{\hat{\phi}_0} + \underbrace{0.9649}_{\hat{\phi}_1} + \underbrace{0.4202}_{\hat{\phi}_{2}} + \underbrace{0.1163}_{\hat{\phi}_{3}}.
    	\end{equation*}
	    As expected, the estimated benchmark $\hat{\phi}_0$ is negative. The difference between the sample $R^2$ and the benchmark is explained by the contribution of the features. We verify that the features contributing the most to the $R^2$ are, in decreasing order, $x_{1}$, $x_{2}$, and $x_{3}$. Given Equation \eqref{R2_theory}, differences in $\phi_j$ values across features come from either $\hat{\beta}_j$ or $\sigma_{y,x_j}$. Here, the ranking of XPER values is the same as the ranking of the covariances of the feature with the target variable as $\sigma_{y,x_1} > \sigma_{y,x_2} > \sigma_{y,x_3}$. 

        \smallskip
        
	    \noindent Local analysis of feature contributions to the $R^2$ is detailed in Table \ref{R2_phi_i_j}. In the second column, we report the individual contributions to the $R^2$ on the test sample. By definition, the average of individual contributions corresponds to the (test) sample $R^2$ equal to $0.7629$. Individuals for which the contribution $G(y_i;\mathbf{x}_i;\hat{\delta}_n)$ is larger than the $R^2$ are those for which the squared residual (reported in last column) is lower than MSE. For instance, individual $i=1$ has a contribution to the $R^2$ ($0.7875$) larger than the sample $R^2$ ($0.7629$) as well as a squared residual ($3.6312$) lower than the MSE ($4.0504$). Individual $i=4$ contributes more to the $R^2$ than individual $i=5$ ($0.9992$ > $0.8478$) because its squared residual is smaller than the latter  ($0.0139$ < $2.6007$). Thus, contribution $G(y_i;\mathbf{x}_i;\hat{\delta}_n)$ measures the ability of the model (as measured by the $R^2$) to  predict the target variable $y$ for a given individual on the test sample.

        \smallskip

	    \noindent Individual feature contributions to the $R^2$ are reported in columns 4, 5, and 6. A positive contribution $\phi_{i,j}$ means that the feature $x_j$ tends to improve the predictive ability  of the model for individual $i$, as measured by the $R^2$. For instance, for individual $i=4$ feature values $x_{4,1}$ and $x_{4,3}$ (respectively $4.8040$ and $0.5165$) contributes to reduce its squared residual, so $\phi_{4,1}$ and $\phi_{4,3}$ are positive (respectively $1.8839$ and $0.2166$). On the contrary, the second feature value $x_{4,2}$ ($-0.8694$) reduces the contribution to the $R^2$ of this individual, i.e., $\phi_{4,2}=-0.0317 < 0$.  The intuition behind this result is as follows: (i) The realization of feature $x_2$ is inferior to the average ($-0.8694$ < $0.0100$) and its coefficient $\hat{\beta}_2$ in the model is positive. (ii) Hence, $x_2$ tends to decrease the prediction $\hat{y}_4$ compared to the average whereas we observe that the target value $y_4$ ($4.8534$) is larger than the average ($0.1138$). (iii) Therefore, feature $x_2$ increases prediction error and so hinders model performance. 

        \begin{figure}[!ht]
            \centering
            \includegraphics[width=1\linewidth]{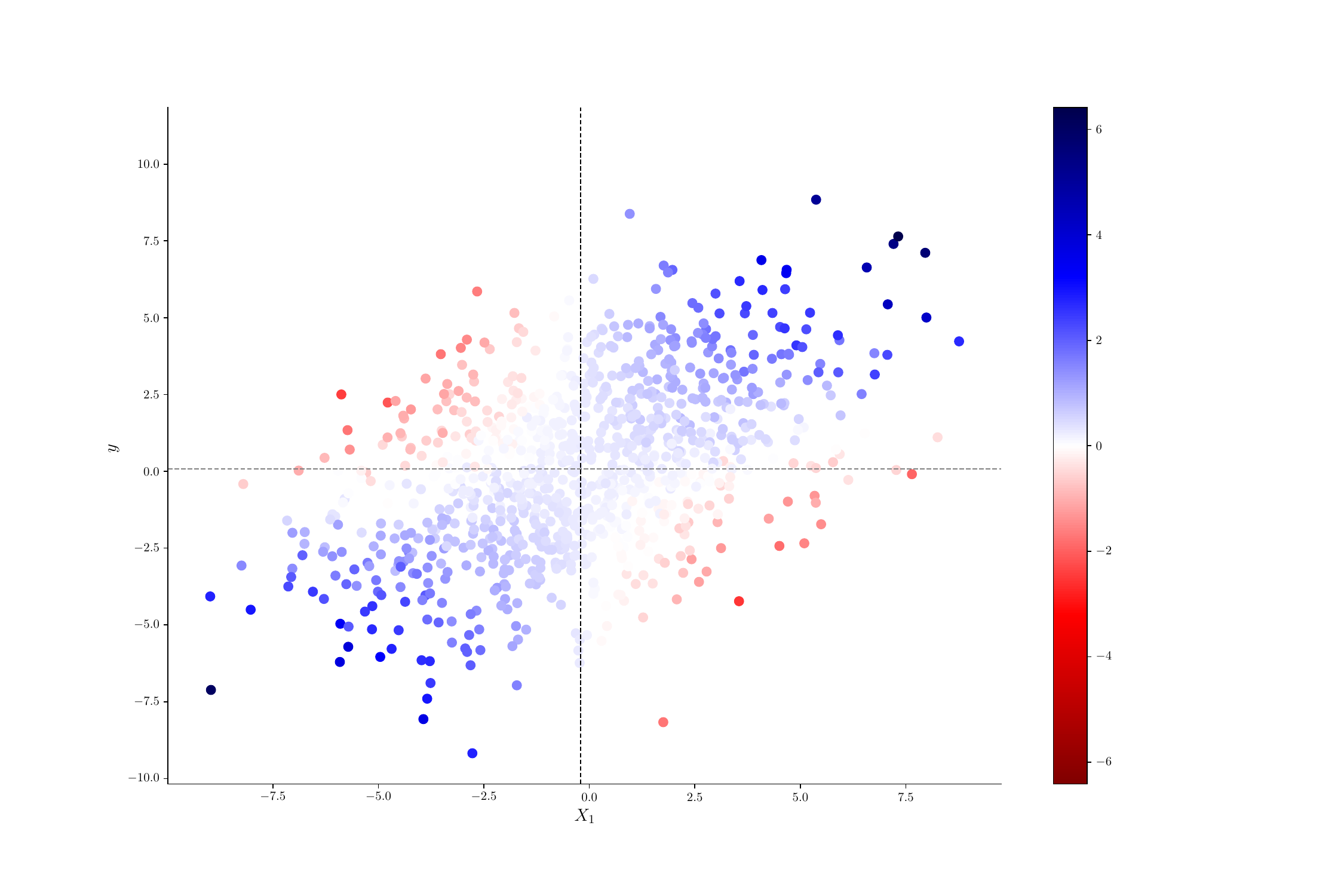}
            \caption{$R^2$ XPER values $\phi_{i,1}$ in a three-fold model}
            \begin{flushleft}
            {\footnotesize Note: This figure displays XPER values $\phi_{i,1}$ as a function of feature $x_1$ (x-axis) and target variable (y-axis). The vertical dotted line refers to the expected value of feature $x_1$ and the horizontal dotted line the one of the target variable.
            \par}
            \end{flushleft}
            \label{fig:R2_SHAP_X1_y}
        \end{figure}

        \begin{figure}[!ht]
            \centering
            \includegraphics[width=1\linewidth]{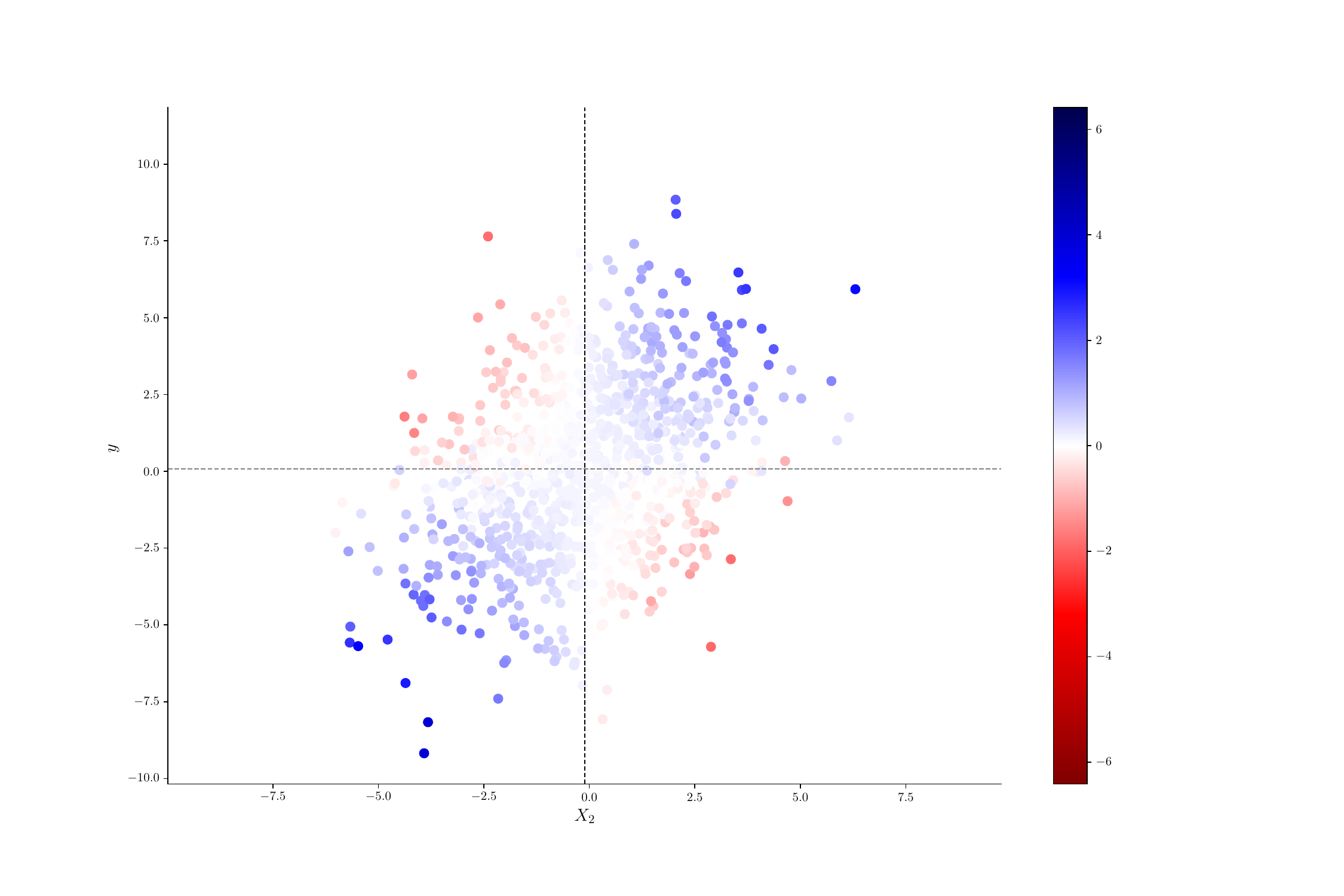}
            \caption{$R^2$ XPER values $\phi_{i,2}$ in a three-fold model}
            \medskip \medskip
        
        \justifying \noindent {\footnotesize Note: This figure displays XPER values $\phi_{i,1}$ as a function of feature $x_1$ (x-axis) and target variable (y-axis). The vertical dotted line refers to the expected value of feature $x_1$ and the horizontal dotted line the one of the target variable.}
            \label{fig:R2_SHAP_X2_y}
        \end{figure}
        
        \smallskip
        
    	\noindent One advantage of this local analysis is to reveal heterogeneity of the model predictive performance. The model better predict the target variable $y$ for some individuals than for others. Our methodology allows to understand why the model is more or less suited for individuals in the test sample, through individual feature contributions $\phi_{i,j}$. 
    	The scatter plot in Figure \ref{fig:R2_SHAP_X1_y} illustrates this heterogeneous analysis. We report the realizations of feature $x_1$ on the x-axis and target values on the y-axis. Blue (red) dots refer to positive (negative) individual feature contributions. We verify that the closer is the characteristic $x_1$ to its average, the less it contributes to the $R^2$, as indicated by the fading color of the dots. On the contrary, when the characteristic $x_1$ moves away from its expected value, individual feature contributions are either positive or negative. In order to understand the intuition behind this result let us assume that the target variable represents the consumption and the first feature corresponds to wages. In the estimated model, the consumption is positively correlated to wages as $\hat{\beta}_1=0.1395> 0$. Therefore, according to the model, individuals with large wages (compared to the average) are likely to have large consumption (compared to the average). Consider an individual with a large (low) wage and a large (low) consumption, i.e., an individual belonging to the  top right (bottom left) panel. In this case, the wage contribution to the $R^2$ is positive (blue dots) as the observed consumption is consistent with the model prediction. On the contrary, if an individual has an above-average salary (top left panel), then this feature misleads the model and thus the corresponding XPER value of the wage turns negative (red dots). 

        \smallskip
        
    	\noindent Finally, the larger is the feature deviation from the average, the larger is its contribution to $G(y_i;\mathbf{x}_i;\hat{\delta}_n)$. This result is highlighted in Figure \ref{fig:R2_SHAP_X1_y} by the darkest dots associated with individuals with both, large deviation from the average and large contribution to $G(y_i;\mathbf{x}_i;\hat{\delta}_n)$. Therefore, the larger the variance of the feature, the larger is its contribution, in absolute value, to the $R^2$. As illustrated by Figures \ref{fig:R2_SHAP_X1_y} and \ref{fig:R2_SHAP_X2_y}, given that $\mathbb{V}(x_1) > \mathbb{V}(x_2)$, feature $x_1$ contributions are larger than feature $x_2$. The corresponding weights are reported in Table \ref{Global_analysis} in the paper. 
        
    \clearpage

    \section{Simulations: Explaining overfitting}\label{Simulations_Appendix} In Appendices \ref{overfit_1} and \ref{overfit_2}, we propose two additional Monte Carlo simulation experiments illustrating how XPER values can be used to detect the origin of overfitting. The latter can arise for at least two reasons: (1) an improper control of the bias-variance trade-off through model hyperparameters, or (2) a shift of the feature distributions between the training and test samples. We illustrate each of these two cases in the following subsections.

    \subsection{Case 1: Improper control of the bias-variance trade-off}
    \label{overfit_1}
    Consider a DGP given by  $y_i = \mathbf{1}(y^*_i>0)$
	with $y_i^{*} = \mathbf{\omega}_i\beta + \varepsilon_{i}$ a latent variable, $\mathbf{\omega}_i=(1:\mathbf{x}^{'}_i)$, and $\varepsilon_{i}$ an i.i.d. error term with $\varepsilon_{i}\sim \mathcal{N}(0,1)$. We consider three independent features such that $\mathbf{x}_i \sim \mathcal{N}(\mathbf{0},\mathbf{\Sigma})$ with $diag(\Sigma)=(1.3,1.2,1.1)$. The true vector of parameters is $\beta=(\beta_0,\beta_1,\beta_2,\beta_3)^{'} = (0.05,0.5,0.5,0.5)^{'}$ with $\beta_0$ the intercept. We generate $K=5,000$ pseudo-samples $\lbrace y_i^s,\mathbf{x}_i^s \rbrace_{i=1}^{T+n}$ of size $1,00$0 using this DGP. Then, we estimate a decision tree using 5-fold cross validation on the first $T=700$ observations of each pseudo-sample and we use the remaining $n=300$ observations as a test sample. In order to intentionally generate overfitting, we impose a minimum tree-depth of 6 nodes for only three features in the model. For each trained model, we implement XPER to decompose the effect of the features on the AUC of the training and test samples. We display in Figure \ref{AUC_overfitting} the empirical distributions of the AUC. As expected, the trained tree models are overfitting the data, illustrated by the relatively low AUC values obtained on the test samples compared to the training samples. The empirical distributions of the XPER values reported in other panels of Figure \ref{illustration_2a_AUC} show that this drop in performance does not come from a particular feature. Indeed, the XPER contributions to the AUC are relatively close between the training and the test sample for all features. Thus, when overfitting is due to an improper control of the bias-variance trade-off, we observe a large decrease of the performance metric along with a stability of XPER values between the training and the test sample. Therefore, XPER can be used as a reverse engineering tool to detect wrong settings of hyperparameters.

		\begin{figure}[!h]
            \begin{subfigure}[b]{0.49\textwidth}
            \begin{center}
                \includegraphics[width=0.6%
			\textwidth,trim={0 0.8cm 0 0},clip]{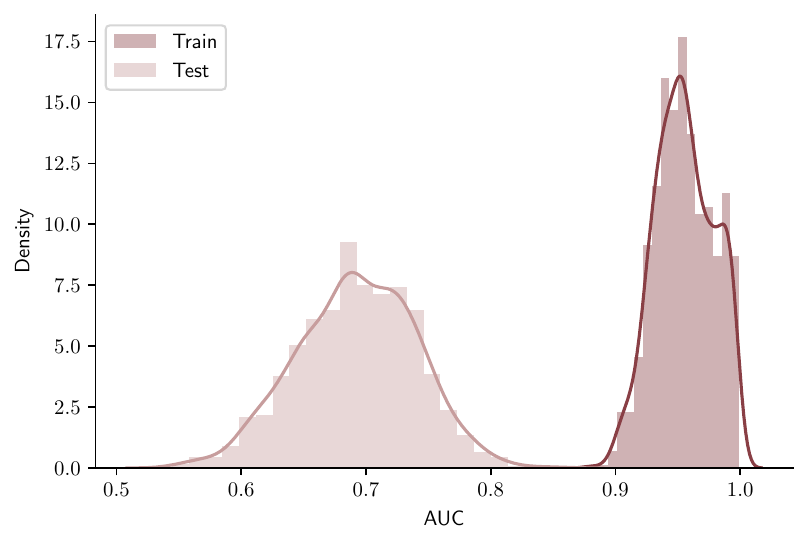}
			\caption{AUC}
			\label{AUC_overfitting}
            \end{center}
			\end{subfigure}
			\hfill
			\begin{subfigure}[b]{0.49\textwidth}
			\makebox[\linewidth]{
			\includegraphics[width=0.6%
			\textwidth,trim={0 0.8cm 0 0},clip]{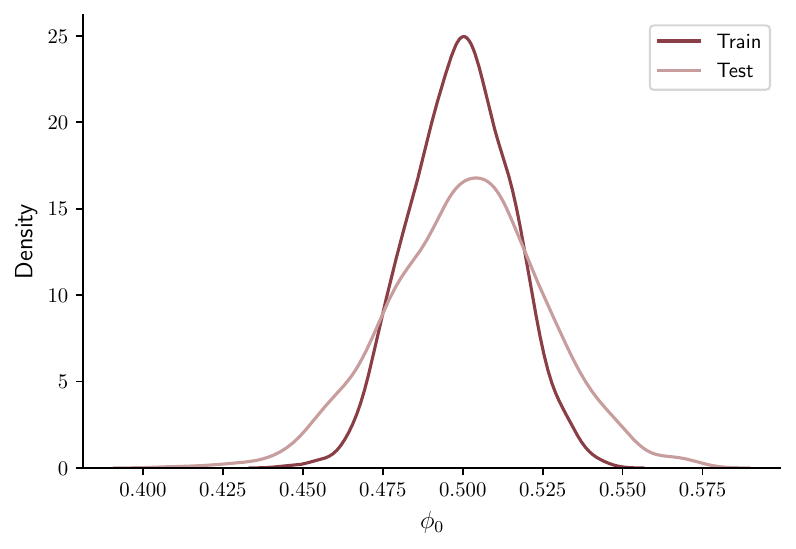}} %
			\caption{$\hat{\phi}_0$}
			\end{subfigure}
			\hfill
			\begin{subfigure}[b]{0.3\textwidth}
			\makebox[\linewidth]{
			\includegraphics[width=1%
			\textwidth,trim={0 0.8cm 0 0},clip]{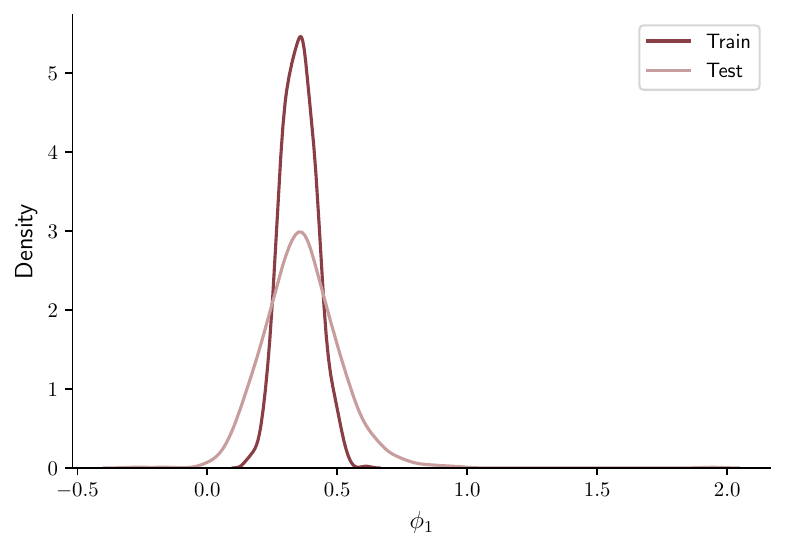}} %
			\caption{$\hat{\phi}_1$}
			\end{subfigure}
			\hfill
			\begin{subfigure}[b]{0.3\textwidth}
			\makebox[\linewidth]{
			\includegraphics[width=1%
			\textwidth,trim={0 0.8cm 0 0},clip]{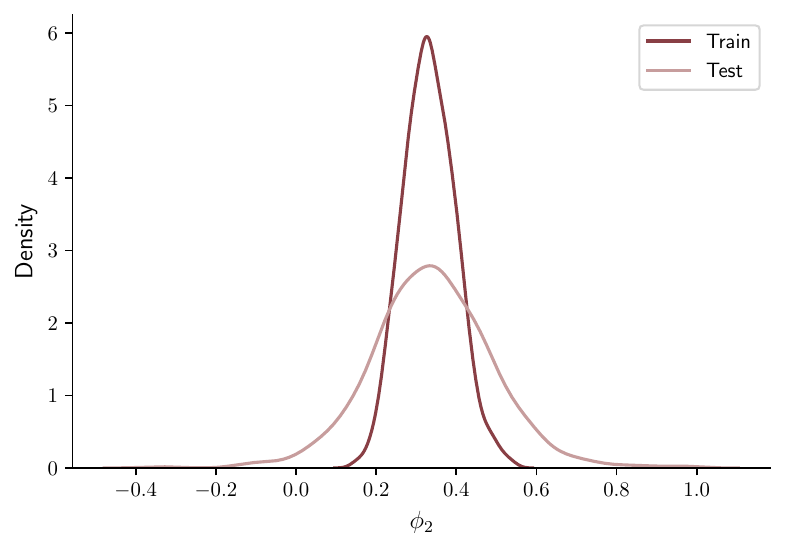}} %
			\caption{$\hat{\phi}_2$}
			\end{subfigure}
			\hfill
			\begin{subfigure}[b]{0.3\textwidth}
			\makebox[\linewidth]{
			\includegraphics[width=1%
			\textwidth,trim={0 0.8cm 0 0},clip]{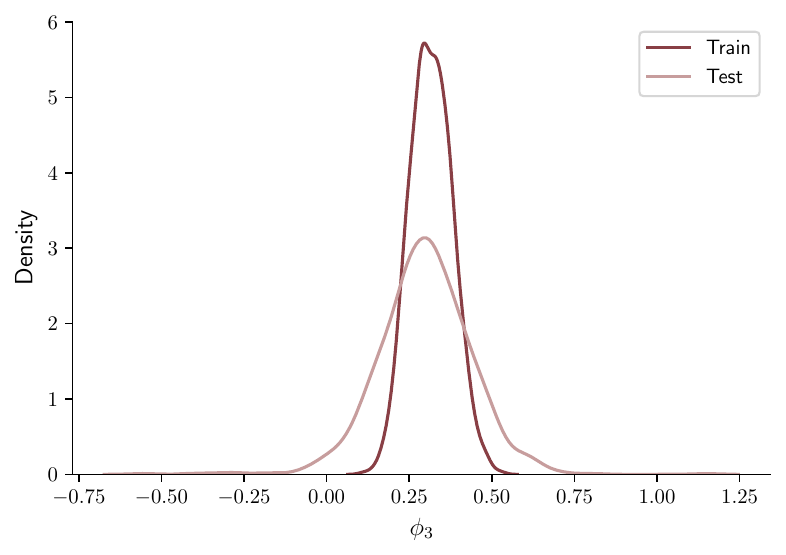}} %
			\caption{$\hat{\phi}_3$}
			\end{subfigure}
      \caption{Empirical distributions of AUC and XPER values in case of overfitting due to improper control of the bias-variance trade-off}
      \label{illustration_2a_AUC}
      \medskip \medskip
        
        \justifying \noindent {\footnotesize Note: This figure displays the empirical distributions of the AUC and XPER values on the training (dark color) and test (light color) sample according to the framework detailed in Illustration 2, case 1. XPER values are divided by the difference between the AUC of the model and the benchmark value to be comparable between the training and the test sample. The solid lines refer to kernel density estimations.}
	\end{figure}

    \subsection{Case 2: Shift of the feature distribution}
    \label{overfit_2}
     Overfitting can also arise from a shift of the feature distributions between the training and the test sample. To illustrate this origin of overfitting, we consider two distinct DGPs for the training and the test sample. For the former, we keep the same DGP as in the first case. 
    For the test sample, we assume an increase in the variance of the first feature while keeping other parameters unchanged, such that $diag(\mathbf{\tilde{\Sigma}})=(3,1.2,1.1)$. In the context of time series, such shift in the variance can come from a structural change. See \cite{Perron2021} on how to detect forecasting performance changes with structural change tests. As in case 1, we generate $K=5,000$ pseudo-samples $\lbrace y_i^s,\mathbf{x}_i^s \rbrace_{i=1}^{T+n}$ of size 1,000 ($T=700$ and $n=300$). For each pseudo-sample, we estimate a decision tree with a depth between 1 to 5 using 5-fold cross validation. Setting a relatively low tree depth avoids overfitting due to an improper control of the bias-variance trade-off.

           \begin{figure}[!h]
            \begin{subfigure}[b]{0.49\textwidth}
            \begin{center}
                \includegraphics[width=0.6%
			\textwidth,trim={0 0.8cm 0 0},clip]{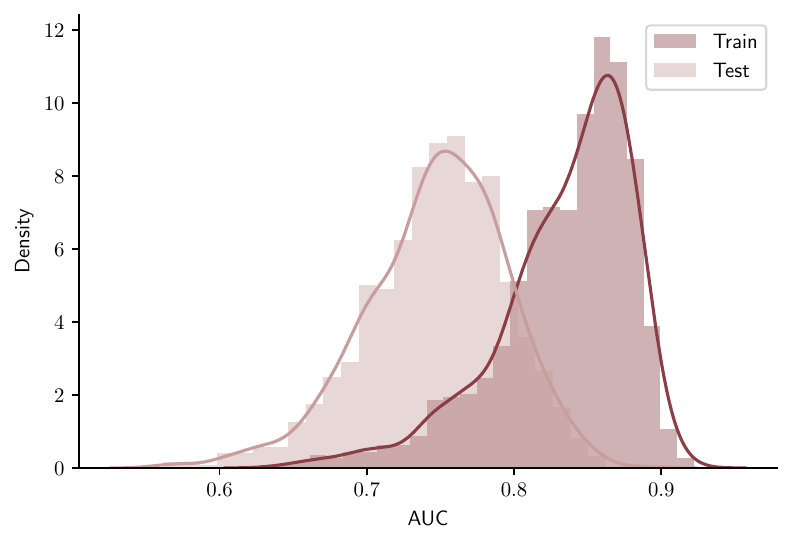} %
			\caption{AUC}
			\label{AUC_overfitting_2}
            \end{center}

			\end{subfigure}
			\hfill
			\begin{subfigure}[b]{0.49\textwidth}
			\makebox[\linewidth]{
			\includegraphics[width=0.6%
			\textwidth,trim={0 0.8cm 0 0},clip]{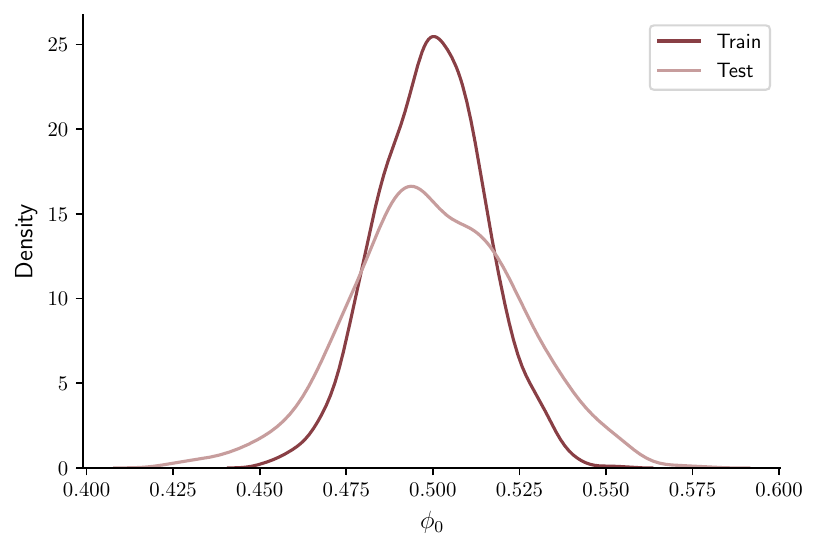}} %
			\caption{$\hat{\phi}_0$}
			\end{subfigure}
			\hfill
			\begin{subfigure}[b]{0.3\textwidth}
			\makebox[\linewidth]{
			\includegraphics[width=1%
			\textwidth,trim={0 0.8cm 0 0},clip]{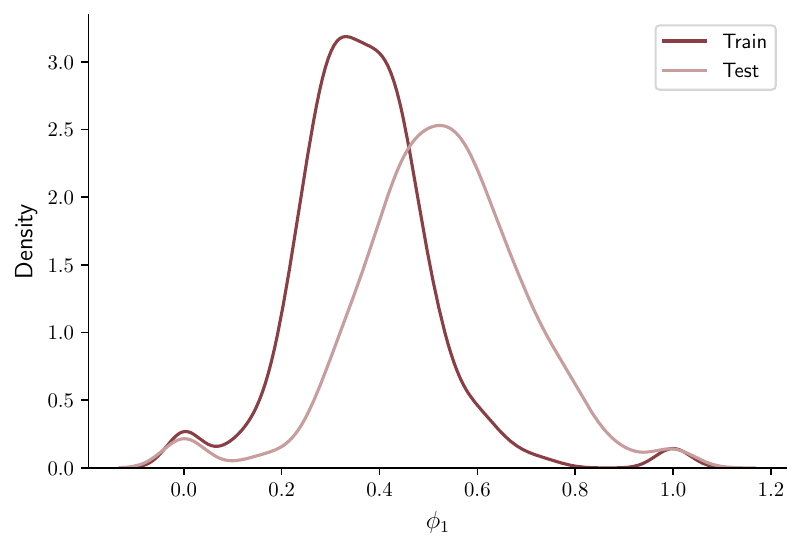}} %
			\caption{$\hat{\phi}_1$}
			\end{subfigure}
			\hfill
			\begin{subfigure}[b]{0.3\textwidth}
			\makebox[\linewidth]{
			\includegraphics[width=1%
			\textwidth,trim={0 0.8cm 0 0},clip]{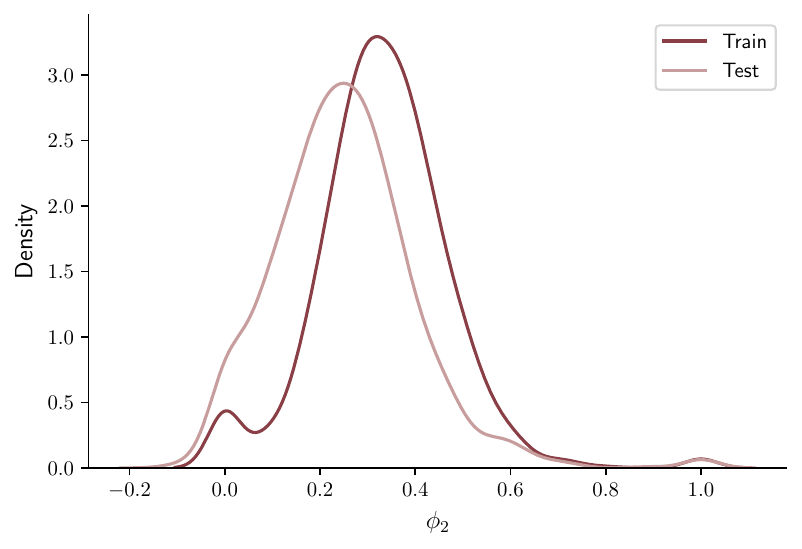}} %
			\caption{$\hat{\phi}_2$}
			\end{subfigure}
			\hfill
			\begin{subfigure}[b]{0.3\textwidth}
			\makebox[\linewidth]{
			\includegraphics[width=1%
			\textwidth,trim={0 0.8cm 0 0},clip]{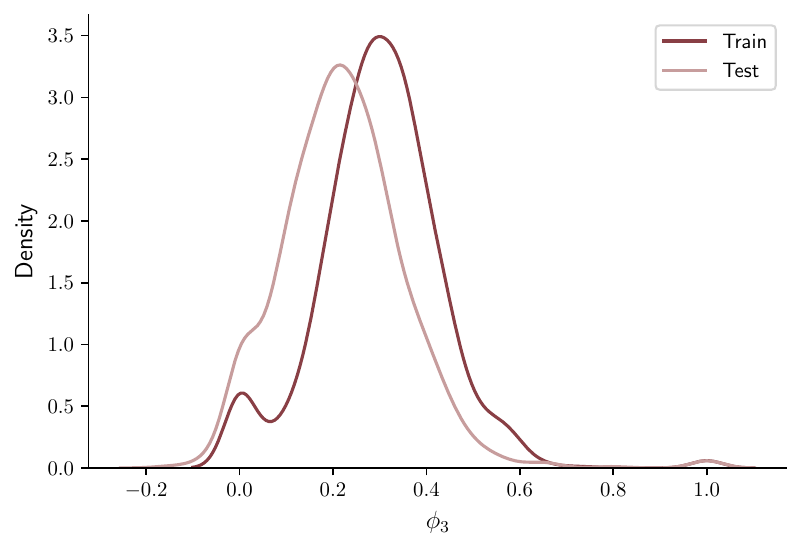}} %
			\caption{$\hat{\phi}_3$}
			\end{subfigure}
   \caption{Empirical distributions of AUC and XPER values in case of overfitting due to a shift of the distribution of the features}
		\label{illustration_2b_AUC}
  \medskip \medskip
        
        \justifying \noindent {\footnotesize Note: This figure displays the empirical distributions of the AUC and XPER values on the training (dark color) and test (light color) sample according to the framework detailed in Illustration 2, case 2. XPER values are divided by the difference between the AUC of the model and the benchmark value to be comparable between the training and the test sample. The solid lines refer to kernel density estimations.}
	\end{figure}
 
     In Figure \ref{AUC_overfitting_2}, we observe a decrease in AUC between the training and test samples. Contrary to the previous case, this decrease is due to the shift of the  distribution of $x_1$ which has also an impact on XPER values. More precisely, in Figure \ref{illustration_2b_AUC}, we observe that the contribution of feature $x_1$ to the AUC increases from the training to the test sample whereas the contribution of the other features decreases. Thus, observing both a drop in the performance of the model and some variations in the XPER values from the training to the test sample can indicate a change in the data structure. Such change is not captured by the model and not related to hyperparameter settings.  
     


\clearpage

\section{Empirical application: Some additional results}

\subsection{Summary statistics and features distribution} 

\begin{table}[htbp]
  \centering
  \caption{Summary Statistics}
  \resizebox{0.9\textwidth}{!}{ 
        \begin{tabular}{l|cccccccc}
    \hline
          & Count & Mean  & Std.  & Minimum & 25\%  & 50\%  & 75\%  & Maximum \vspace{0.1cm} \\ 
    \hline
    Job tenure & 7,440  & 9.3298 & 9.9787 & 0     & 2     & 5     & 15    & 58 \vspace{0.1cm} \\
    Age & 7,440  & 45.1691 & 14.7965 & 18    & 33    & 46    & 55    & 89 \vspace{0.1cm} \\
    Car price & 7,440  & 12,935 & 6,204 & 546 & 8,149 & 11,950 & 16,500 & 47,051 \vspace{0.1cm}\\
    Funding amount & 7,440  & 11,461 & 6,019 & 546 & 6,846 & 10,382 & 15,000 & 30,000 \vspace{0.1cm}\\
    Loan duration & 7,440  & 56.2176 & 19.3833 & 6     & 48    & 60    & 72    & 96 \vspace{0.1cm} \\
    Monthly payment & 7,440  & 0.1051 & 0.0611 & 0.0051 & 0.0690 & 0.0947 & 0.1304 & 2.6300 \vspace{0.1cm} \\
    Downpayment & 7,440  & 0.0897 &       & 0     &       &       &       & 1 \vspace{0.1cm} \\
    Credit event & 7,440  & 0.0220 &       & 0     &       &       &       & 1 \vspace{0.1cm}\\
    Married & 7,440  & 0.5347 &       & 0     &       &       &       & 1 \vspace{0.1cm}\\
    Homeowner & 7,440  & 0.3848 &       & 0     &       &       &       & 1 \vspace{0.1cm}\\
    Default & 7,440  & 0.2000 &       & 0     &       &       &       & 1 \vspace{0.1cm}\\
    \hline
    \end{tabular}}
  \label{tab:summary_stats}%
  \medskip \medskip
        
        \justifying \noindent {\footnotesize Note: This table displays summary statistics for each feature used in the XGBoost model as well as the target variable. For each categorical feature, the standard deviation (Std.) and the quartiles (25\%, 50\% and 75\%) are not displayed.}
        \vspace{-0.5cm}
\end{table}%

\bigskip

\begin{figure}[!h]
			\begin{center}
			    \includegraphics[width=0.65%
			\textwidth,trim={0 0cm 0cm 0cm},clip]{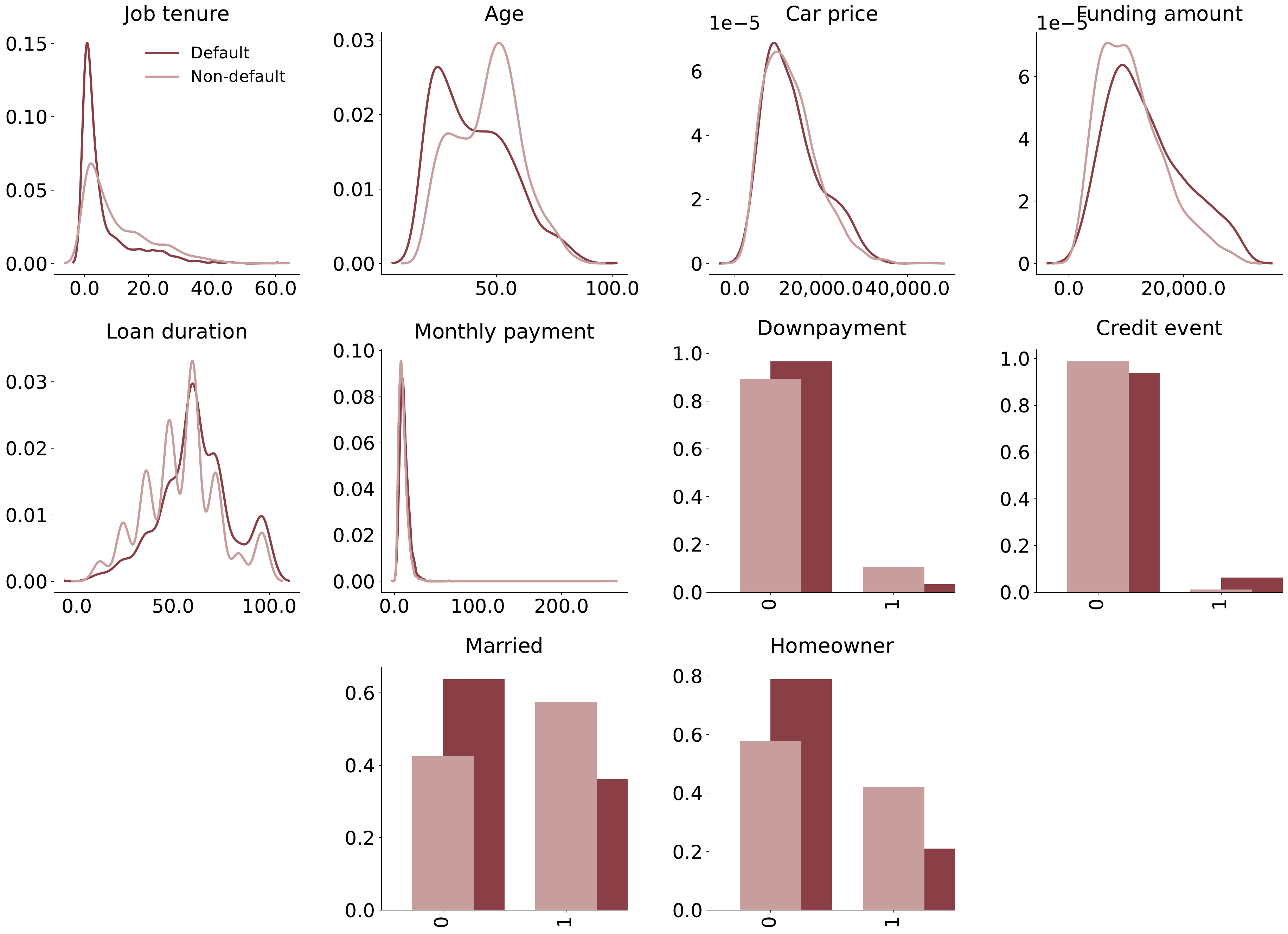} %
		\caption{Features distribution by default class}
  \label{fig:FeatureDistributions_by_default}
			\end{center} 
   \medskip \medskip
        
        \justifying \noindent {\footnotesize Note: This figure displays the distribution of the features by default class on the training sample, using kernel density estimation for continuous features. Dark red refers to defaulting borrowers and light red to non-defaulting borrowers.}
	\end{figure}

\clearpage
\subsection{XPER decomposition}
	\begin{figure}[!h]
		\begin{center}
		    			\includegraphics[width=0.45%
			\textwidth,trim={0 0 0 0},clip]{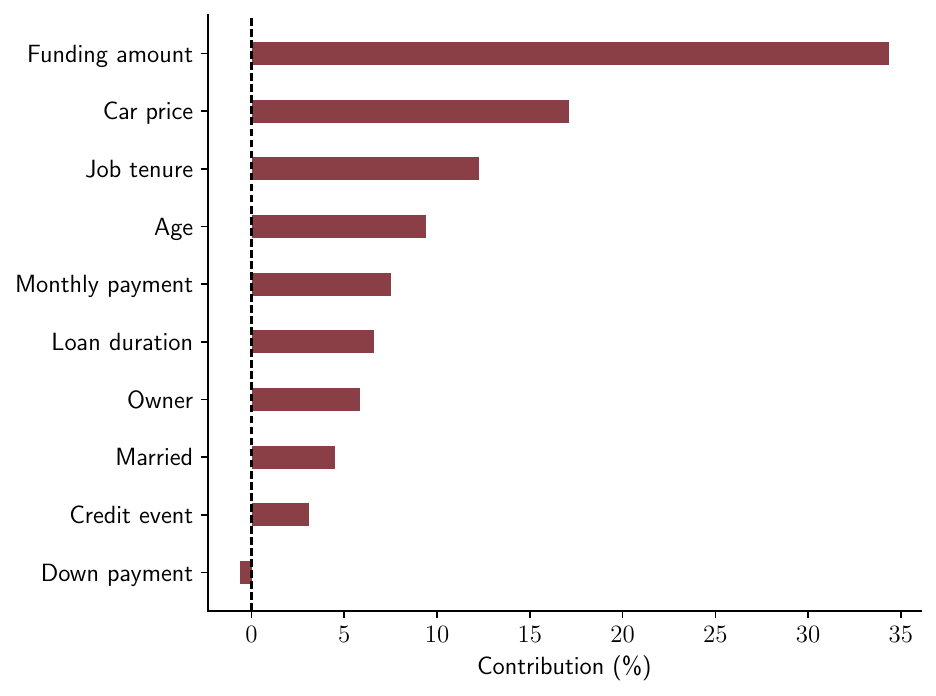}
             \caption{XPER decomposition of the AUC}
             \label{fig:AUC_ESHAP_Empirical_train} 
		\end{center}
     \begin{flushleft}
        {\footnotesize Note: This figure displays the XPER values of the AUC of the XGBoost model estimated on the training sample.\par}
        \end{flushleft}
        \vspace{-0.75cm}
	\end{figure}

\subsection{XPER decomposition for various predictive performance metrics}

\noindent In this section, we show that we can easily calculate and compare XPER values across different performance metrics. To illustrate this, we have applied XPER to decompose all the classification performance metrics listed in Table \ref{all_perfs}: AUC, Brier score, accuracy, balanced accuracy, sensitivity, specificity, and precision. Figure \ref{fig:XPER_all_metrics} displays the XPER values for these seven metrics, highlighting several key insights.
    
\noindent First, some model features have contrasting effects on different metrics. For instance, \textit{Car price} negatively impacts sensitivity on the test set, as indicated by its negative XPER value. This suggests it reduces the model's ability to correctly identify true credit defaults. However, this feature improves precision, which measures the proportion of true positive predictions out of all positive predictions. This example demonstrates how XPER values offer a nuanced understanding of how features affect performance across various metrics.
    
\noindent Second, the XPER values (expressed as percentages) for accuracy, balanced accuracy, sensitivity, and specificity are exactly equal. This equivalence can be attributed to the shared underlying structure of these metrics. As detailed in Table \ref{all_dec}, the differences between these metrics and their benchmarks depends on the covariance between the target variable and the predictions, $Cov(y,\hat{f}(x))$, along with a normalization factor independent of the features. For example, this normalization factor is $1/\mathbb{P}(Y=1)$ for sensitivity and $1/\mathbb{P}(Y=0)$ for specificity (see Appendix \ref{sensitivity_specificity} for more details). Thus, XPER measures the contribution of features to this normalized covariance, explaining why the values are identical when expressed as percentages, as the normalization factor cancels out.
    
\noindent Lastly, regardless of the metric considered, two features —\textit{Funding amount} and \textit{Job tenure}— account for a significant portion of the model's predictive performance, explaining at least 56\% of the difference between the performance and its benchmark value. Beyond these top-ranked features, the ranking of other features varies across metrics. For example, \textit{Car price} ranks as third and fourth for AUC and precision, respectively, but is either the last or second-to-last for other metrics. Similarly, \textit{Age} ranks fourth for AUC and ninth for precision, highlighting the diverse influence features can have across different metrics.
        
         \begin{figure}[!ht]
            \begin{subfigure}[b]{1\textwidth}
            \makebox[\linewidth]{
			\includegraphics[width=0.7%
			\textwidth,trim={0 0 0 0},clip]{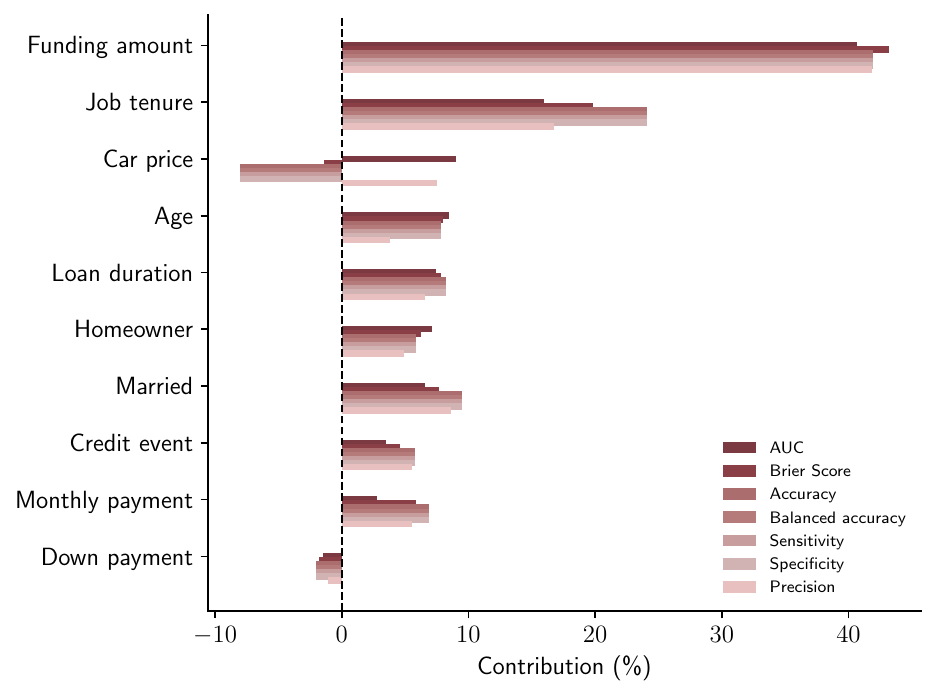}}
		      \label{fig:XPER_all_metrics}
			\end{subfigure}
			\hfill
            \vspace{-1cm}
            \caption{XPER decompositions for seven predictive performance metrics}
            \label{fig:XPER_all_metrics}
            \medskip \medskip
        
        \justifying \noindent {\footnotesize Note: This figure displays the XPER values of seven classification performance metrics for the XGBoost model: (1) AUC, (2) Brier Score, (3) Accuracy, (4) Balanced accuracy, (5) Sensitivity, (6) Specificity, and (7) Precision. The results are obtained on the test sample.}
            \vspace{-0.75cm}
	      \end{figure}

\subsection{XPER decomposition for imbalanced datasets}

\label{appendix:imbalanced}

The issue of imbalanced data is well known in machine learning, particularly in credit scoring, as it affects both the learning capacity of models (predictions tend to favor the majority class) and the validity of certain evaluation metrics (e.g., accuracy, sensitivity, specificity). However, these are not the only consequences. A recent study by \cite{Chen2024} highlights another issue: the instability of interpretability methods such as LIME and SHAP under extreme class imbalance. Specifically, they demonstrate that feature importance rankings generated by these methods become unstable as class imbalance increases, and they also observe greater variability in the absolute SHAP values for the same feature in low-default scenarios.
        
            \medskip
            
\noindent Recognizing the importance of this issue, we now show how XPER can be used and behaves for imbalanced datasets. First, we illustrate how the model’s performance and its XPER decomposition vary when the average default rate changes from 5\% to 30\% by applying undersampling to the initial dataset. Second, we apply XPER decomposition to sensitivity and specificity, and compare the results with those obtained for the AUC (see Figure \ref{fig:AUC_ESHAP_Empirical}).

            \medskip
            
\noindent \textbf{Varying the default rate.} We report in Table \ref{tab:XGBoost_Performances_def} and Figure \ref{fig:XPER_AUC_Defaults}, the AUC and the corresponding XPER decomposition obtained for our credit scoring empirical application when the average default rate is equal to 30\%, 20\%, 10\%, and 5\%. To decrease the default rates from the initial sample (20\%), we implement random undersampling, which involves removing some default observations from the initial sample at random. Conversely, to increase the default rate to 30\%, we remove non-default observations from the sample. Our findings are as follows: (i) the AUC varies between 0.7154 and 0.7884, (ii) significant variations in XPER values are observed when the default rate is as low as 5\%, compared to the other rates. Interestingly, even when the default rates are set to 5\% and 30\%, producing similar AUC values (0.7174 vs. 0.7154), the XPER decompositions differ significantly. This discrepancy arises because two distinct models can achieve the same AUC: (i) one that accurately predicts most defaults but performs poorly on non-defaults, and (ii) one that performs well on non-defaults but poorly on defaults. This demonstrates why it is essential to also evaluate the model’s sensitivity and specificity.
\bigskip 

            \begin{table}[ht!]
\centering
\caption{Model performances for different default rates}
\resizebox{0.8\textwidth}{!}{ 
\begin{tabular}{c|cccccc} \hline
    Default Rate & AUC & Brier Score  & Accuracy & BA & Sensitivity & Specificity \\ \hline
    5\% & 0.7174 & 0.0479  & 71.74 &  50.45 & 1.06 & 99.83\mbox{ } \\ 
    10\% & 0.7884 & 0.0771  & 90.53 &  61.66 & 25.50 & 97.82\mbox{ } \\ 
    20\% & 0.7521 & 0.1433  & 79.53 &  58.69 & 23.99 & 93.39\mbox{ } \\ 
    30\% & 0.7154 & 0.195  & 70.83 &  61.99 & 39.91 & 84.07\mbox{ } \\\hline
\end{tabular}}
\label{tab:XGBoost_Performances_def}
\medskip \medskip
        
        \justifying \noindent {\footnotesize Note: This table displays the performance metrics of an XGBoost model trained on datasets with default rates varying from 5\% to 30\%. BA stands for Balanced Accuracy.}
\end{table}

         \begin{figure}[ht!]
            \begin{subfigure}[b]{0.5\textwidth}
            \makebox[\linewidth]{
			\includegraphics[width=1%
			\textwidth,trim={0 0 0 0},clip]{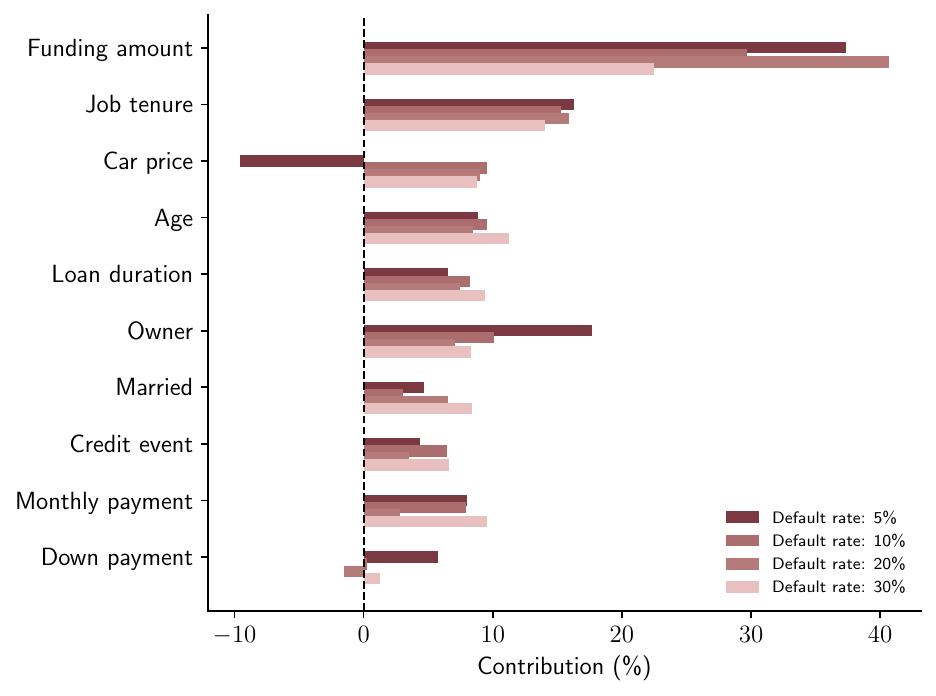}}
		\caption{AUC}
            \label{fig:XPER_AUC_Defaults}
			\end{subfigure}
			\hfill
			\begin{subfigure}[b]{0.5\textwidth}
			\makebox[\linewidth]{
			\includegraphics[width=1%
        			\textwidth,trim={0 0 0 0},clip]{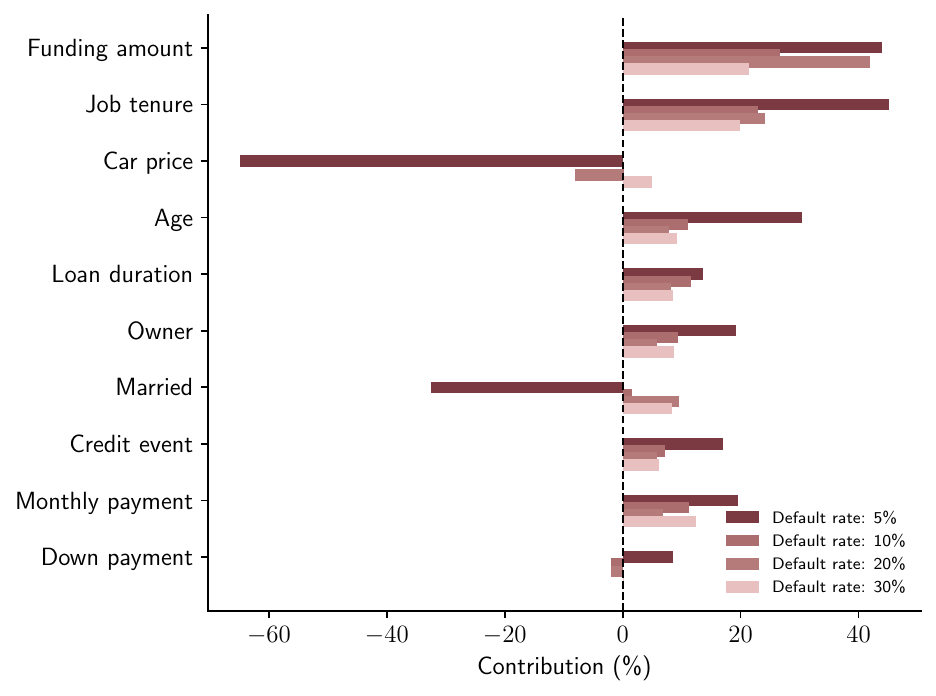}} 
        			\caption{Sensitivity/Specificity} 
           \label{fig:XPER_Sensitivity_Defaults}
			\end{subfigure}
   \hfill
   \vspace{-0.5cm}
    \caption{XPER decomposition of the AUC and Sensitivity/Specificity for multiple default rates}
    \label{fig:all}
     \medskip \medskip
        
        \justifying \noindent {\footnotesize Note: This figure displays the XPER values of the AUC (Panel (a)) and Sensitivity/Specificity (Panel (b)) considering four different test samples with default rates of 5\%, 10\%, 20\%, and 30\%. The results for the sample with a 20\% default rate for the AUC correspond to those shown in Figure 2a. Note that the results in Panel (b) correspond to both Sensitivity and Specificity, as we obtain identical XPER values (in \%) for these two performance metrics.}
	\end{figure}
            \noindent \textbf{Sensitivity and specificity.} We also decompose the sensitivity and specificity with XPER across different imbalance rates, ranging from 5\% to 30\%. We compare these results with the decomposition obtained for the AUC and report the performance metrics in Table \ref{tab:XGBoost_Performances_def}. As expected, the model's sensitivity significantly decreases as the dataset becomes more imbalanced, dropping from 25.50\% at a 10\% default rate to just 1.06\% at a 5\% default rate. Conversely, the model’s specificity increases as the default rate decreases, rising from 84.07\% at a 30\% default rate to 99.83\% at a 5\% default rate. This stark variation in sensitivity and specificity suggests that the underlying models trained on datasets with different default rates are fundamentally distinct. Consequently, we anticipate significant variations in XPER values for sensitivity and specificity, particularly at the lower default rate of 5\%. 
            
           \noindent We present the XPER values (in \%) for sensitivity and specificity in Figure \ref{fig:XPER_Sensitivity_Defaults}. Since these values are identical for both metrics, we display only one XPER value for each feature (see Appendix \ref{sensitivity_specificity} for further details). As expected, at a default rate of 5\%, the XPER values differ substantially from those at higher default rates. For instance, at a 5\% default rate, the \textit{Car price} feature drastically decreases the model's sensitivity/specificity ($\hat{\phi} < -60\%$), whereas its contribution is relatively negligible ($\hat{\phi} < -5\%$) for the other default rates. It is worth noting that while sensitivity and specificity vary across default rates of 10\%, 20\%, and 30\%, these variations are much less pronounced compared to the significant differences observed at a default rate of 5\%. 
           
	\subsection{Individual XPER decomposition and data visualization} \label{Individual_empirical}
	We analyze the impact of the various features on the performance metric but we now do it for each borrower individually. We start by analyzing in Figure \ref{fig:Force_plot} the XPER decomposition for two sample borrowers. These force plots enable us to decompose the individual performance of each borrower, as defined in Equation \eqref{ind_decompos}. By doing so, they allow us to understand why some individuals contribute more to the AUC of the model than others. In each panel of Figure \ref{fig:Force_plot}, \textit{Performance} refers to the contribution of the borrower to the AUC of the model and \textit{Benchmark} to their benchmark value, i.e., $\phi_{i,0}$ in Equation \eqref{ind_decompos}. For each borrower, the features increasing (respectively decreasing) the performance appear in red (blue). Borrower \#3 has a relatively high individual AUC compared to borrower \#28 (both have theoretically the same benchmark). The over-performance of borrower \#3 is mainly due to the large positive XPER values for \textit{funding amount}, \textit{job tenure}, and \textit{car price}. It also comes from the small negative XPER values for the marital status (\textit{married}) and the share of the monthly payment in the borrower's income (\textit{monthly payment}).
	 	 
		\begin{figure}[h]
        \begin{center}
	\centering
	\vspace{0.3cm}
        \begin{subfigure}[b]{1\textwidth}
			\includegraphics[width=1%
			\textwidth,trim={0 2.2cm 0 0},clip]{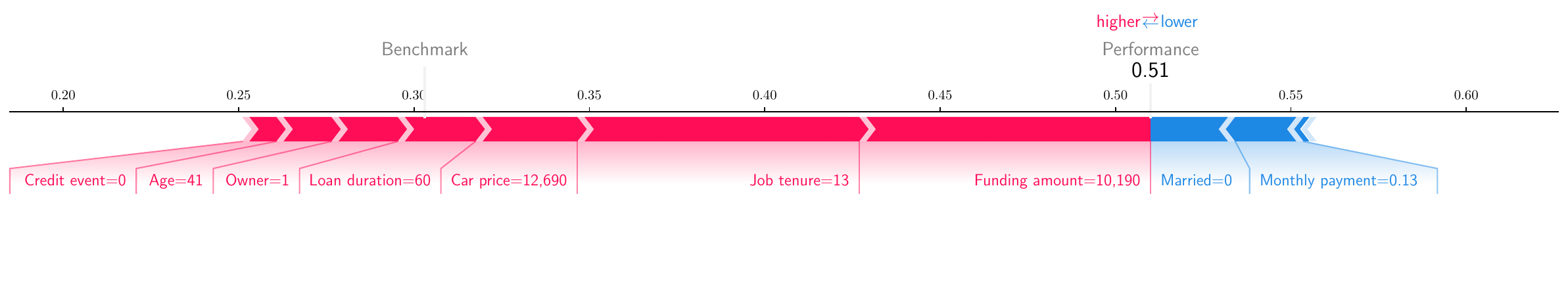} %
			\label{fig:E_SHAP_FP_non_def}
        \caption{Borrower \#3}
        \end{subfigure}
        \begin{subfigure}[b]{1\textwidth}
			\includegraphics[width=1%
			\textwidth,trim={0 2.2cm 0 0},clip]{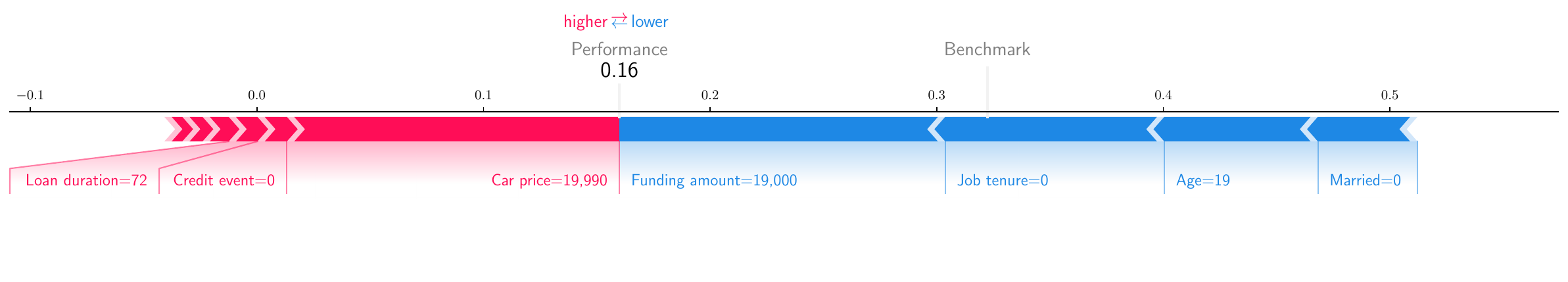} 
			\label{fig:E_SHAP_FP_non_def_obs}
        \caption{Borrower \#28}
        \end{subfigure}
        \caption{Force plots of individual XPER values}
        \label{fig:Force_plot}
        \end{center}
         \medskip \medskip
        
        \justifying \noindent {\footnotesize Note: This figure displays the XPER decomposition for the AUC of two loan borrowers (see Equation \eqref{DefLSHAP}). Borrower \#3 did not default on his loan and has a probability of default of 8\% according to the XGBoost model (Panel (a)). Borrower \#28 did not default on his loan and has a probability of default of 57\% according to the XGBoost model (Panel (b)). \textit{Performance} refers to the individual level of the AUC whereas \textit{Benchmark} represents the individual contribution to the AUC associated with a population where the target variable $y_i$ is independent from the features $\mathbf{x}_i$. The red color refers to positive XPER values, i.e., features increasing performance. The blue color refers to negative XPER values, i.e., features decreasing performance.}
  \end{figure}
To better understand the relative influence of each feature for the two borrowers, we analyze their risk-profiles and probabilities of default predicted by the model. Let us start with borrower \#3. He is 41 years old, homeowner, has a stable job, and applied for a loan to buy a moderately-priced car. He provided a down payment greater than 50\% of the car value and experienced no past credit event. Intuitively, we would naturally classify this borrower as low-risk and this is confirmed by the 8\% default probability estimated by the XGBoost model. Thus, as borrower \#3 eventually did not default on his loan, his contribution to the AUC is high. The situation of borrower \#28 is quite different as he exhibits a higher risk profile (young, jobless, not married, relatively large credit amount, no down payment). Yet, the model remains quite undecided about his capacity to pay back the loan with a 57\% estimated default probability. As the AUC measures the discriminatory ability of the model, this uncertainty leads to a low individual contribution, and even lower than the benchmark value.

	\begin{figure}[!h]
             \begin{center}
            			\includegraphics[width=1%
            			\textwidth,trim={0 0 0 0},clip]{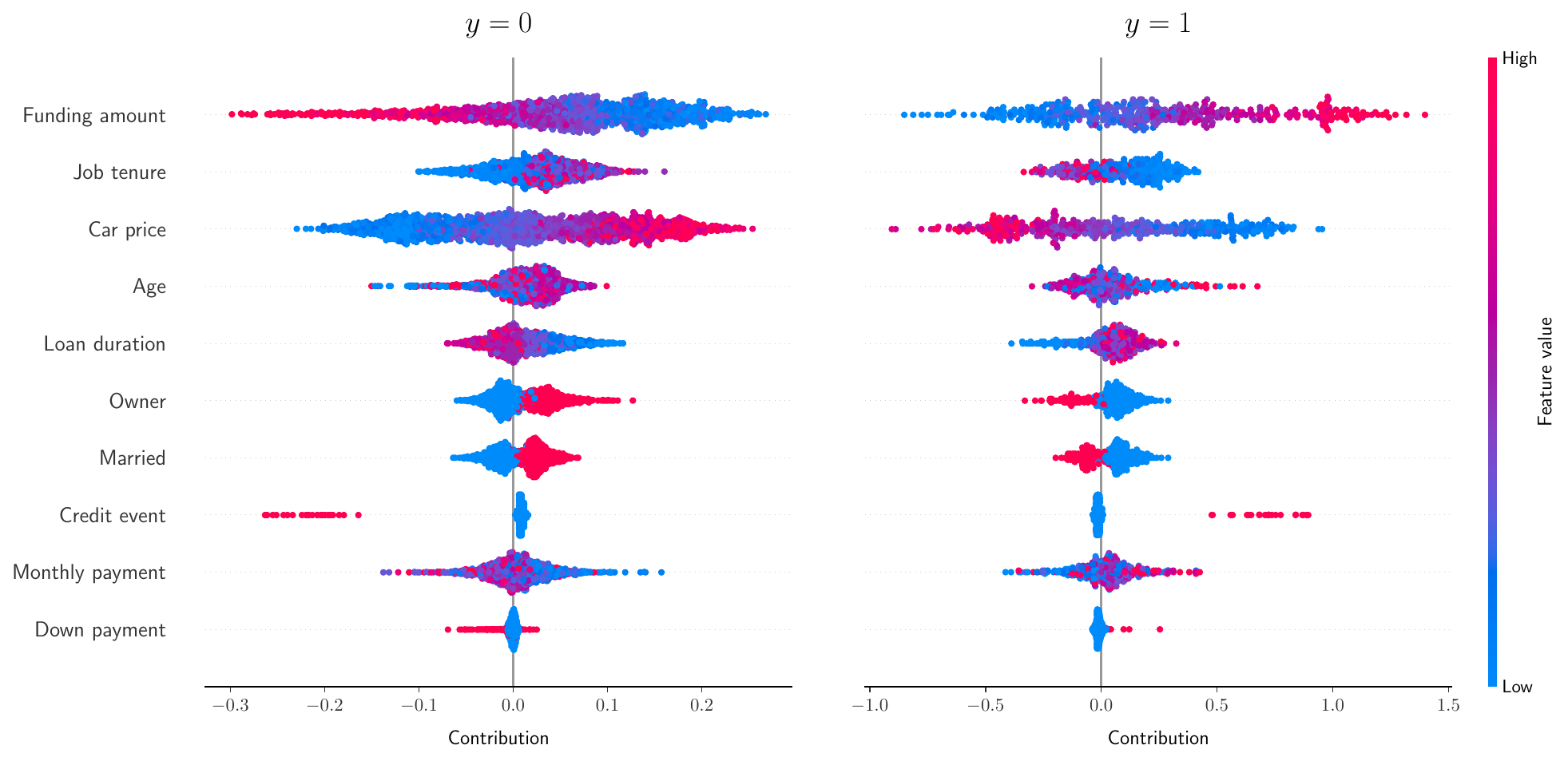} %
               \caption{Summary plots of individual XPER values}
            			\label{fig:EShapley_ind} 
             \end{center}
         \medskip 
        
        \justifying \noindent {\footnotesize Note: This figure displays the individual XPER values for each feature used in the XGBoost model. Each dot represents the value for a given borrower (see Equation \eqref{DefLSHAP}). We display the results for the borrowers who paid back their loans (left panel) and those who do not (right panel).}
	\end{figure}
	   We then consider the entire sample of borrowers. In Figure \ref{fig:EShapley_ind}, we display the XPER values for each feature as a function of the feature value. We analyze these results according to two types of borrowers: non-defaulting borrowers (y=0) and defaulting borrowers (y=1). We clearly see that depending on the value of the feature and the type of borrower, we know if this feature contributes to increase or decrease the performance of the model. For instance, for a non-defaulting borrower (left panel), a relatively high job tenure is associated with a positive XPER value. This result is due to the fact that a relatively long job tenure tends to  lower the probability of default in the model. Hence, this increases the ability of the model to distinguish him from the defaulting borrowers and boosts the XPER value. On the opposite, for a defaulting borrower (right panel), a relatively high job tenure leads to a negative XPER value and thus decreases his contribution to the AUC of the model.

\clearpage
 
    \subsection{Boosting model performance}
    \label{appendix_boosting}

        \begin{table}[ht]
            \centering
            \caption{Silhouette scores for the clustering based on XPER values and features }
            \begin{tabular}{cccc}
            \toprule
            Number of Clusters & XPER clustering & Feature clustering \\
            \midrule
            2                  & 0.2037                & 0.2129             \\
            3                  & -0.0244               & 0.1846             \\
            4                  & -0.0608               & 0.1718             \\
            5                  & -0.1858               & 0.1609             \\
            6                  & -0.0624               & 0.1589             \\
            7                  & -0.0734               & 0.1595             \\
            8                  & -0.0596               & 0.1501             \\
            9                  & -0.1027               & 0.1514             \\
            10                 & -0.0665               & 0.1471             \\
        \bottomrule
        \end{tabular}
        \medskip \medskip
        
        \justifying \noindent {\footnotesize
        Note: This table presents the silhouette scores obtained for different numbers of clusters when clustering based on XPER values or feature values. The silhouette score ranges from $-1$ (worst) to $1$ (best).}
         \label{Tab:Silhouette_scores}
        \end{table}

 \begin{figure}[!h]
            \begin{subfigure}[b]{1\textwidth}
            \makebox[\linewidth]{
			\includegraphics[width=0.65
			\textwidth,trim={0 0 0 0},clip]{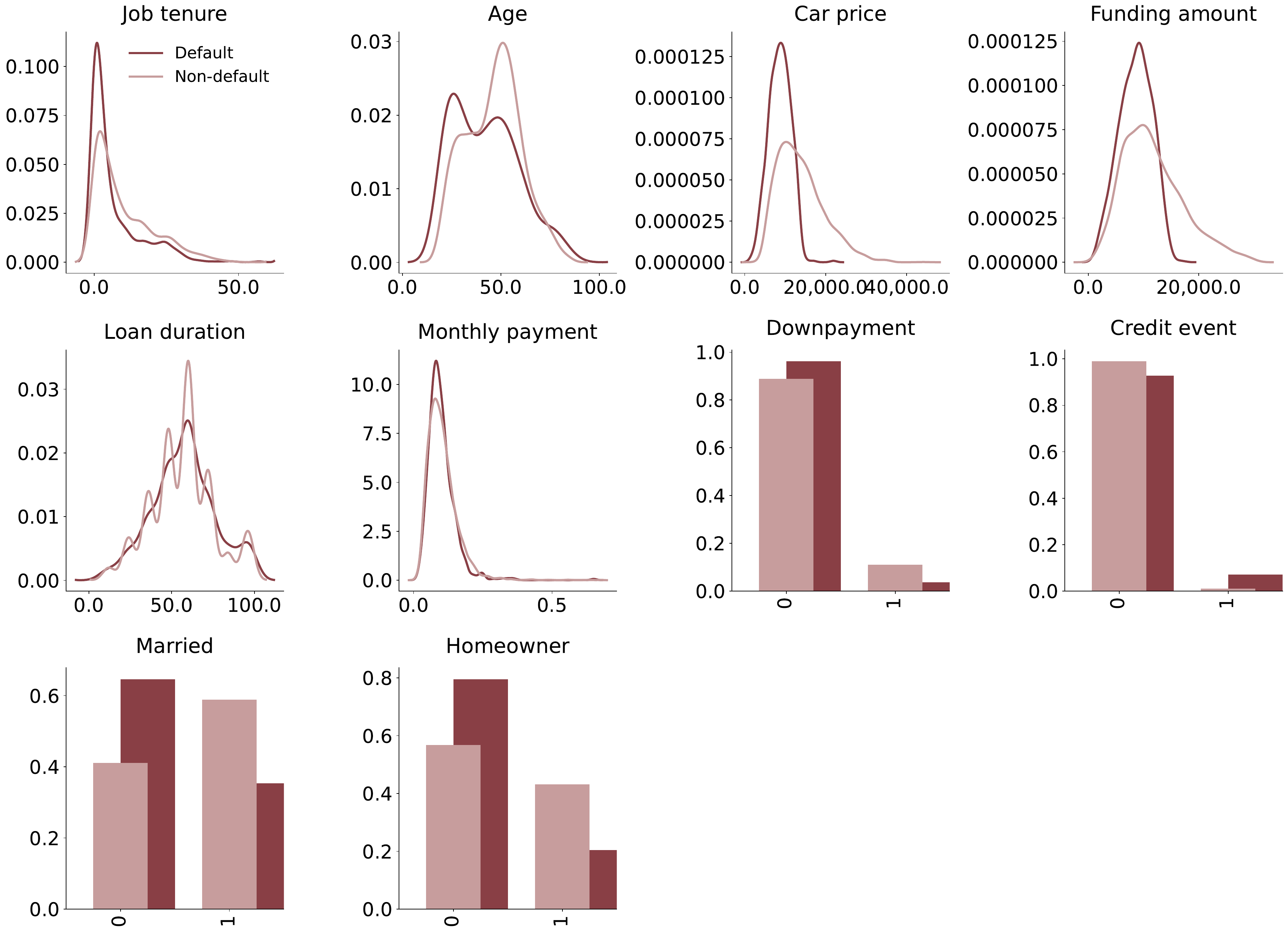}}
		\caption{Features distribution by default class in group 1}
            \label{fig:Feature_distrib_G1_train_default}
            \vspace{0.5cm}
			\end{subfigure}
			\hfill
            
			\begin{subfigure}[b]{1\textwidth}
			\makebox[\linewidth]{
			\includegraphics[width=0.65
        			\textwidth,trim={0 0 0 0},clip]{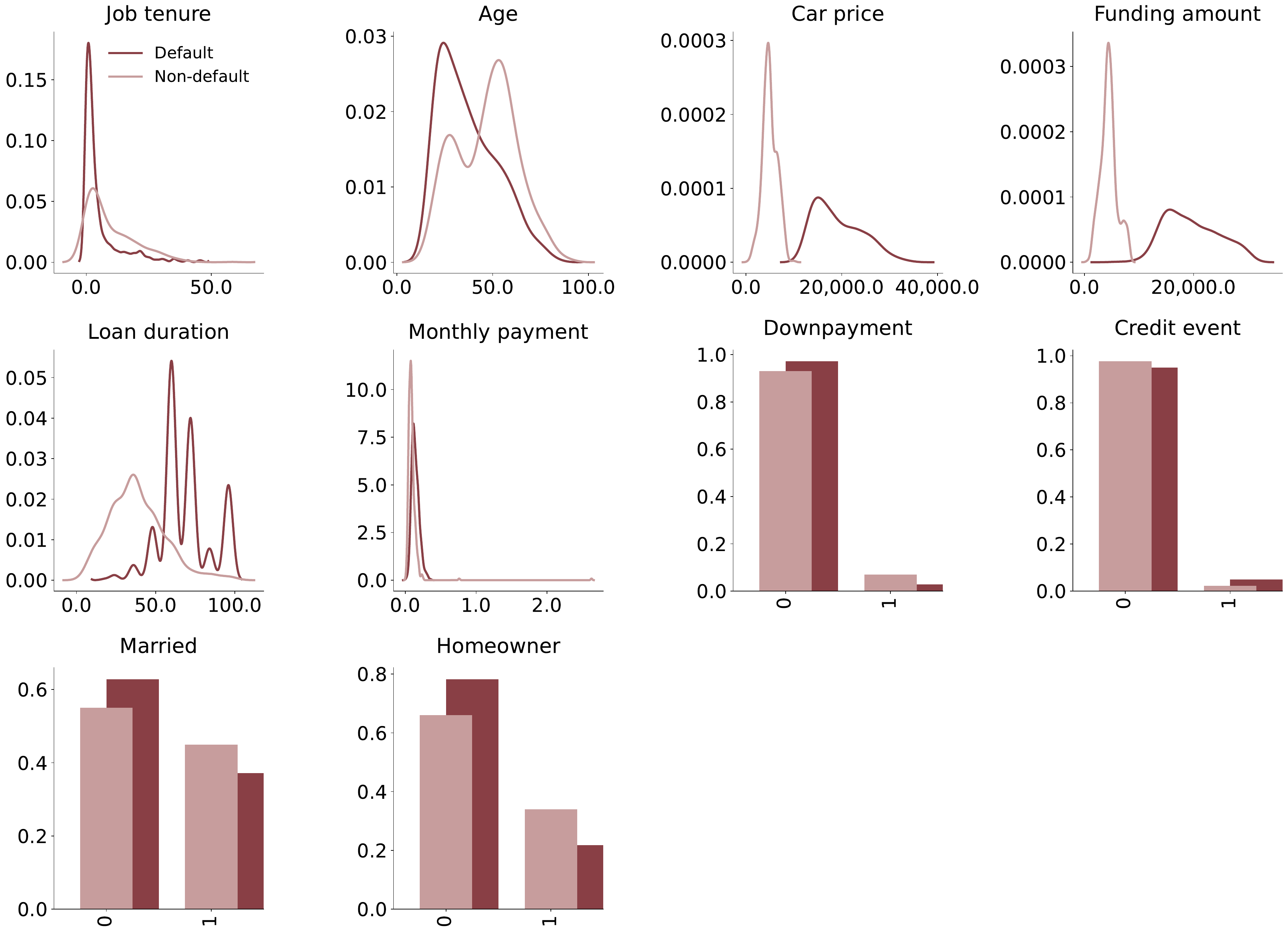}} 
        			\caption{Features distribution by default class in group 2}
           \label{fig:Feature_distrib_G2_train_default}
			\end{subfigure}
    \caption{Features distribution by default class for each group}
     \medskip \medskip
        
        \justifying \noindent {\footnotesize Note: This figure displays the distribution of the features by default class on the training sample, for the first (Panel (a)) and second group (Panel (b)) created from individual XPER values using the K-Medoids methodology. For continuous features, we use a kernel density estimation. Dark red refers to defaulting borrowers and light red to non-defaulting borrowers.}

	\end{figure}
\clearpage
\begin{figure}[h]
		\begin{center}
		    			\includegraphics[width=0.5%
			\textwidth,trim={0 0 0 0},clip]{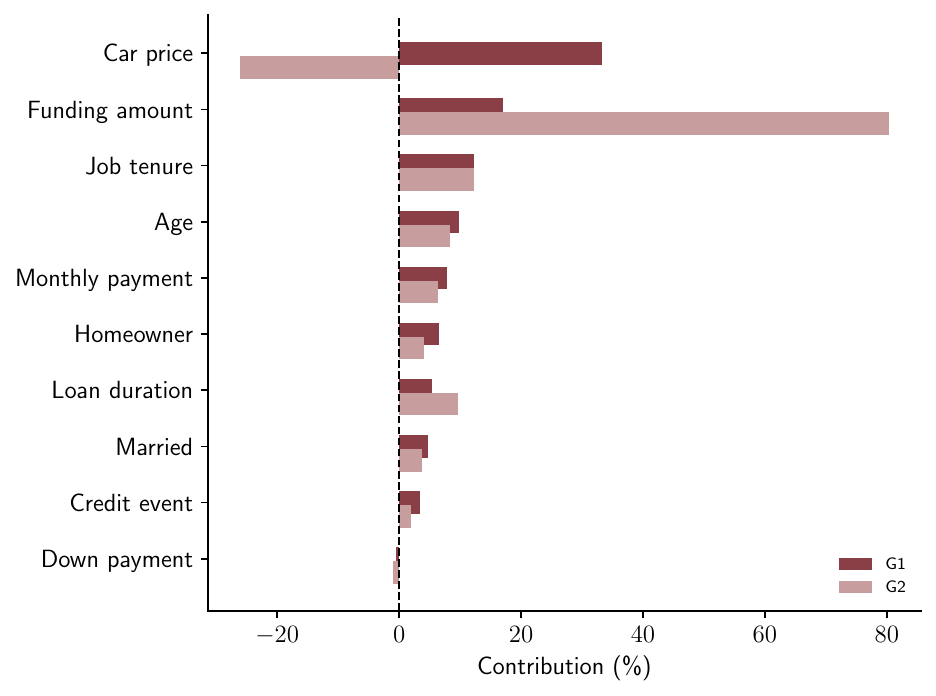}
             \caption{XPER decomposition of the AUC by group}
             \label{fig:AUC_cluster} 
		\end{center}
     \medskip
        
        \justifying \noindent {\footnotesize Note: This figure displays the XPER values of the AUC of the XGBoost model estimated on the training sample by group creating with the K-Medoids method.}
        \end{figure}

\end{document}